\documentclass[accepted]{uai2026} 
                        
\usepackage[american]{babel}
\usepackage{graphicx}
\usepackage{subcaption}
\usepackage{etoc}
\usepackage{enumitem}
\usepackage{algorithm}
\usepackage{algpseudocode}  
\usepackage{natbib} 
\usepackage{booktabs} 
\usepackage{tikz} 
\usepackage{subfiles}
\usepackage{amsthm}
\usepackage{amsmath,amssymb,mathtools}
\DeclarePairedDelimiter{\norm}{\lVert}{\rVert}
\newtheorem{theorem}{Theorem}[section]

\newtheorem{proposition}[theorem]{Proposition}
\newtheorem{lemma}[theorem]{Lemma}
\newtheorem{corollary}[theorem]{Corollary}

\theoremstyle{definition}
\newtheorem{definition}[theorem]{Definition}

\newtheorem{remark}[theorem]{Remark}
\newtheorem{assumption}[theorem]{Assumption}

\newtheorem*{proposition*}{Proposition}

\newtheorem{claim}[theorem]{Claim}

\title{Federated Causal Representation Learning in State-Space Systems \\ for Decentralized Counterfactual Reasoning}

\author[1]{Nazal Mohamed\textsuperscript{*}}
\author[1]{Ayush Mohanty\textsuperscript{*}}
\author[1]{Nagi Gebraeel}
\affil[1]{%
    Georgia Institute of Technology\\
    Atlanta, GA
}

\begin{document}
\maketitle
\begingroup
\renewcommand\thefootnote{*}
\footnotetext{Equal contribution.}
\endgroup
\begin{abstract}

Networks of interdependent industrial assets (clients) are tightly coupled through physical processes and control inputs, raising a key question: \textit{how would the output of one client change if another client were operated differently?} This is difficult to answer because client-specific data are high-dimensional and private, making centralization of raw data infeasible. Each client also maintains proprietary local models that cannot be modified. We propose a federated framework for causal representation learning in state-space systems that captures interdependencies among clients under these constraints. Each client maps high-dimensional observations into low-dimensional latent states that disentangle intrinsic dynamics from control-driven influences. A central server estimates the global state-transition and control structure. This enables decentralized counterfactual reasoning where clients predict how outputs would change under alternative control inputs at others while only exchanging compact latent states. We prove convergence to a centralized oracle and provide privacy guarantees. Our experiments demonstrate scalability, and accurate cross-client counterfactual inference on synthetic and real-world industrial control system datasets.
\end{abstract}
\addtocontents{toc}{\protect\etocsettocdepth{-1}}

\section{Introduction}
Complex industrial systems such as oil refineries, water treatment networks, and distributed process plants consist of geographically dispersed yet tightly coupled assets (clients) that interact dynamically through control actions and material flows. A change at one client such as a pump adjustment or valve operation, can causally propagate to others, influencing both performance and safety across the network. These interdependencies make the system vulnerable to cascading disruptions, where local disturbances can escalate into large-scale operational failures \cite{Pournaras, Ghosh}. Managing these risks requires causal models that not only identify which clients influence one another, but also reason about how interventions propagate across the system.

Granger causality (GC) is a classical tool for detecting interdependencies in time-series data \cite{granger1969, geweke1982, tang2023causality}. Yet GC is fundamentally predictive rather than interventional, and thus \textbf{cannot} answer counterfactual questions such as: \textit{how would my output have changed if another client had altered its control input?} Moreover, GC conflates intrinsic state dynamics with control-driven influences, leaving the source of interdependencies ambiguous. 

In practice, however, operators often need exactly this type of counterfactual reasoning \cite{tang2023causality, RUIZTAGLE2022108785}. For instance, assessing ``\textit{how reducing flow at an upstream unit would affect downstream stability}'', or ``\textit{how adjusting a peer’s control settings might influence local product quality}''. Without such capabilities, decision-making remains limited to retrospective diagnostics rather than proactive intervention planning. Practical constraints further complicate this problem. Modern assets generate high-dimensional sensor data that is expensive to transmit, while bandwidth and communication limits prevent centralization. Privacy regulations and confidentiality concerns restrict data sharing across organizational boundaries. Finally, each client typically maintains a proprietary local model trained on its own data, which must remain intact and cannot be altered in a data analysis pipeline. To address these challenges, we develop a federated learning framework tailored to distributed industrial systems.  

\textbf{Contributions.} We focus on linear state-space systems, which provide a tractable representation of causally interdependent dynamics using low-dimensional latent states. In our framework, each client retains its proprietary model and maps raw high-dimensional measurements into disentangled low-dimensional latent states. Specifically, at the client level, we utilize a modeling approach that separates autoregressive time-series dynamics from control effects. A central server acting as a conduit between these interdependent clients estimates the global structures for the state-transition and input matrices without direct access to raw high-dimensional data (measurements). This architecture equips each client to answer counterfactual queries, such as predicting how its outputs would change if another client altered its controls without sharing raw data, inputs, or proprietary models. To the best of our knowledge, this is the first work to enable decentralized counterfactual reasoning. 
The key technical contributions of this paper are as follows:
\begin{itemize}[
    leftmargin=*,
    itemsep=0.5pt,
    topsep=1pt
]
    \item We employ the ``\textit{Abduction-Action-Prediction}'' framework of \cite{pearl2009causality} to formally derive the average treatment effect in the federated setting, enabling decentralized counterfactual reasoning.
    \item We propose a federated learning approach in which the server explicitly estimates the state-transition and input matrices, while clients learn their implicit effects, without moving any private, high-dimensional data.
    \item We design an iterative optimization scheme across servers and clients that provably disentangles autoregressive dynamics from exogenous inputs, a key requirement for decentralized counterfactual reasoning.
    \item We prove convergence of the decentralized framework to a centralized oracle with direct access to all high-dimensional client measurements. In addition, we provide differential privacy guarantees for the federated protocols by quantifying the perturbations required in communicated quantities to protect client data.
    \item Experiments on synthetic data validate the main claims of the paper w.r.t. counterfactual reasoning, convergence to a centralized oracle, and differential privacy, and further demonstrate robustness and scalability. Real-world experiments confirm the applicability of the approach of industrial control systems.
\end{itemize}

 \section{Related Work}
 \textbf{Causal Inference in Dynamical Systems.} 
Granger causality (GC) is the classical tool for detecting predictive dependencies in time series \cite{granger1969,geweke1982,barnett2009granger}. Extensions such as structural GC \cite{eichler2010} and causal state-space models \cite{huang2019causal,mastakouri2021necessary} have attempted to incorporate interventional semantics, but they typically assume centralized access and conflate endogenous dynamics with exogenous inputs. Probabilistic approaches including dynamic Bayesian networks \cite{murphy2002} and latent-variable state-space models \cite{krishnan2015dkf,fraccaro2017kvae} improve representation power but remain predictive in nature.  

\textbf{Federated Causal Discovery.}  
Federated learning \cite{mcmahan2017} has been extended to causal discovery across distributed clients. \cite{Li2024FedCD} propose federated conditional independence testing for heterogeneous data. \cite{Mian22Regret,Mian23PrivacyFCD} introduce regret-based methods with privacy guarantees.  FedECE \cite{IEEE2024FCD1} estimates causal effects across institutions using graphical modeling without data sharing. \cite{IEEE2024FCD2} develop a distributed annealing algorithm for federated DAG learning under generalized linear models, offering identifiability guarantees.  More recent works address scalability. FedCausal \cite{Yang_He_Wang_Yu_Domeniconi_Zhang_2024} unifies local and global optimization through adaptive strategies, while FedCSL \cite{Guo_Yu_Liu_Li_2024} improves accuracy with local-to-global learning and weighted aggregation.  Despite these advances, existing methods focus on structural recovery or associations. They lack counterfactual semantics. 

\textbf{Causal Representation Learning.} 
Pearl’s Abduction–Action–Prediction (AAP) framework \cite{pearl2009causality} formalizes counterfactual inference and underpins modern causal reasoning. Extensions such as invariant causal prediction \cite{peters2016icp}, invariant risk minimization \cite{arjovsky2019irm}, and causal representation learning \cite{scholkopf2021causal,brehmer2022weakly,lippe2022citris} aim to learn invariant and intervention-aware representations. Disentangled state-space models \cite{miladinovic2019disentangled, weng2025sde}, and disentangled representation learning \cite{li2024disentangled, wang2024disentangled} improve interpretability but assume centralized data access. Recent counterfactual time-series approaches like \cite{mastakouri2021necessary, todo2023counterfactual} also remain centralized.

\section{Preliminaries}\label{sec:Prelim}

\textbf{Linear Time Invariant (LTI) Systems.} The dynamics of a linear system can be described by the following LTI state-space equations:
\begin{equation}
\begin{aligned}
h^t &= A\,h^{t-1} + B\,u^{t-1} + w^{t-1}; \hspace{0.1cm}
y^t = C\,h^t + v^t,
\end{aligned}
\label{eq:lti}
\end{equation}
where $h^t \in \mathbb{R}^P$ is the latent state, $u^t \in \mathbb{R}^S$ is the input (or control), and $y^t \in \mathbb{R}^D$ is the measurement (also called observation or output). The matrices $A \in \mathbb{R}^{P\times P}$, $B \in \mathbb{R}^{P\times S}$, and $C \in \mathbb{R}^{D\times P}$ are time-invariant state-transition, input, and output matrices, respectively. $w^{t-1} \sim \mathcal{N}(0, Q_w)$, and $v^t \sim \mathcal{N}(0, R_v)$ are the process noise and measurement noise, respectively. 

\textbf{Structural Causal Models (SCMs).} In a causal system, we represent interdependencies among variables using a \textit{causal graph}, where nodes correspond to variables and directed edges encode direct causal influence. For any variable $X$, we denote by $\mathrm{Pa}(X)$ the set of its parent nodes. We then distinguish between two types of variables: \textbf{(1)} \textit{Exogenous variables} ($U$) are determined outside the system being modeled. They capture background factors, noise, or unobserved influences not explained by the model itself. \textbf{(2)} \textit{Endogenous variables} ($H$) are determined within the system by structural equations. Structural assignments $E$ generate each endogenous variable $H$ from its parents $\mathrm{Pa}(H)$ in the causal graph and an associated exogenous variable $U$ such that, $ H := E(\mathrm{Pa}(H),U)$. A structural causal model (SCM) specifies an ordered triplet $\langle U,H,E \rangle$, providing a principled way to infer causal relationships in a system. 

\textbf{Intervention.} An intervention, denoted by $\mathrm{do}(u = u'),\; u \in U$, is an experiment performed by replacing the structural assignment of $u$ with the constant $u'$ and leaving all other assignments unchanged. \cite{pearl2009causality}'s Abduction-Action-Prediction (AAP) scheme given below, provides a principled approach to infer the consequence of intervention:
\begin{enumerate}[leftmargin=*, itemsep=0.5pt,
    topsep=1pt]
    \item \textbf{Abduction:} infer latent quantities from past observations
    \item \textbf{Action:} encode the intervention by modifying the relevant structural assignment(s)
    \item \textbf{Prediction:} propagate the intervened model forward to obtain interventional outcomes.
\end{enumerate}

Predicting the consequence of intervention is the key to answering counterfactual questions such as, ``\textit{what would $h \in H$ have been under a different value of $u$?}''. A common way to quantify such effects is via the \emph{average treatment effect} (ATE), defined as follows:
\begin{equation}
    ATE \coloneqq \mathbb{E}[h \mid \mathrm{do}(u = u_1)] - \mathbb{E}[h \mid \mathrm{do}(u = u_0)],
\end{equation}
In other words, $ATE$ measures the expected change in the endogenous variable $h$ when the exogeneous variable $u$ is set to $u_1$ instead of $u_0$ through intervention. 


\textbf{LTI system as an SCM.} The LTI state-space system in \eqref{eq:lti} can be alternatively viewed as an SCM, where $h^t$ and $y^t$ are endogenous variables, and $u^{t-1}$, $w^t$ and $v^t$ are exogenous. The input $u^{t-1}$ is not governed by the system dynamics but externally specified. Hence $u^{t-1}$ qualifies as an intervention variable. The state $h^t$ \textbf{cannot} be observed directly (thus, called \textit{latent state}) and has to be estimated from the measurements $y^t$, using a Kalman filter \cite{og_kalman} . 

\section{Counterfactuals for LTI}\label{Sec:Counterfactual}
\begin{proposition} \label{prop:counterfactual_lti}
    Let the input to an LTI system be $u^{t-1} = u_0$. Then the ATE on measurements $y^t$ under the intervention $\mathrm{do}(u^{t-1} = u_1)$ equals $CB\,(u_1 - u_0)$.
\end{proposition}
Proposition~\ref{prop:counterfactual_lti} shows that only the matrices $C$ and $B$ parameterize the $ATE$ of externally applied inputs $u$ on downstream measurements $y$.

\subsection{Multi-Client System} \label{subsec: multi_client}
We consider a \emph{multi-client} LTI state-space setting with $M$ clients (subsystems). Client $m \in \{1,\dots,M\}$ has local state $h_m^t \in \mathbb{R}^{P_m}$, input $u_m^t \in \mathbb{R}^{U_m}$, and measurement $y_m^t \in \mathbb{R}^{D_m}$, with $D_m >> P_m$, $U_m$. The joint dynamics follow Eq~\eqref{eq:lti}. We further assume:
\begin{assumption}\label{assumption_C}
    The matrix $C$ is block-diagonal.
    
    \textbf{\textit{Rationale.}} In many engineered systems, sensors are tied to a physical subsystem, so $y_m^t$ depends \textbf{only} on $h_m^t$, i.e., $y_m^t = C_{mm} h_m^t$.
\end{assumption}`
Under Assumption~\ref{assumption_C}, Proposition \ref{prop:counterfactual_lti} gives the following causal effect ($ATE$) in the multi-client setting:
\begin{equation}\label{eq_counterfactual}
\begin{split}
   \mathbb{E}[y_m^t \mid \mathrm{do}(u_n^{t-1} &= u_{n_1})] - \mathbb{E}[y_m^t \mid \mathrm{do}(u_n = u_{n_0})] \\&= C_{mm}B_{mn}(u_{n_1} - u_{n_0}) \hspace{0.25cm} \forall m \neq n
\end{split}
\end{equation}
Equation~\ref{eq_counterfactual} addresses the \textbf{counterfactual} query:

\textbf{Q1.} \textit{“What would $y_m^t$ have been if client $n$’s input $u_n^{t-1}$ had changed?”}

\textbf{Problem Formulation.} We study counterfactual reasoning in multi-client LTI systems under decentralized constraints. Each client observes only high-dimensional local measurements $y_m^t$, with no access to other clients’ data. The goal is to estimate cross-client causal effects $B_{mn}$ (Eq.~\ref{eq_counterfactual}) that answer \textbf{Q1}, while respecting privacy and comm. constraints.

\textbf{Role of $A$ in Latent-State Recovery.} Although the $ATE$ in Proposition~\ref{prop:counterfactual_lti} depends only on $B_{mn}$, estimating it in practice requires access to latent states $h^t$. These states evolve according to the transition matrix $A$. Without modeling $A$, one cannot disentangle \textbf{(i)} endogenous propagation of past states via $A_{mn}$ from \textbf{(ii)} exogenous input effects via $B_{mn}$. Thus, learning $A_{mn} \forall m, n$ is essential both for recovering $h^t$ and for isolating the causal contribution of $B_{mn}$.
\begin{figure*}[t]
    \centering
    \includegraphics[width=\textwidth]{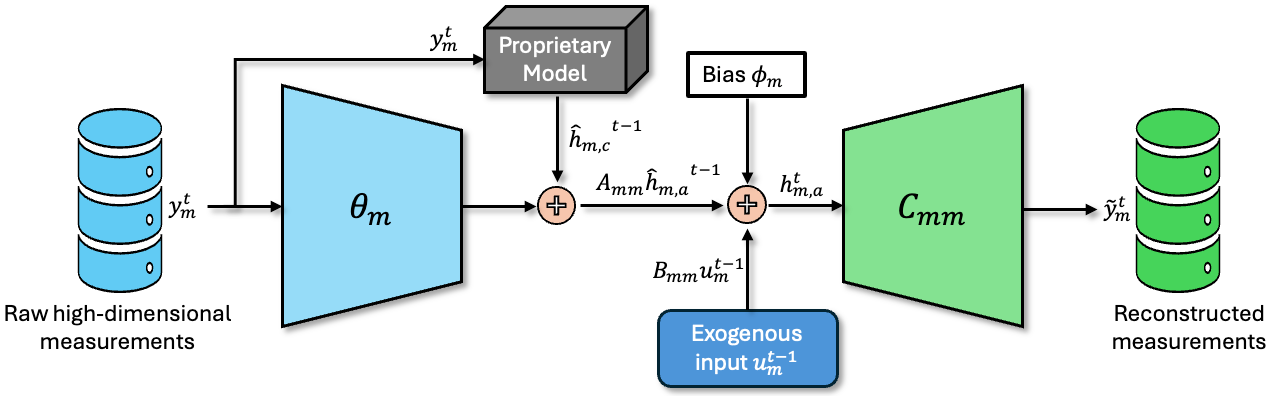}
    \caption{A simplified view of the computation inside the client $m$}
    \label{fig:client}
\end{figure*}
\subsection{Decentralized Setting}\label{Sec:Counterfactual_Decentralized}
While Eq.~\ref{eq_counterfactual} identifies $B_{mn}$ as the key parameter for counterfactual reasoning, estimating it in a decentralized environment is challenging due to:
\textbf{(1) Partial observability:} Only noisy measurements $y^t$ are observed locally. Recovering $B_{mn}$ requires first inferring latent states $h^t$ by modeling $A$ and then separating the effects of $A_{mn}$ and $B_{mn}$.
\textbf{(2) Data sharing constraints:} Clients cannot share raw $y^t$ due to privacy, and $y^t$ is often high-dimensional, making direct communication infeasible.

These constraints motivate a federated approach to estimate $\{A_{mn}, B_{mn}\}$. Each client transmits low-dimensional state estimates $\hat{h}_m^t$, and inputs $u_m^t$ to a coordinating server, without exposing raw measurements $y_m^t$. The server explicitly estimates $\{A_{mn}, B_{mn}\}$ using aggregated $\{\hat{h}^t, u^t\}$ across clients (extending prior work, e.g., \cite{ICLR2025_fedGC}, which solved only the special case $B \equiv 0$). This enables server-side counterfactual reasoning via Eq.~\ref{eq_counterfactual}.

\textbf{Key challenge for clients.} Unlike the server, client $m$ observes only its own $\hat{h}_m^t$ and $u_m^t$, and receives no states or inputs from others. Local estimates, therefore, conflate two sources of variation: \textbf{(i)} endogenous propagation of past states (through $A_{mn}$) and \textbf{(ii)} exogenous input effects (through $B_{mn}$). To support \textbf{Q1} locally, clients must learn representations in which these effects are \textbf{disentangled}, so that $B_{mn}$ can be isolated and used in Eq.~\ref{eq_counterfactual}. Our training methodology is designed to enforce this disentanglement at the client, while relying only on server communications.

\section{Federated Training}
\label{Sec: Federated_training_methodology}
\subsection{Proprietary Client Model}
\begin{assumption}\label{assumption_blocks}
    Each client $m$ knows its local block matrices $A_{mm}, B_{mm},$ and $C_{mm}$.

    \textbf{\textit{Rationale.}} The matrices $A_{mm}, B_{mm},$ and $C_{mm}$ correspond to the diagonal blocks of the global state-transition matrix $A$, input matrix $B$, and output matrix $C$ of the LTI system. These local blocks do not affect counterfactual query \textbf{Q1}. Moreover, if not pre-specified, they can be readily estimated by each client using only its own data.
\end{assumption}
\textbf{Notation.} Kalman filtering distinguishes a one-step-ahead \emph{predicted state} from a \emph{refined estimated state}. We denote by $h^t$ the one-step-ahead prediction at time $t$ (based on information up to $t-1$), and by $\hat{h}^t$ the refined estimate after observing $y^t$ at time $t$. 

The proprietary client model is a Kalman filter (in general it can be any state estimator), which makes use of the local block matrices $A_{mm}$, $B_{mm}$, $C_{mm}$, the local measurement $y_m^t$, and input $u^t_m$ to compute a \textit{predicted state} $h^t_{m,c}$ and a \textit{refined estimate} $\hat{h}^t_{m,c}$. 

\begin{assumption}
    The proprietary model provides both $h^t_{m,c}$ (predicted state) and $\hat{h}^t_{m,c}$ (refined estimate), which are treated as given inputs to our federated learning method. 
\end{assumption}
\textbf{{Limitation of the Proprietary Model.}} Since the proprietary model relies exclusively on local information (i.e., $A_{mm}, B_{mm}, C_{mm}, y_m^t,$ and $u_m^t$), it \textbf{cannot} capture cross-client interdependencies and is therefore unable to answer query \textbf{Q1} (Section \ref{Sec:Counterfactual_Decentralized}).

\subsection{Augmented Client Model}
To overcome the limitations of the proprietary model, we \textbf{\textit{augment}} it with additional ML parameters, yielding an \textit{augmented client model}. The governing equations are:
\begin{align}
    h^{t}_{m,a} &= A_{mm}\, \hat{h}^{t-1}_{m,a} + B_{mm}\, u^{t-1}_{m} + \phi_{m}, \label{eq:aug_state_pred}\\
    \hat{h}^{t}_{m,a} &= \hat{h}^{t}_{m,c} + \theta_{m}\, y^{t}_{m}. \label{eq:aug_state_estimate}
\end{align}
Here, $h^{t}_{m,a}$ denotes the \textit{augmented predicted state}, and $\hat{h}^{t}_{m,a}$ the \textit{augmented refined estimate}, which extends the proprietary estimate $\hat{h}^{t}_{m,c}$ by incorporating the learned parameter $\theta_m$. The corresponding \textbf{client loss} is defined as,  
   $L_{m,a}^t \coloneqq \|r^t_{m,a}\|_2^2,$
where we have, 
\begin{equation}\label{eq:client_residual}
    r^{t}_{m,a} = y^{t}_{m} - \tilde{y}^{t}_{m, a}, 
    \hspace{0.25cm} \text{with} \hspace{0.25cm} 
    \tilde{y}^{t}_{m, a} = C_{mm}\, h^{t}_{m,a}. 
\end{equation}
Therefore, the client loss is just a \textit{reconstruction loss} with $\tilde{y}^{t}_{m, a}$ denoting the \textit{reconstructed measurement},  which serves as the key output for counterfactual analysis (used later in Section \ref{sec:disentanglement}). 

\textbf{Parameter updating.} 
Each client updates $\phi_m$ and $\theta_m$ by combining gradients from its own loss $L_{m,a}$ with gradients derived from the \textbf{server loss $L_s$}. Importantly, the server never sends raw data; instead, it communicates only the gradients with respect to the augmented states, namely $\nabla_{h_{m,a}^t} L_s$ and $\nabla_{\hat{h}_{m,a}^{t-1}} L_s$. 
The client then uses the chain rule of partial derivatives (proof can be found in the Appendix) to compute:   
\begin{align} 
\nabla_{\theta_m} L_s &=\sum_{t = 1}^T \left(A_{mm}^\top \cdot \nabla_{h^t_{m,a}} L_s + \nabla_{\hat{h}_{m,a}^{t-1}} L_s \right) {y^{t-1}_m}^\top  \\ 
\nabla_{\phi_m} L_s &=  \sum_{t = 1}^T \nabla_{h^t_{m,a}} L_s.\label{eq:L_s_grad_chain_rule} 
\end{align}
Using the locally computed $\nabla_{\theta_m} L_s, \nabla_{\phi_m} L_s$, and learning rates $\eta_1, \eta_2, \gamma_1, \gamma_2$, the updated parameters are then given by,
\begin{align}
    \theta_m^{k+1} &= \theta_m^{k}
    - \eta_1 \,\nabla_{\theta_m} L_{m,a}
    - \eta_2 \,\nabla_{\theta_m} L_{s}, \\
    \phi_m^{k+1} &= \phi_m^{k}
    - \gamma_1 \,\nabla_{\phi_m} L_{m,a}
    - \gamma_2 \,\nabla_{\phi_m} L_{s}. \label{eq:client_update}
\end{align} 

\begin{claim}\label{claim_phi_theta_combined}
    The local client parameters $\theta_m$ and $\phi_m$ learn to encode cross-client interdependencies arising from off-diagonal blocks $A_{mn}$ and $B_{mn}$ ($n \neq m$).
\end{claim}

\textbf{\textit{Intuition on Claim \ref{claim_phi_theta_combined}}.}  
Although each client observes only local data, the server gradients $\nabla_{\theta_m} L_s$ and $\nabla_{\phi_m} L_s$ implicitly carry information about how other clients’ data influence the server loss. When combined with local gradients, these updates guide $\theta_m$ and $\phi_m$ to represent cross-client effects, enabling the augmented model to capture interdependencies that the proprietary model cannot. A theoretical proof is provided in Section~\ref{Sec:Causal_Disentanglement}, with empirical validation in Section~\ref{Sec:Experiments}.

\textbf{Comm. to the Server:} Each client $m$ sends $\{\hat{h}_{m,c}^t,\hat{h}_{m,a}^t,h_{m,a}^t,u_m^t\}$ to the server. 

\subsection{Server Model}\label{Sec:Server}
\begin{assumption}\label{assumption_server_blocks}
    The diagonal blocks $A_{mm}$ and $B_{mm}$ are known at the server. 
\end{assumption}

\textbf{\textit{Rationale behind Assumption \ref{assumption_server_blocks}.}}  
From Assumption~\ref{assumption_blocks}, each client $m$ already has access to $A_{mm}$ and $B_{mm}$. Their effects are purely local and do not influence cross-client counterfactual reasoning, hence they are not central to our analysis. Communicating them to the server, therefore, incurs only a fixed one-time cost during initialization.  

After receiving $\hat{h}_{m,c}^t, \hat{h}_{m,a}^t, h_{m,a}^t,$ and $u_m^t$ from client $m$, together with the known diagonal blocks $\{ A_{mm}, B_{mm}\}$ and the current estimates of $\{\hat A_{mn}, \hat B_{mn}\}_{n \neq m}$, the server predicts the future state $h^{\,t}_{m,s}$, and computes a residual $r_{m, s}^t$ given by, 
\begin{align}
h^{\,t}_{m,s} &= A_{mm}\,\hat h^{\,t-1}_{m,c}
+ \sum_{\substack{n\neq m}} \hat A_{mn}\,\hat h^{\,t-1}_{n,c}
\nonumber \\&+ B_{mm}\,u_m^{t-1}
+ \sum_{\substack{n\neq m}} \hat B_{mn}\,u_n^{t-1}, \hspace{0.25cm} \text{and} \label{eq: server_state}\\
r_{m,s}^{t} &= h^{\,t}_{m,a} - h^{\,t}_{m,s}. \label{eq:server_residual}
\end{align}

To learn an estimate of the off-diagonal blocks $\{A_{mn}, B_{mn}\}_{n \neq m}$, the server minimizes a loss, 
\begin{equation}
\label{eq:server_loss}
\begin{split}
L_s &:=\frac{1}{T}\sum_{t=1}^{T}\sum_{m=1}^{M} \Big(\| r^t_{m,s} \|_2^2
\\&+ 
\xi_m\underset{\mathcal{D}}{\underbrace{
\,
\big\|
A_{mm}\big(\hat h^{\,t-1}_{m,a} - \hat h^{\,t-1}_{m,c}\big)
- \sum_{n \neq m} \hat A_{mn}\,\hat h^{\,t-1}_{n,c}
\big\|_2^2
\Big)}}
\end{split}
\end{equation}
\begin{claim}\label{claim:disentanglement}
With large $\xi_m$, $\mathcal{D}$ disentangles cross-client effects of $A_{mn}$ (endogenous) and $B_{mn}$ (exogeneous).
\end{claim}

Claim~\ref{claim:disentanglement} ensures that each client separates the contributions of $A_{mn}$ and $B_{mn}$ when learning its parameters $\theta_m$ and $\phi_m$, respectively. Theoretical proof of Claim~\ref{claim:disentanglement} is provided in Section~\ref{Sec:Causal_Disentanglement}.

\textbf{Parameter Updating.} The server updates the off-diagonal blocks by block gradient descent, 
\begin{align}
\hat A_{mn}^{k+1} &= \hat A_{mn}^k - \alpha_A \,\nabla_{\hat A_{mn}^k} L_s, \hspace{0.2cm} \text{and} \\
\hat B_{mn}^{k+1} &= \hat B_{mn}^{k} - \alpha_B \,\nabla_{\hat B_{mn}^k} L_s, \qquad \forall n \neq m. 
\end{align}
\textbf{Comm. to the Clients.} The server communicates gradients $\nabla_{h^{\,t}_{m,a}} L_s$ to each client $m$. 

\subsection{Overhead}
\textbf{Communication.} At each training round, every client transmits latent states and inputs $\{\hat{h}_{m,c}^t,\hat{h}_{m,a}^t,h_{m,a}^t,u_m^t\}$ of dimension $\mathcal{O}(P_m + U_m)$ to the server, and receives gradient updates of size $\mathcal{O}(P_m)$, leading to a per-round communication cost of $\mathcal{O}(M(P+U))$, where $P = \sum_m P_m$ and $U = \sum_m U_m$. 

\textbf{Computation.} The server performs block gradient descent over matrices $\{A_{mn}, B_{mn}\}$, with per-round complexity $\mathcal{O}(M^2 P^2 + M^2 U^2)$ in the worst case. Each client’s local updates (Eqs.~\ref{eq:client_update}) require $\mathcal{O}(TP_m^2)$ operations per round, dominated by matrix-vector products with dimension $P_m$.

\section{Causal Representations}\label{Sec:Causal_Disentanglement}
\label{sec:disentanglement}
The training in Section~\ref{Sec: Federated_training_methodology} is iterative:  
\textbf{(i)} the server learns cross-client dependencies via off-diagonal blocks $A_{mn},B_{mn}$ ($n\!\neq\!m$), and  
\textbf{(ii)} clients encode them in local parameters $\theta_m,\phi_m$.  
Cross-client information flows to the server through low-dimensional states, and to clients via server-loss gradients $L_s$ (Claim~\ref{claim_phi_theta_combined}). We now show these gradients explicitly encode $A_{mn},B_{mn}$:

The gradients sent from the server to the client $m$ are $\nabla_{h_{m,a}^t} L_s$ defined as, 
\begin{align}
\nabla_{h_{m,a}^t} L_s \nonumber
&:= \frac{2}{T}\!\Big(
h_{m,a}^t
- \big[
A_{mm}\,\hat h^{\,t-1}_{m,c}
+ \sum_{\substack{n\neq m}} \mathbf{{\hat A}}_{mn}\,\hat h^{\,t-1}_{n,c}
\\& + B_{mm}\,u_m^{t-1}
+ \sum_{\substack{n\neq m}} \mathbf{{\hat B}}_{mn}\,u_n^{t-1}
\big] \Big). \label{eq:L_s_local_gradients}
\end{align}
Eq.~\ref{eq:L_s_local_gradients} shows that $\nabla_{h_{m,a}^t} L_s$ depends on both $A_{mn}$ and $B_{mn}$. By Eqs.~\ref{eq:L_s_grad_chain_rule}–\ref{eq:client_update}, $\theta_m$ and $\phi_m$ thus learn an \textbf{entangled representation} of these effects. Yet counterfactuals in Eq.~\ref{eq_counterfactual} require isolating $B_{mn}$ from $A_{mn}$, making entanglement a barrier to answering query \textbf{Q1}. The next section addresses disentanglement via Claim~\ref{claim:disentanglement}.

For all clients $m$, at any time $t$, the following hold:
\begin{theorem} \label{thm:disentanglement_main_thm}
In Eq.~\ref{eq:server_loss}, as $\xi \rightarrow \infty$, the stationary points $\hat{A}_{mn}^*$ of the server and $\theta_m^*$ of the clients satisfy  
\begin{equation}
\mathrm{E}\Big[A_{mm}\theta_m^* y_m^{t}\Big] = \mathrm{E}\Big[\sum_{n \neq m} \hat A_{mn}^*\,\hat h^{\,t}_{n,c}\Big] \quad \forall m \neq n.
\end{equation}
\end{theorem}
\begin{corollary} \label{cor:disentanglement}
As $\xi \rightarrow \infty$, the stationary points $B_{mn}^*$ of the server and $\phi_m^*$ of the clients satisfy  
\begin{equation}
    \phi_m^* = \mathrm{E}\Big[\sum_{n \neq m} \hat{B}_{mn}^* u_n^{t}\Big] \quad \forall m \neq n.
\end{equation}
\end{corollary}
With sufficiently large $\xi_m$ (Eq.~\ref{eq:server_loss}), Theorem~\ref{thm:disentanglement_main_thm} and Corollary~\ref{cor:disentanglement} ensure client parameters disentangle $A_{mn}$ and $B_{mn}$, even though the server communicated gradients $\nabla_{h_{m,a}^t} L_s$ remain entangled (Eq~\eqref{eq:L_s_local_gradients}).

We have established that the client $m$ implicitly learns $\sum_{n \neq m} B_{mn} u_n^t$ (in expectation) through its augmented parameter $\phi_m$, while the server explicitly learns $B_{mn}$ for all $n \neq m$. Consequently, both can answer counterfactual queries, albeit with different fidelity:  

\vspace{0.1cm}
\textbf{(1) Server.} The server evaluates counterfactuals at the state level, estimating the following quantity: 
\begin{align}
    \mathbb{E}[h_{m,s}^t \mid \mathrm{do}(u_n^{t-1} = u_{n_1})] & - 
    \mathbb{E}[h_{m,s}^t \mid \mathrm{do}(u_n^{t-1} = u_{n_0})] \nonumber
    \\& = B_{mn}(u_{n_1} - u_{n_0}),\label{eq_counterfactual_server}
\end{align}
answering:
\textbf{Q2.} \textit{``What would $h_{m,s}^t$ have been if client $n$’s input $u_n^{t-1}$ had changed?''}  

\vspace{0.1cm}
\textbf{(2) Client.} The client evaluates counterfactuals at the measurement level, essentially estimating:
\begin{align}
    \mathbb{E}[\tilde{y}^{t}_{m, a} & \mid \mathrm{do}(\phi_m  = \phi_{m_1})] - 
    \mathbb{E}[\tilde{y}^{t}_{m, a} \mid \mathrm{do}(\phi_m = \phi_{m_0})] \nonumber
    \\& = C_{mm}(\phi_{m_1} - \phi_{m_0}), \nonumber\\
    & = C_{mm} \cdot \mathrm{E}\Big[\sum_{n \neq m}B_{mn}(u_{n_1}-u_{n_0})\Big]
    \label{eq_counterfactual_client}
\end{align}
which, by Corollary \ref{cor:disentanglement}, answers the query: \\ 
\textbf{Q3.} \textit{``What would $\tilde{y}^{t}_{m, a}$ have been if the aggregated effects of inputs from other clients had changed?''}  

Thus, the server isolates the effect of a specific client’s input $u_n$, while the client reasons only through the aggregated influence $\sum_{n \neq m} B_{mn} u_n$ encoded in $\phi_m$.

\subsection{Causal Identifiability} \label{sec:causal_id}

A central question in causal representation learning is whether a 
model has truly learned the underlying causal structure or has merely 
fit the observed data. Ideally, 
one would verify that the server learned parameters $\hat{A}_{mn}$ and 
$\hat{B}_{mn}$ match the ground-truth structural parameters $(A_mn, B_mn)$. However, 
this comparison is complicated by a fundamental identifiability issue: 
The structural parameters of an LTI system are \textbf{not} uniquely recoverable 
from input $u$ and measurements $y$ (also called \textit{input-output} data). 

For example, if the latent state $h^t_m$ 
is reparameterized by any invertible transformation $T$ as $h_t' = 
Th_t$, then the triple $(A' = TAT^{-1},\ B' = TB,\ C' = CT^{-1})$ 
defines a different-looking set of structural matrices that generates 
identical observations $y^t$. Consequently, different parameterizations 
of $(A, B, C)$ are all equally consistent with the input-output data. This
makes a direct comparison of learned parameters to a specific 
ground-truth realization, an ill-posed criterion.

That brings us to the important question: \textit{What is uniquely 
determined by the input-output data?} That quantity is the 
\textit{transfer function} $G(z) = C(zI - A)^{-1}B$, which encodes 
how interventions on $u^t$ propagate to measurements $y^t$. Given the 
representation of LTI systems as SCMs discussed in Sections 
\ref{sec:Prelim} and \ref{Sec:Counterfactual}, the \textbf{transfer function 
is precisely the object that carries causal meaning}. It is invariant 
to basis changes in the latent state and uniquely determines the ATE 
of Proposition~\ref{prop:counterfactual_lti}. Our framework implicitly learns this transfer function $G(z)$ 
without any data centralization.

\section{Privacy Analysis}
We provide theoretical results on the differential privacy analysis of the proposed framework in the Appendix \ref{appendix:privacy}.

\section{Convergence to an Oracle}
 \begin{figure*}[t]
  \centering
  \begin{subfigure}[t]{0.32\textwidth}
    \centering
    \includegraphics[width=\linewidth]{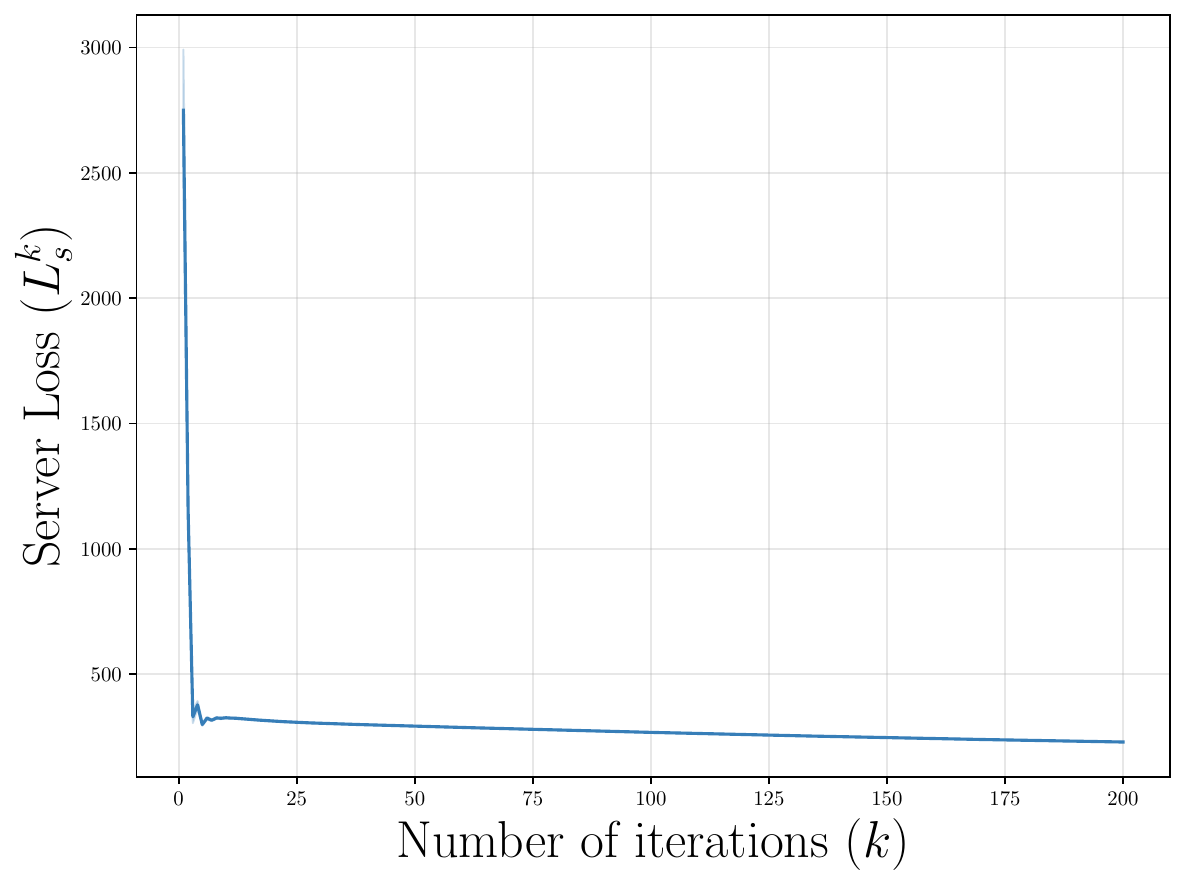}
    \caption{Global loss $L_s^k$}
    \label{fig:syn_global_loss}
  \end{subfigure}
    \hfill
  \begin{subfigure}[t]{0.32\textwidth}
    \centering
    \includegraphics[width=\linewidth]{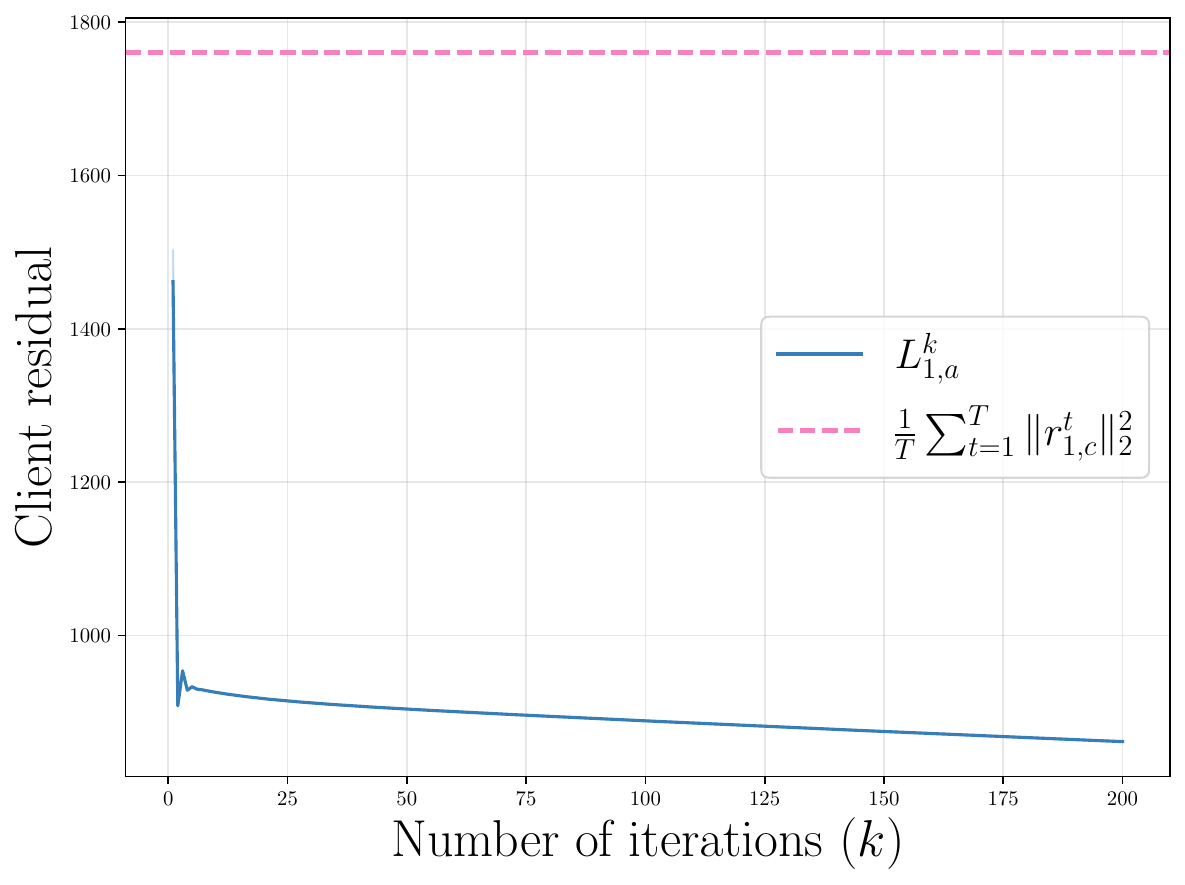}
    \caption{Client 1: losses}
    \label{fig:syn_client1_residual}
  \end{subfigure}
  \hfill
  \begin{subfigure}[t]{0.32\textwidth}
    \centering
    \includegraphics[width=\linewidth]{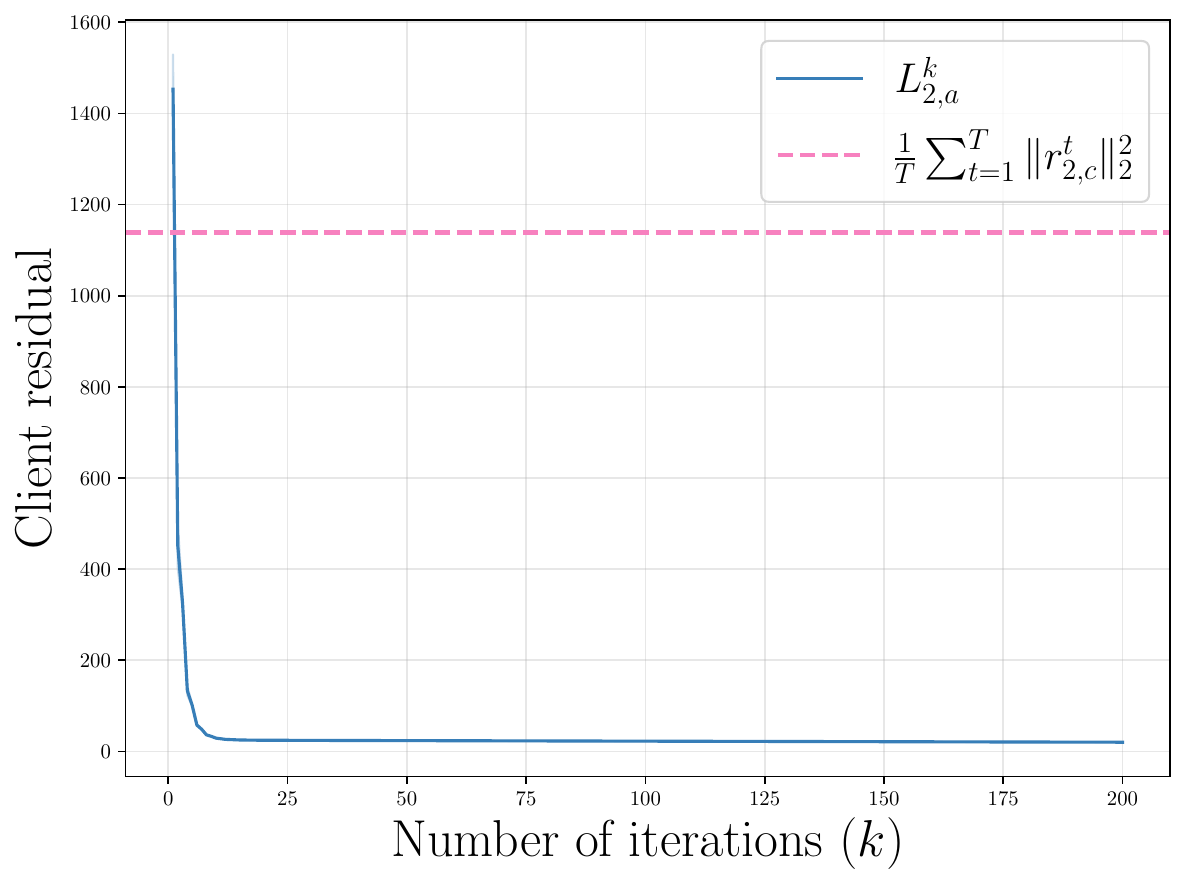}
    \caption{Client 2: losses}
    \label{fig:syn_client2_residual}
  \end{subfigure}
  \caption{Global loss at the server and losses at the client from different models (augmented client loss: $L_{m,a}^k$, proprietary client loss: $\frac{1}{T} \sum_{t = 1}^T {||{r^t_{m,c}}_2||}^2$ are plotted against the number of iterations ($k$).}
  \label{fig:syn_losses}
\end{figure*}
We consider a centralized model as the \textit{oracle} for our approach. Since each proprietary client model is a local Kalman filter (KF), a natural oracle that encodes full cross-client causality is a centralized KF with access to all clients’ data and full knowledge of $(A,B)$. Thus the oracle's dynamics are governed as: 
\begin{align}
h_o^{t} &= A\,\hat h_o^{t-1} + B\,u^{t-1}; \hspace{0.5cm}
\tilde y_o^{t} = C\,h_o^{t}, \\
r_o^{t} &= y^{t} - \tilde y_o^{t}; \hspace{0.25cm} \text{and} \hspace{0.25cm}
\hat h_o^{t} = h_o^{t} + K_o^{t}\,r_o^{t},
\end{align}
where $K_o^{t}$ is the centralized Kalman gain. 
\begin{assumption}[\textbf{Exogenous Input}]
\label{ass:erg_pe}
The input $u^{t}$ is zero-mean, \textit{stationary}, \textit{ergodic}, and independent of noises $w^{t},$ and $v^{t}$ (Eq~\eqref{eq:lti}), with $\mathbb{E}\|u^{t}\|^{2}<\infty$. Moreover, $u^{t}$ is \emph{persistently excited} of order $L$ if there exist $L\in\mathbb{N}$ and $\alpha>0$ such that
\begin{equation}
\frac{1}{L}\sum_{k=t}^{t+L-1} u^{k}u^{k\!\top}\ \succeq\ \alpha I_{d_u}\qquad\text{for all }t.
\end{equation}
\end{assumption}
\begin{lemma}[\textbf{Normal equations}]
\label{lem:normal_eqs}
Under Assumption~\ref{ass:erg_pe}, at stationary points we have,
\begin{align}
\mathrm{E}\!\big[(C_{mm}A_{mm})^\top r_{m,a}^{*,t}(y_m^{t-1})^\top\big] = 0, \nonumber\\ 
\mathrm{E}\!\big[C_{mm}^\top r_{m,a}^{*,t}\big] = 0.
\end{align}
where, $r_{m,a}^{*,t} = y_m^t - \tilde{y}_{m, a}^t(\theta_m^*, \phi_m^*)$, same as Eq.~\eqref{eq:client_residual}.
\end{lemma}
\begin{theorem}[\textbf{Convergence to the oracle}]
\label{thm:equality_oracle_expectation}
Assume $C_{mm}A_{mm}$ is full column rank and Assumption~\ref{ass:erg_pe} holds. Then in the regime $\xi_m\to\infty$ we have,
\begin{equation}
\mathrm{E}\big[\tilde y_{m,o}^{t} - \tilde y_{m,a}^{t}(\theta_m^{*},\phi_m^{*})\big]
\;=\;J_m\,\mathrm{E}[z_m^{t-1}].
\end{equation}
where, $z_m^{t-1}:=[y_m^{t-1};u^{t-1};\hat h_{m,c}^{t-1}]$,
$\Sigma_z:=\mathrm{E}[z_m^{t-1}z_m^{t-1\!\top}]$,
$e_t:=C_{mm}\!\Big(A_{mm}\theta_m^* y_m^{t-1}-\!\!\sum_{n\neq m}\hat A_{mn}^* \hat h_{n,c}^{t-1}
+\phi_m^*-\!\!\sum_{n\neq m}\hat B_{mn}^* u_n^{t-1}\Big)$, and
$J_m:=\mathrm{E}[\,e_t z_m^{t-1\!\top}\,]\Sigma_z^{-1}$.
\end{theorem}
Lemma~\ref{lem:normal_eqs} states that, at stationarity, the augmented client residual is orthogonal (in expectation) to the local regressors it uses ($y_m^t, u^t, \hat{h}_{m, c}^t$); i.e., no further linear correction along those directions can reduce the residual. Theorem~\ref{thm:equality_oracle_expectation} then says that, when the client models have causal disentanglement ($\xi_m \to \infty$ regime), the federated reconstruction matches the oracle in expectation up to a bias $J_m\,\mathrm{E}[z_m^{t-1}]$.

\section{Experiments}\label{Sec:Experiments}
\begin{figure*}[t]
  \centering

  \begin{subfigure}[t]{0.32\textwidth}
    \centering
    \includegraphics[width=\linewidth]{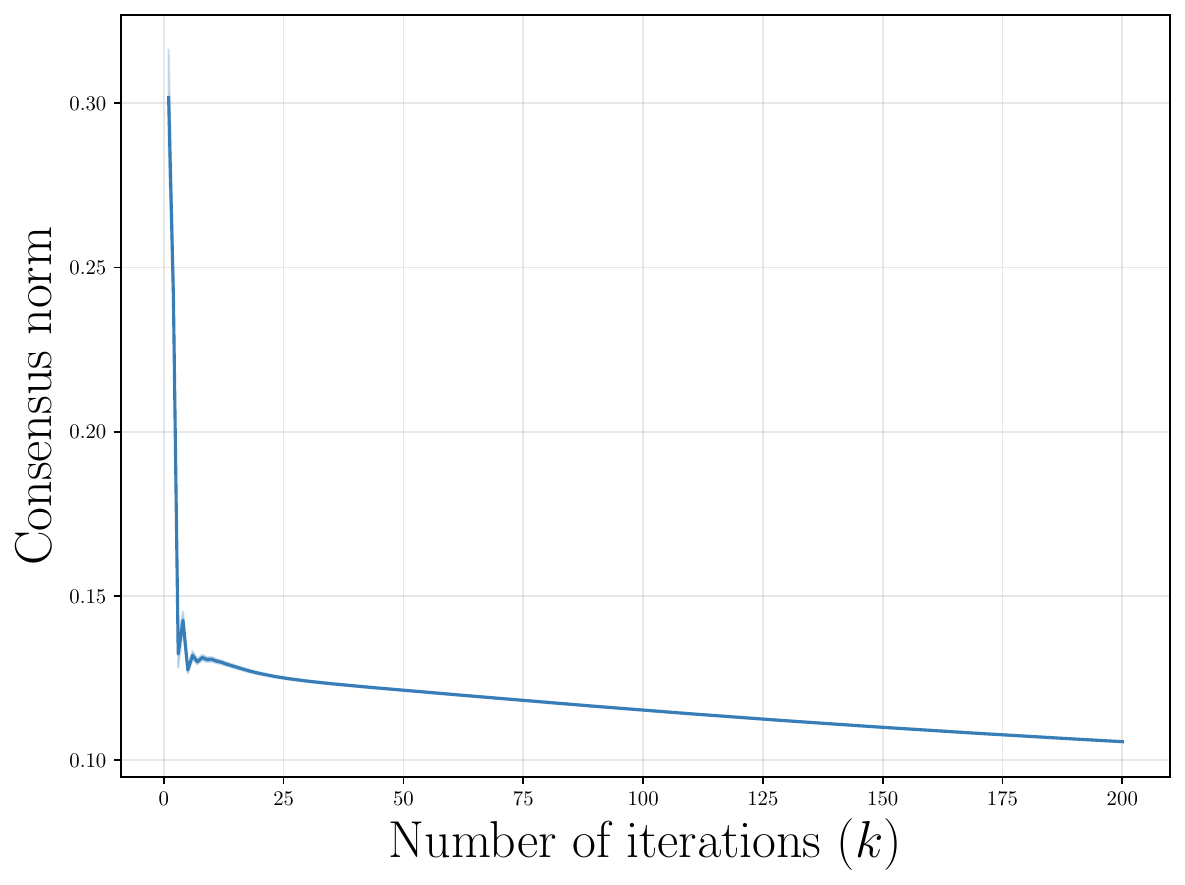}
    \caption{Disentang. penalty: $\mathcal{D}$}
    \label{fig:disentangle_penalty}
  \end{subfigure}
  \hfill
  \begin{subfigure}[t]{0.32\textwidth}
    \centering
    \includegraphics[width=\linewidth]{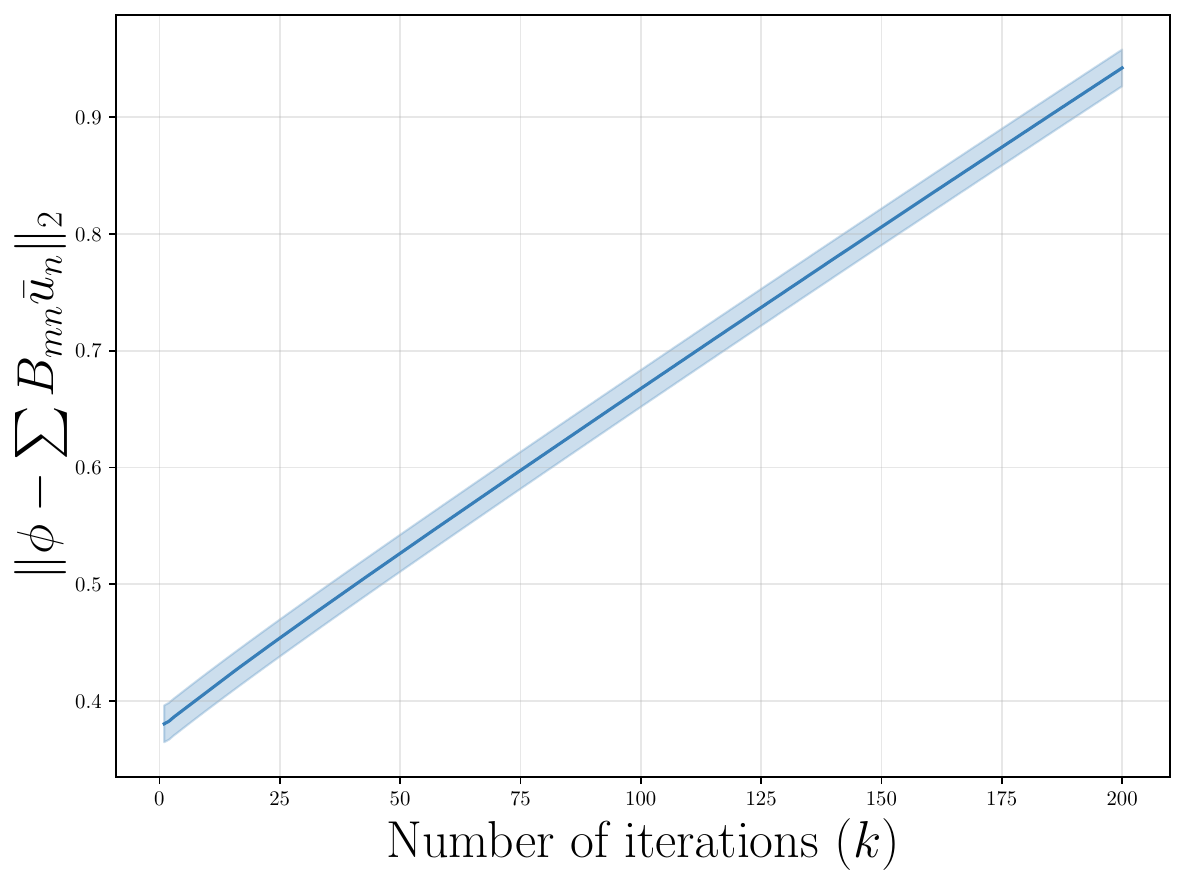}
    \caption{Client 1: $\delta_d$}
    \label{fig:delta_d_client1}
  \end{subfigure}
  \hfill
  \begin{subfigure}[t]{0.32\textwidth}
    \centering
    \includegraphics[width=\linewidth]{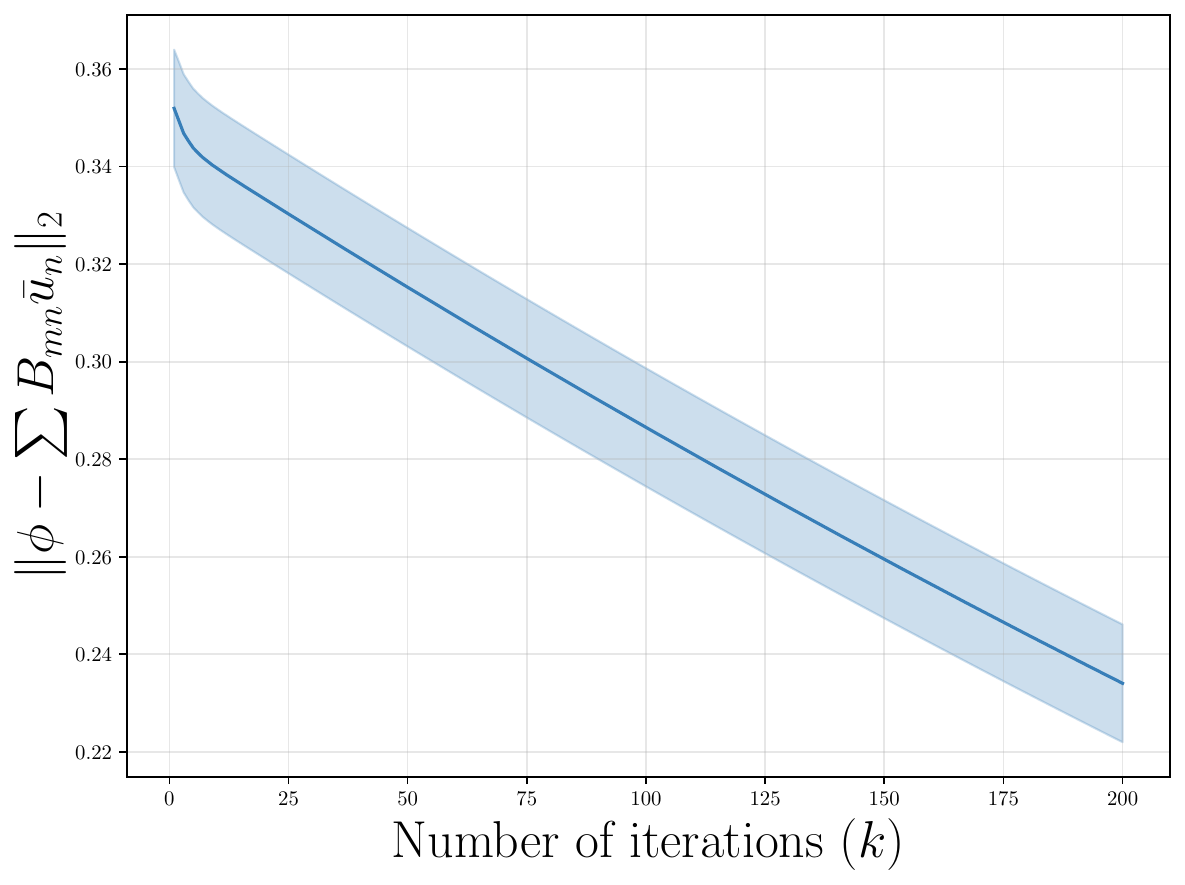}
    \caption{Client 2: $\delta_d$}
    \label{fig:delta_d_client2}
  \end{subfigure}

  \caption{Disentanglement penalty at the server and $\delta_d$ at the clients vs.\ number of iterations.}
  \label{fig:syn_losses}
\end{figure*}
\subsection{Synthetic Datasets}
We establish the performance of our algorithm on a multi-client LTI system described in Section \ref{subsec: multi_client}. Unless otherwise specified, we focus our experiments on a two-client system with $P_m = 2$, $U_m = 2$, and $D_m = 2$ for $m = \{1,2\}$. A one-way directed dependency with $A_{12} = 0$ and $A_{21} \neq 0$ is considered for the ease of interpretation. The elements of input vectors are i.i.d. samples from a normal distribution $\mathcal{N}(0,1)$. Clients use \eqref{eq:aug_state_pred} and \eqref{eq:aug_state_estimate} to compute augmented state predictions and estimates. The proprietary client state estimate $\hat{h}^t_{m,c}$ is estimated in our experiment using a Kalman filter that uses only the client-specific (local) diagonal blocks ($A_{mm}, B_{mm}, C_{mm})$. The communication and training proceed as detailed in Section \ref{Sec: Federated_training_methodology}. The server maintains its state estimate according to \eqref{eq: server_state}, computed using current estimates of $\{\hat{A}_{mn}\}_{n \neq m}$ and $\{\hat{B}_{mn}\}_{n \neq m}$. 


\textbf{Training.} Training curves are obtained after running 10 Monte-Carlo simulations. Each element of initial values of $\phi^0_m, \theta_m^0, \{\hat{A}_{mn}^0\}_{n \neq m}$ and $\{\hat{B}_{mn}^0\}_{n \neq m}$ for all $m = 1,2,  \cdots, M$ are independently sampled from normal distributions. Global loss and client loss for both clients during training are plotted in Fig.
 \ref{fig:syn_global_loss}-\ref{fig:syn_client2_residual}.
The baseline used for comparison is the time-averaged norm of residuals of the proprietary client model ($r^t_{m,c}:= y^t_m - C_{mm} h^t_{m,c}$). As training progresses, the augmented client loss $L_{m,a}$ outperforms the proprietary client loss. Thus, the parameters $\theta_m$ and $\phi_m$, through the augmented client, encode cross-client interdependencies from off-diagonal blocks of $A$ and $B$ - which is not captured by the proprietary client model, validating Claim \ref{claim_phi_theta_combined}.

\textbf{Validation.} The estimated system is validated by plotting the prediction error on a new set of data generated from the same system (Figure \ref{fig:pred_error_valid}). Both our models at the client (ACM) and the server performs better than the proprietary client model on the new data. Validation data contains a mean shift and is non-stationary, which causes a sudden jump in the performance of the model. Performance deterioration in our models due to mean shift is less than the proprietary model and the recovery is better. 
\begin{figure}
  \centering
  \begin{subfigure}[t]{0.48\columnwidth}
    \centering
    \includegraphics[width=\linewidth]{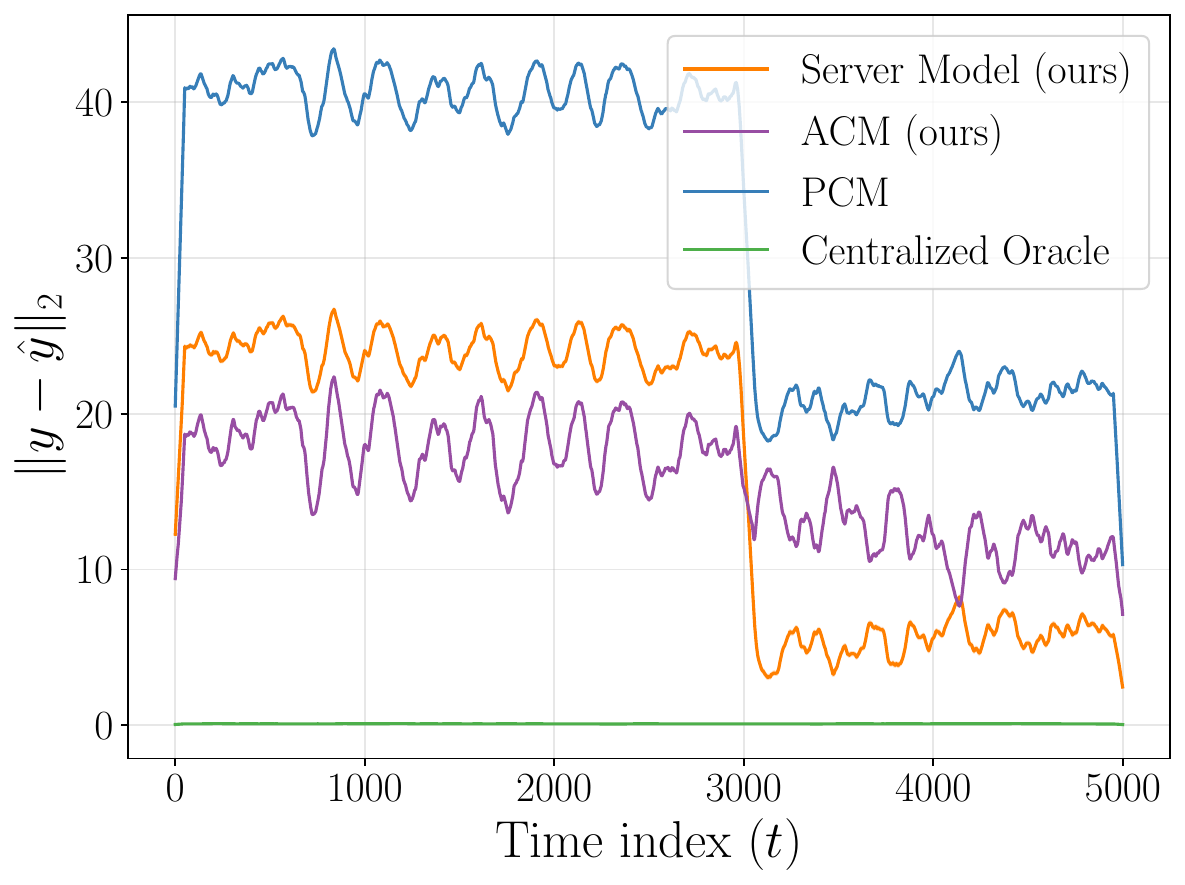}
    \caption{Client 1: Prediction error}
    \label{fig:pred_error_valid_1}
  \end{subfigure}\hfill
  \begin{subfigure}[t]{0.48\columnwidth}
    \centering
    \includegraphics[width=\linewidth]{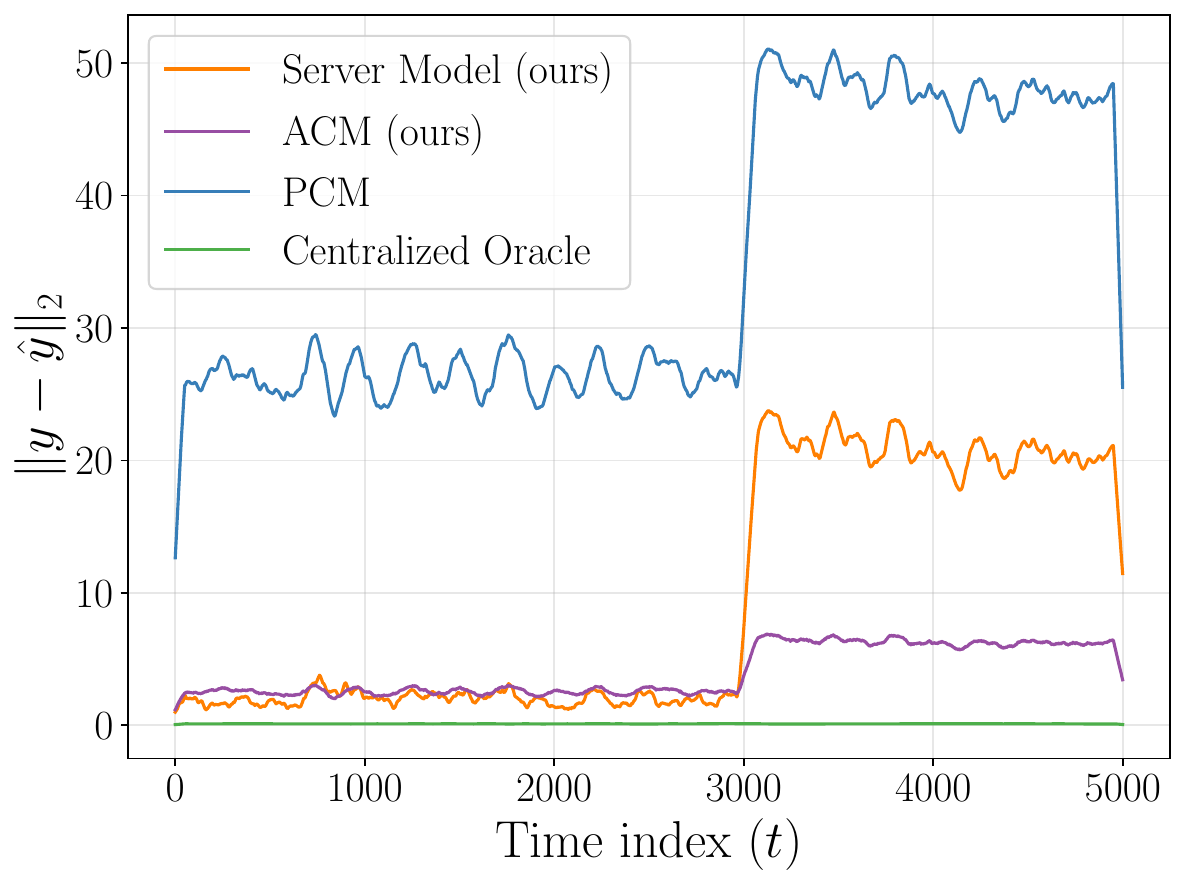}
    \caption{Client 2: Prediction error}
    \label{fig:pred_error_valid_2}
  \end{subfigure}
  \caption{Prediction error on the validation data for different models and the centralized oracle. ACM: Augmented Client Model, PCM: Proprietary Client Model}
  \label{fig:pred_error_valid}
\end{figure}

\textbf{Transfer Function.} Following Section~\ref{sec:causal_id}, the poles, zeros, and singular value spectra of the true and estimated transfer functions \cite{ljung1998system} are shown in Fig.~\ref{fig:tf_plots}. The estimated poles lie within the unit circle, confirming that the decentralized model preserves the stable dynamic structure of the ground truth. While some dispersion in the zeros and deviations in low-frequency singular values remain, the overall frequency-response profile is retained. Combined with the prediction performance in Fig.~\ref{fig:pred_error_valid}, these results indicate that our framework is approximating the underlying causal dynamics, while not yet fully recovering the centralized oracle, as expected under decentralized constraints.

\textbf{Disentanglement.} The disentanglement penalty $\mathcal{D}$ at the server is plotted in Fig. \ref{fig:disentangle_penalty}, and the difference $\delta d := \norm{\phi_m - \frac{1}{T}\sum_{t = 1}^T \left(\sum_{n \neq m}\hat{B}_{mn} u_n^{t-1} \right)}_2$ is plotted for both clients in Fig. \ref{fig:delta_d_client1} and \ref{fig:delta_d_client2}. This demonstrates that not only server effectively disentangle the effect of inputs from the states, but it is also propagated to the respective clients. It was observed that there exists a trade-off between optimality and stability when employing the penalty $\mathcal{D}$. Details are provided in the Appendix~\ref{appendix:Synthetic}. 

\textbf{Scalabillity.} (i) We fix the number of clients $M$, state dimension $P_m$ for each client and vary the measurement dimension $D_m$. (ii) The state dimension $P_m$ and measurement dimension $D_m$ for each client are fixed, and the number of clients is varied. Results of the scalability experiments can be found in the Appendix~\ref{appendix:Synthetic}. 

\subsection{Real-world Datasets}
\textbf{Experimental Setting.} We evaluate our approach using two industrial cybersecurity datasets. The \textbf{HAI} dataset is generated using an industrial control system testbed integrated with a Hardware-in-the-Loop (HIL) simulator that models steam-turbine power generation and pumped-storage hydropower. The system includes four processes: (P1) Boiler, (P2) Turbine, (P3) Water Treatment, and (P4) HIL Simulation, each serving as a client in our experiments. The \textbf{SWaT} dataset originates from a six-stage water treatment plant testbed. Each stage is treated as a client in our experiments. We apply the following preprocessing steps to HAI/SWaT:  
\begin{itemize}[
    leftmargin=*,
    itemsep=0.5pt,
    topsep=1pt
]
\item \textbf{HAI.} All sensor measurements are first normalized to ensure consistent scaling. A subspace identification method is then applied to fit an LTI system, yielding a state-transition matrix $A$, input matrix $B$, and an observation matrix $C$. The dimensionality of the state space $P_m$ is determined using the singular value decay of the Hankel matrix. Since the initial $C$ is not block diagonal, we perform $L_2$-norm based thresholding to assign each state variable to a specific process and construct a block diagonal $C$. Then, $A$ and $B$ are re-estimated via least squares.  

\item \textbf{SWaT.} For each client, only those variables with a Pearson correlation greater than 0.3 with other clients are retained. Then, the same subspace identification and re-estimation steps as in the HAI case are performed.  
\end{itemize}

\textbf{Results.}
Due to space constraints, we provide results for both real-world datasets in the Appendix~\ref{appendix:realWorld}.

\section{Limitations}
We assume LTI models, which may not capture real-world scenarios such as nonlinear or time-varying dynamics. Similarly, reliance on accurate proprietary client models and known local matrices may be unrealistic in practice. Clients can only infer the aggregated effects of counterfactual input. Convergence results only hold up to a bias in reconstruction.

\clearpage
\bibliography{ref.bib}

\begin{thebibliography}{39}
\providecommand{\natexlab}[1]{#1}
\providecommand{\url}[1]{\texttt{#1}}
\expandafter\ifx\csname urlstyle\endcsname\relax
  \providecommand{\doi}[1]{doi: #1}\else
  \providecommand{\doi}{doi: \begingroup \urlstyle{rm}\Url}\fi

\bibitem[Arjovsky et~al.(2019)Arjovsky, Bottou, Gulrajani, and Lopez-Paz]{arjovsky2019irm}
Martin Arjovsky, L{\'e}on Bottou, Ishaan Gulrajani, and David Lopez-Paz.
\newblock Invariant risk minimization.
\newblock \emph{arXiv preprint arXiv:1907.02893}, 2019.

\bibitem[Ascher et~al.(1995)Ascher, Ruuth, and Wetton]{AsherRuuthIMEX1995}
Uri~M. Ascher, Steven~J. Ruuth, and Brian T.~R. Wetton.
\newblock Implicit-explicit methods for time-dependent partial differential equations.
\newblock \emph{SIAM Journal on Numerical Analysis}, 32\penalty0 (3):\penalty0 797--823, 1995.
\newblock \doi{10.1137/0732037}.
\newblock URL \url{https://doi.org/10.1137/0732037}.

\bibitem[Barnett et~al.(2009)Barnett, Barrett, and Seth]{barnett2009granger}
Lionel Barnett, Adam~B Barrett, and Anil~K Seth.
\newblock Granger causality and transfer entropy are equivalent for gaussian variables.
\newblock \emph{Physical Review Letters}, 103\penalty0 (23):\penalty0 238701, 2009.

\bibitem[Bertsekas and Rheinboldt(2014)]{bertsekas2014constrained}
D.P. Bertsekas and W.~Rheinboldt.
\newblock \emph{Constrained Optimization and Lagrange Multiplier Methods}.
\newblock Computer science and applied mathematics. Academic Press, 2014.
\newblock ISBN 9781483260471.
\newblock URL \url{https://books.google.com/books?id=j6LiBQAAQBAJ}.

\bibitem[Brehmer et~al.(2022)Brehmer, De~Haan, Lippe, and Cohen]{brehmer2022weakly}
Johann Brehmer, Pim De~Haan, Phillip Lippe, and Taco~S Cohen.
\newblock Weakly supervised causal representation learning.
\newblock \emph{Advances in Neural Information Processing Systems}, 35:\penalty0 38319--38331, 2022.

\bibitem[Eichler(2010)]{eichler2010}
Michael Eichler.
\newblock Granger causality and path diagrams for multivariate time series.
\newblock \emph{Statistical Modelling}, 10\penalty0 (3):\penalty0 233--255, 2010.

\bibitem[Fraccaro et~al.(2017)Fraccaro, Kamronn, Paquet, and Winther]{fraccaro2017kvae}
Marco Fraccaro, Simon Kamronn, Ulrich Paquet, and Ole Winther.
\newblock A disentangled recognition and nonlinear dynamics model for unsupervised learning.
\newblock \emph{Advances in neural information processing systems}, 30, 2017.

\bibitem[Geweke(1982)]{geweke1982}
John Geweke.
\newblock Measurement of linear dependence and feedback between multiple time series.
\newblock \emph{Journal of the American Statistical Association}, 77\penalty0 (378):\penalty0 304--313, 1982.

\bibitem[Ghosh et~al.(2025)Ghosh, Dwivedi, Tajer, Yeo, and Gifford]{Ghosh}
Shiuli~Subhra Ghosh, Anmol Dwivedi, Ali Tajer, Kyongmin Yeo, and Wesley~M. Gifford.
\newblock Cascading failure prediction via causal inference.
\newblock \emph{IEEE Transactions on Power Systems}, 40\penalty0 (4):\penalty0 3361--3373, 2025.
\newblock \doi{10.1109/TPWRS.2024.3516477}.

\bibitem[Granger(1969)]{granger1969}
Clive~WJ Granger.
\newblock Investigating causal relations by econometric models and cross-spectral methods.
\newblock \emph{Econometrica}, 37\penalty0 (3):\penalty0 424--438, 1969.

\bibitem[Guo et~al.(2024)Guo, Yu, Liu, and Li]{Guo_Yu_Liu_Li_2024}
Xianjie Guo, Kui Yu, Lin Liu, and Jiuyong Li.
\newblock Fedcsl: A scalable and accurate approach to federated causal structure learning.
\newblock \emph{Proceedings of the AAAI Conference on Artificial Intelligence}, 38\penalty0 (11):\penalty0 12235--12243, Mar. 2024.

\bibitem[Huang et~al.(2019)Huang, Zhang, Gong, and Glymour]{huang2019causal}
Biwei Huang, Kun Zhang, Mingming Gong, and Clark Glymour.
\newblock Causal discovery and forecasting in nonstationary environments with state-space models.
\newblock In \emph{International conference on machine learning}, pages 2901--2910. Pmlr, 2019.

\bibitem[Kalman(1960)]{og_kalman}
Rudolph~Emil Kalman.
\newblock A new approach to linear filtering and prediction problems.
\newblock \emph{Transactions of the ASME--Journal of Basic Engineering}, 82\penalty0 (Series D):\penalty0 35--45, 1960.

\bibitem[Krishnan et~al.(2015)Krishnan, Shalit, and Sontag]{krishnan2015dkf}
Rahul~G Krishnan, Uri Shalit, and David Sontag.
\newblock Deep kalman filters.
\newblock \emph{arXiv preprint arXiv:1511.05121}, 2015.

\bibitem[Li et~al.(2024{\natexlab{a}})Li, Pan, and Bareinboim]{li2024disentangled}
Adam Li, Yushu Pan, and Elias Bareinboim.
\newblock Disentangled representation learning in non-markovian causal systems.
\newblock \emph{Advances in Neural Information Processing Systems}, 37:\penalty0 104843--104903, 2024{\natexlab{a}}.

\bibitem[Li et~al.(2024{\natexlab{b}})Li, Ng, Luo, Huang, Chen, Liu, Gu, and Zhang]{Li2024FedCD}
Loka Li, Ignavier Ng, Gongxu Luo, Biwei Huang, Guangyi Chen, Tongliang Liu, Bin Gu, and Kun Zhang.
\newblock Federated causal discovery from heterogeneous data.
\newblock \emph{arXiv preprint arXiv:2402.13241}, 2024{\natexlab{b}}.

\bibitem[Lippe et~al.(2022)Lippe, Magliacane, L{\"o}we, Asano, Cohen, and Gavves]{lippe2022citris}
Phillip Lippe, Sara Magliacane, Sindy L{\"o}we, Yuki~M Asano, Taco Cohen, and Stratis Gavves.
\newblock Citris: Causal identifiability from temporal intervened sequences.
\newblock In \emph{International Conference on Machine Learning}, pages 13557--13603. PMLR, 2022.

\bibitem[Ljung(1998)]{ljung1998system}
L.~Ljung.
\newblock \emph{System Identification: Theory for the User}.
\newblock Pearson Education, 1998.
\newblock ISBN 9780132440530.
\newblock URL \url{https://books.google.com/books?id=fYSrk4wDKPsC}.

\bibitem[Mastakouri et~al.(2021)Mastakouri, Rohde, and Sch{\"o}lkopf]{mastakouri2021necessary}
Atalanti Mastakouri, David Rohde, and Bernhard Sch{\"o}lkopf.
\newblock Necessary and sufficient conditions for causal feature selection in time series with latent confounders.
\newblock In \emph{NeurIPS}, 2021.

\bibitem[Mathur and Tippenhauer(2016)]{swat_dataset}
Aditya~P. Mathur and Nils~Ole Tippenhauer.
\newblock Swat: a water treatment testbed for research and training on ics security.
\newblock In \emph{2016 International Workshop on Cyber-physical Systems for Smart Water Networks (CySWater)}, pages 31--36, 2016.
\newblock \doi{10.1109/CySWater.2016.7469060}.

\bibitem[McMahan et~al.(2017)McMahan, Moore, Ramage, Hampson, and y~Arcas]{mcmahan2017}
Brendan McMahan, Eider Moore, Daniel Ramage, Seth Hampson, and Blaise y~Arcas.
\newblock Communication-efficient learning of deep networks from decentralized data.
\newblock In \emph{Artificial Intelligence and Statistics (AISTATS)}, 2017.

\bibitem[Mian et~al.(2022)Mian, Kaltenpoth, and Kamp]{Mian22Regret}
Osman Mian, David Kaltenpoth, and Michael Kamp.
\newblock Regret-based federated causal discovery.
\newblock In Thuc~Duy Le, Lin Liu, Emre Kıcıman, Sofia Triantafyllou, and Huan Liu, editors, \emph{Proceedings of The KDD'22 Workshop on Causal Discovery}, volume 185 of \emph{Proceedings of Machine Learning Research}, pages 61--69. PMLR, 15 Aug 2022.

\bibitem[Mian et~al.(2023)Mian, Kaltenpoth, Kamp, and Vreeken]{Mian23PrivacyFCD}
Osman Mian, David Kaltenpoth, Michael Kamp, and Jilles Vreeken.
\newblock Nothing but regrets—privacy-preserving federated causal discovery.
\newblock In \emph{International Conference on Artificial Intelligence and Statistics}, pages 8263--8278. PMLR, 2023.

\bibitem[Miladinovi{\'c} et~al.(2019)Miladinovi{\'c}, Gondal, Sch{\"o}lkopf, Buhmann, and Bauer]{miladinovic2019disentangled}
{\DJ}or{\dj}e Miladinovi{\'c}, Muhammad~Waleed Gondal, Bernhard Sch{\"o}lkopf, Joachim~M Buhmann, and Stefan Bauer.
\newblock Disentangled state space representations.
\newblock \emph{arXiv preprint arXiv:1906.03255}, 2019.

\bibitem[Mohanty et~al.(2025)Mohanty, Mohamed, Ramanan, and Gebraeel]{ICLR2025_fedGC}
Ayush Mohanty, Nazal Mohamed, Paritosh Ramanan, and Nagi Gebraeel.
\newblock Federated granger causality learning for interdependent clients with state space representation.
\newblock In \emph{The Thirteenth International Conference on Learning Representations}, 2025.

\bibitem[Murphy(2002)]{murphy2002}
Kevin~Patrick Murphy.
\newblock \emph{Dynamic Bayesian networks: representation, inference and learning}.
\newblock PhD thesis, UC Berkeley, 2002.

\bibitem[Pearl(2009)]{pearl2009causality}
J.~Pearl.
\newblock \emph{Causality}.
\newblock Causality: Models, Reasoning, and Inference. Cambridge University Press, 2009.
\newblock ISBN 9780521895606.
\newblock URL \url{https://books.google.com/books?id=f4nuexsNVZIC}.

\bibitem[Peters et~al.(2016)Peters, B{\"u}hlmann, and Meinshausen]{peters2016icp}
Jonas Peters, Peter B{\"u}hlmann, and Nicolai Meinshausen.
\newblock Causal inference by using invariant prediction: identification and confidence intervals.
\newblock \emph{Journal of the Royal Statistical Society Series B: Statistical Methodology}, 78\penalty0 (5):\penalty0 947--1012, 2016.

\bibitem[Pournaras et~al.(2020)Pournaras, Taormina, Thapa, Galelli, Palleti, and Kooij]{Pournaras}
Evangelos Pournaras, Riccardo Taormina, Manish Thapa, Stefano Galelli, Venkata Palleti, and Robert Kooij.
\newblock Cascading failures in interconnected power-to-water networks.
\newblock \emph{SIGMETRICS Perform. Eval. Rev.}, 47\penalty0 (4):\penalty0 16–20, May 2020.
\newblock ISSN 0163-5999.
\newblock \doi{10.1145/3397776.3397781}.
\newblock URL \url{https://doi.org/10.1145/3397776.3397781}.

\bibitem[Ruiz-Tagle et~al.(2022)Ruiz-Tagle, Lopez-Droguett, and Groth]{RUIZTAGLE2022108785}
Andres Ruiz-Tagle, Enrique Lopez-Droguett, and Katrina~M. Groth.
\newblock A novel probabilistic approach to counterfactual reasoning in system safety.
\newblock \emph{Reliability Engineering \& System Safety}, 228:\penalty0 108785, 2022.
\newblock ISSN 0951-8320.

\bibitem[Schölkopf et~al.(2021)Schölkopf, Locatello, Bauer, Ke, Kalchbrenner, Goyal, and Bengio]{scholkopf2021causal}
Bernhard Schölkopf, Francesco Locatello, Stefan Bauer, Nan~Rosemary Ke, Nal Kalchbrenner, Anirudh Goyal, and Yoshua Bengio.
\newblock Toward causal representation learning.
\newblock \emph{Proceedings of the IEEE}, 109\penalty0 (5):\penalty0 612--634, 2021.

\bibitem[Shin et~al.(2021)Shin, Lee, Yun, and Min]{hai_2103}
Hyeok-Ki Shin, Woomyo Lee, Jeong-Han Yun, and Byung-Gi Min.
\newblock Two ics security datasets and anomaly detection contest on the hil-based augmented ics testbed.
\newblock In \emph{Cyber Security Experimentation and Test Workshop}, CSET '21, page 36–40, New York, NY, USA, 2021. Association for Computing Machinery.
\newblock ISBN 9781450390651.
\newblock \doi{10.1145/3474718.3474719}.
\newblock URL \url{https://doi.org/10.1145/3474718.3474719}.

\bibitem[Tang et~al.(2023)Tang, Liu, Zhou, and Ding]{tang2023causality}
Wenbing Tang, Jing Liu, Yuan Zhou, and Zuohua Ding.
\newblock Causality-guided counterfactual debiasing for anomaly detection of cyber-physical systems.
\newblock \emph{IEEE Transactions on Industrial Informatics}, 20\penalty0 (3):\penalty0 4582--4593, 2023.

\bibitem[Todo et~al.(2023)Todo, Selmani, Laurent, and Loubes]{todo2023counterfactual}
William Todo, Merwann Selmani, B{\'e}atrice Laurent, and Jean-Michel Loubes.
\newblock Counterfactual explanation for multivariate times series using a contrastive variational autoencoder.
\newblock In \emph{ICASSP 2023-2023 IEEE International Conference on Acoustics, Speech and Signal Processing (ICASSP)}, pages 1--5. IEEE, 2023.

\bibitem[Wang et~al.(2024)Wang, Chen, Tang, Wu, and Zhu]{wang2024disentangled}
Xin Wang, Hong Chen, Si'ao Tang, Zihao Wu, and Wenwu Zhu.
\newblock Disentangled representation learning.
\newblock \emph{IEEE Transactions on Pattern Analysis and Machine Intelligence}, 46\penalty0 (12):\penalty0 9677--9696, 2024.

\bibitem[Weng et~al.(2025)Weng, Han, Jiang, and Liu]{weng2025sde}
Zixuan Weng, Jindong Han, Wenzhao Jiang, and Hao Liu.
\newblock Sde: A simplified and disentangled dependency encoding framework for state space models in time series forecasting.
\newblock In \emph{Proceedings of the 31st ACM SIGKDD Conference on Knowledge Discovery and Data Mining V. 2}, pages 3168--3179, 2025.

\bibitem[Yang et~al.(2024)Yang, He, Wang, Yu, Domeniconi, and Zhang]{Yang_He_Wang_Yu_Domeniconi_Zhang_2024}
Dezhi Yang, Xintong He, Jun Wang, Guoxian Yu, Carlotta Domeniconi, and Jinglin Zhang.
\newblock Federated causality learning with explainable adaptive optimization.
\newblock \emph{Proceedings of the AAAI Conference on Artificial Intelligence}, 38\penalty0 (15):\penalty0 16308--16315, Mar. 2024.

\bibitem[Ye et~al.(2024)Ye, Amini, and Zhou]{IEEE2024FCD1}
Qiaoling Ye, Arash~A. Amini, and Qing Zhou.
\newblock Federated learning of generalized linear causal networks.
\newblock \emph{IEEE Transactions on Pattern Analysis and Machine Intelligence}, 46\penalty0 (10):\penalty0 6623--6636, 2024.

\bibitem[Zhao et~al.(2025)Zhao, Yu, Xiang, Guo, and Cao]{IEEE2024FCD2}
Yongsheng Zhao, Kui Yu, Guodu Xiang, Xianjie Guo, and Fuyuan Cao.
\newblock Fedece: Federated estimation of causal effect based on causal graphical modeling.
\newblock \emph{IEEE Transactions on Artificial Intelligence}, 6\penalty0 (8):\penalty0 2327--2341, 2025.

\end{thebibliography}


\onecolumn
\begin{center}
    {\LARGE\bfseries
    Federated Causal Representation Learning in State-Space Systems\\
    for Decentralized Counterfactual Reasoning\par}
    \vspace{0.5em}
    {\huge\bfseries (Supplementary Material)\par}
    \vspace{1.5em}
    \rule{\textwidth}{2pt}
\end{center}
\vspace{1em}

\addtocontents{toc}{\protect\etocsettocdepth{2}}
\begingroup
\small
\setlength{\parskip}{10pt}
\setlength{\itemsep}{5pt}
\etocsettocstyle{}{}
\tableofcontents
\endgroup
\clearpage
\section{Privacy Analysis}\label{appendix:privacy}
\subsection{Notation and Sensitivity}
\label{app:dp:notation}
\begin{definition}[\textbf{Local dataset and horizon}]
For each client $m$, the private dataset is
$\mathcal{D}_m = \{ y_m^t \}_{t=1}^T,
$ collected over a fixed horizon $T \in \mathbb{N}$. The union $\mathcal{D} = \{\mathcal{D}_1, \dots, \mathcal{D}_M\}$ denotes all clients’ measurement data.
\end{definition}

\begin{definition}[\textbf{Neighboring datasets}]
Two global datasets $\mathcal{D}$ and $\mathcal{D}'$ are \emph{neighbors} if they differ at a single measurement $y_{m^\star}^{t^\star}$ for some client $m^\star$ and time $t^\star$.
\end{definition}

\begin{assumption}[\textbf{Bounded and clipped signals}]\label{assumption:bounded_clipped}
Input $u_m^t$ and output (or measurements) $y_m^t$ are bounded (to ensure stability of state-spaces), and are clipped to those known bounds:
$\|y_m^t\|_2 \le R_y, \hspace{0.1cm} \text{and} \hspace{0.1cm} \|u_m^t\|_2 \le R_u.$
The maximum magnitude across all input/output signals is denoted by
$R_{\max} = \max(R_y, R_u).$
\end{assumption}
\textbf{\textit{Rationale}:} Clipping in Assumption \ref{assumption:bounded_clipped} ensures bounded sensitivity prior to noise addition.
\begin{assumption}[\textbf{Local estimator contractivity}]\label{assumption:contractivity}
Each client’s augmented estimator is contractive with constants $L_m > 0$ and $\beta_m \in (0,1)$:
$\|h_{m,a}^{t} - h_{m,a}^{t\,\prime}\|_2
\le L_m \beta_m^{t-s} \|y_m^s - y_m^{s\,\prime}\|_2, \quad \forall s \le t.$
\end{assumption}
\textbf{\textit{Rationale}:} Assumption \ref{assumption:contractivity} ensures bounded propagation of perturbations in local measurements.

Let $z_m^t = [\hat{h}_{m,c}^t; \hat{h}_{m,a}^t; h_{m,a}^t; u_m^t] \in \mathbb{R}^{3P_m + U_m}$ be the message communicated from client $m$ to the server. 
\begin{definition}[\textbf{Message generation map}]
The deterministic, pre-noise mapping from measurements to the transmitted message is given as
$\mathcal{M}_{m,\mathrm{msg}}^t: (y_m^{1:t}) \mapsto z_m^t.$ This map depends on the internal dynamics of the proprietary and augmented estimators.
\end{definition}

\begin{lemma}[\textbf{$\ell_2$-Sensitivity of client messages}]
Under bounded measurements and contractive estimators, there exists $\kappa_m > 0$ such that for any neighboring datasets $\mathcal{D}_m$ and $\mathcal{D}_m'$ differing at one $y_m^{t^\star}$,
\[
\|\mathcal{M}_{m,\mathrm{msg}}^t(\mathcal{D}_m) - \mathcal{M}_{m,\mathrm{msg}}^t(\mathcal{D}_m')\|_2
\le
\Delta_{m,\mathrm{msg}} = \kappa_m R_{\max}.
\]
\end{lemma}
\begin{remark}[\textbf{Role of block-diagonal $C$}]
In our paper, the observation matrix $C$ is block-diagonal. When $C$ is block-diagonal (so that $y_m^t = C_{mm} h_m^t$), each client’s $\kappa_m$ depends only on its own local block. 
\end{remark}
The server communicates the gradient $g_m^t := \nabla_{h_{m,a}^t} L_s$ to client $m$. Given the expression of server loss $L_s$ in Eq (13), the analytical expression of $g_m$ can be computed as, 
\[
g_m^t = \frac{2}{T}
\Big(
h_{m,a}^t - [A_{mm} \hat{h}_{m,c}^{t-1}
+ \sum_{n \ne m} \hat{A}_{mn} \hat{h}_{n,c}^{t-1}
+ B_{mm} u_m^{t-1}
+ \sum_{n \ne m} \hat{B}_{mn} u_n^{t-1}]
\Big),
\]
\begin{lemma}[\textbf{$\ell_2$-Sensitivity of server gradients}]
Given the expression of $g_m^t$ above, the sensitivity with respect to one changed measurement $y_{m'}^{t^\star}$ (for some client $m'$ and time $t^\star$) is bounded by
\[
\|g_m^t - g_m^{t\,\prime}\|_2
\le
\Delta_{m,\mathrm{grad}}
= \frac{2}{T}(1+\|A_{mm}\|)\,\kappa_{m,\mathrm{mix}}\,R_{\max},
\]
where the constant $\kappa_{m,\mathrm{mix}}$ is given explicitly as
\[
\kappa_{m,\mathrm{mix}}
=
L_{m'}\!
\left(
\|C_{mm}\|
+
\sum_{n \neq m}
\!\big(
\|A_{mn}\|\beta_{n}^{t-t^\star}
+
\|B_{mn}\|
\big)
\right).
\]
where, $L_{m'}$ arises from the local contractivity bound in Assumption~\ref{assumption:contractivity}.
\end{lemma}

\subsection{Noise Mechanisms and Calibration}
\label{app:dp:mechanisms}

\begin{definition}[\textbf{Clipping constants}]
Each client enforces $\ell_2$-norm clipping thresholds
$C_{\mathrm{msg}}, C_{\mathrm{grad}} > 0,$
so that $\|z_m^t\|_2 \le C_{\mathrm{msg}}$ and $\|g_m^t\|_2 \le C_{\mathrm{grad}}$ for all $m,t$. These are chosen such that
\[
C_{\mathrm{msg}} \ge \Delta_{m,\mathrm{msg}}, \qquad
C_{\mathrm{grad}} \ge \Delta_{m,\mathrm{grad}}.
\]
where $\Delta_{m,\mathrm{msg}}$ and $\Delta_{m,\mathrm{grad}}$ denote the $\ell_2$-sensitivities of the client message and server gradients, respectively.
\end{definition}

\begin{definition}[\textbf{Perturbation}]
Privacy is achieved by adding independent Gaussian noise:
\[
\tilde{z}_m^t = z_m^t + \mathcal{N}(0, \sigma_{\mathrm{msg}}^2 C_{\mathrm{msg}}^2 I), \quad
\tilde{g}_m^t = g_m^t + \mathcal{N}(0, \sigma_{\mathrm{grad}}^2 C_{\mathrm{grad}}^2 I).
\]
Noise multipliers $\sigma_{\mathrm{msg}}$ and $\sigma_{\mathrm{grad}}$ are selected to achieve target privacy parameters $(\varepsilon, \delta)$.
\end{definition}

\begin{proposition}[\textbf{Per-release Gaussian DP}]
For the client $m$'s message and server gradient release mechanisms, 
the sufficient noise multipliers satisfying $(\varepsilon,\delta)$-DP are
\[
\sigma_{\mathrm{msg}}
\ge
\frac{\Delta_{m,\mathrm{msg}}}{C_{\mathrm{msg}}}
\sqrt{2\ln\!\left(\frac{1.25}{\delta}\right)}\,
\frac{1}{\varepsilon_{\mathrm{msg}}},
\qquad
\sigma_{\mathrm{grad}}
\ge
\frac{\Delta_{m,\mathrm{grad}}}{C_{\mathrm{grad}}}
\sqrt{2\ln\!\left(\frac{1.25}{\delta}\right)}\,
\frac{1}{\varepsilon_{\mathrm{grad}}}.
\]
\end{proposition}

The Gaussian perturbation scales $\sigma_{\mathrm{msg}}$ and $\sigma_{\mathrm{grad}}$ quantify privacy strength for the message and gradient channels.  
Larger noise yields stronger privacy at the expense of higher variance in estimated parameters.  
Clipping thresholds $C_{\mathrm{msg}}$ and $C_{\mathrm{grad}}$ cap the maximum influence of any single measurement $y_m^t$, ensuring the injected noise dominates any one individual’s contribution.

\subsection{Composition and Post-Processing}
\label{app:dp:composition}

\begin{proposition}[\textbf{Sequential and joint composition of Gaussian mechanisms}]
Let each client $m$ release privatized messages $\tilde{z}_m^t$ and gradients $\tilde{g}_m^t$ over $R$ communication rounds.  
If each mechanism satisfies $(\varepsilon_{\mathrm{msg}}, \delta_{\mathrm{msg}})$-DP for messages and 
$(\varepsilon_{\mathrm{grad}}, \delta_{\mathrm{grad}})$-DP for gradients, then the composition across all rounds satisfies,
\[\big(\varepsilon_{\mathrm{total}}, \delta_{\mathrm{total}}\big)
=
\big(R\,(\varepsilon_{\mathrm{msg}} + \varepsilon_{\mathrm{grad}}),\; R\,(\delta_{\mathrm{msg}} + \delta_{\mathrm{grad}})\big).\]

\end{proposition}

\begin{corollary}[\textbf{Post-processing invariance}]
All downstream quantities, including the server’s learned matrices $\{\hat{A}_{mn}, \hat{B}_{mn}\}$, client parameters $\{\theta_m, \phi_m\}$, and counterfactual outputs, retain the same $(\varepsilon,\delta)$ privacy guarantee as the privatized training transcript.
\end{corollary}

\clearpage
\section{Gradients Communicated by the Server}
\begin{proposition}
    Using the chain rule, the derivatives of the server loss $L_s$ with respect to the augmented client parameters of client $m$ (i.e., $\phi_m$ and $\theta_m$) are
    \begin{align*}
        \nabla_{\theta_m} L_s &= \sum_{t=1}^T \big(A_{mm}^\top \nabla_{h^t_{m,a}} L_s + \nabla_{\hat h^{t-1}_{m,a}} L_s\big)\,{y_m^{t-1}}^{\!\top}, \hspace{0.1cm} \text{and} \\
        \nabla_{\phi_m} L_s &= \sum_{t=1}^T \nabla_{h^t_{m,a}} L_s .
    \end{align*}
\end{proposition}

\begin{proof}
From the augmented client model,
\[
h^t_{m,a}=A_{mm}\hat h^{t-1}_{m,a}+b_m^{t-1}+\phi_m,\qquad
\hat h^{t-1}_{m,a}=c_m^{t-1}+\theta_m y_m^{t-1}.
\]
Consider the proxy loss, 
\[
\ell(h^t_{m,a},\hat h^{t-1}_{m,a}) \coloneqq \frac1T \sum_{t=1}^T
\Big(\,\|h^t_{m,a}-X^t_m\|_2^2+\xi\|A_{mm}\hat h^{t-1}_{m,a}-Z^{t-1}_m\|_2^2\Big),
\]
This proxy loss follows the same computational graph as the client model with, 
\[
h^t = A_{mm}\hat h^{t-1}+b^{t-1}+\phi,\qquad
\hat h^{t-1}=c^{t-1}+\theta y^{t-1}.
\]

Using \textbf{Einstein index notation}, repeated Latin indices are summed over automatically.  
Indices $p,k,r\in\{1,\dots,P_m\}$ denote state components and $j,s\in\{1,\dots,D_m\}$ denote measurement components.  
The Kronecker delta is $\delta_{ij}$, and subscripts on a vector or matrix denote its components.

\paragraph{Derivative with respect to $\theta_m$.}
By the chain rule,
\begin{align*}
\frac{\partial \ell}{\partial \theta_{m,rs}}
&=\frac1T\sum_{t=1}^T\sum_{p=1}^{P_m}
\left(
\frac{\partial \ell}{\partial h^t_{m,a,p}}\frac{\partial h^t_{m,a,p}}{\partial \theta_{m,rs}}
+
\frac{\partial \ell}{\partial \hat h^{t-1}_{m,a,p}}\frac{\partial \hat h^{t-1}_{m,a,p}}{\partial \theta_{m,rs}}
\right).
\end{align*}
From 
\[
h^t_{m,a,p}=A_{mm,pk}\hat h^{t-1}_{m,a,k}+b^{t-1}_{m,p}+\phi_{m,p}, 
\qquad 
\hat h^{t-1}_{m,a,k}=c^{t-1}_{m,k}+\theta_{m,kj}y^{t-1}_{m,j},
\]
we have
\[
\frac{\partial \hat h^{t-1}_{m,a,k}}{\partial \theta_{m,rs}}=\delta_{kr}y^{t-1}_{m,s},
\qquad
\frac{\partial h^t_{m,a,p}}{\partial \theta_{m,rs}}
=A_{mm,pk}\frac{\partial \hat h^{t-1}_{m,a,k}}{\partial \theta_{m,rs}}
=A_{mm,pk}\delta_{kr}y^{t-1}_{m,s}=A_{mm,pr}y^{t-1}_{m,s}.
\]
Hence,
\begin{align*}
\frac{\partial \ell}{\partial \theta_{m,rs}}
&=\frac1T\sum_{t=1}^T
\left[
\Big(\tfrac{\partial \ell}{\partial h^t_{m,a,p}}\Big)A_{mm,pr}
+
\Big(\tfrac{\partial \ell}{\partial \hat h^{t-1}_{m,a,p}}\Big)\delta_{pr}
\right] y^{t-1}_{m,s}\\
&=\frac1T\sum_{t=1}^T
\left[
\big(A_{mm}^\top\nabla_{h^t_{m,a}}\ell\big)_r
+
\big(\nabla_{\hat h^{t-1}_{m,a}}\ell\big)_r
\right] y^{t-1}_{m,s}.
\end{align*}
Collecting the $(r,s)$-entries, the matrix form is
\[
\nabla_{\theta_m}\ell
=\frac1T\sum_{t=1}^T \big(A_{mm}^\top\nabla_{h^t_{m,a}}\ell+\nabla_{\hat h^{t-1}_{m,a}}\ell\big)\,{(y^{t-1}_m)}^\top.
\]

\paragraph{Derivative with respect to $\phi_m$.}
Again by the chain rule,
\begin{align*}
\frac{\partial \ell}{\partial \phi_{m,a}}
&=\frac1T\sum_{t=1}^T\sum_{p=1}^{P_m}
\frac{\partial \ell}{\partial h^t_{m,a,p}}\frac{\partial h^t_{m,a,p}}{\partial \phi_{m,a}}
=\frac1T\sum_{t=1}^T\sum_{p=1}^{P_m}
\frac{\partial \ell}{\partial h^t_{m,a,p}}\,\delta_{pa}
=\frac1T\sum_{t=1}^T [\nabla_{h^t_{m,a}}\ell]_a,
\end{align*}
or in matrix form,
\[
\nabla_{\phi_m}\ell=\frac1T\sum_{t=1}^T \nabla_{h^t_{m,a}}\ell.
\]

Finally, replacing $\ell$ by $L_s$ (the factor $1/T$ is an inconsequential constant) yields
\[
\nabla_{\theta_m} L_s = \sum_{t=1}^T \big(A_{mm}^\top \nabla_{h^t_{m,a}} L_s + \nabla_{\hat h^{t-1}_{m,a}} L_s\big)\,{y_m^{t-1}}^{\!\top},
\qquad
\nabla_{\phi_m} L_s = \sum_{t=1}^T \nabla_{h^t_{m,a}} L_s,
\]
as claimed.
\end{proof}

\clearpage
\section{Experiments}
\subsection{Additional Results on Synthetic Datasets}\label{appendix:Synthetic}
The results plotted in Fig. \ref{fig:syn_losses} are obtained by directly optimizing the server loss (Eq~\ref{eq:server_loss}) with the disentanglement penalty $\xi$ and tuning the learning rates. It was observed that direct implementation of the penalty-based disentanglement could make the framework unstable for large values $\xi$ - which is a requirement for effective disentanglement. This issue can be mitigated by implementing an Augmented Lagrangian method (\cite{bertsekas2014constrained}), which results in better disentanglement while keeping the framework stable. However, the Augmented Lagrangian method might converge sub-optimal solution compared to the penalty-based disentanglement. We explain this in the next paragraph.

\paragraph{Augmented Lagrangian with IMEX}
To robustify training under large penalty $\xi$ and improve local disentanglement, we also experimented with an Augmented Lagrangian (AL) formulation \cite{bertsekas2014constrained} combined with an implicit--explicit (IMEX) time discretization \cite{AsherRuuthIMEX1995}. In AL (as opposed to the \textbf{Penalty Method (PM)} which optimizes $L_s$ at the server), the hard alignment/consensus condition is enforced via a quadratic penalty and a dual term $\lambda$ (similar to a Lagrangian multiplier), i.e., we update the server and client variables by minimizing the AL server loss,
\[
\mathcal L_{\mathrm{AL}}
= \frac{1}{T}\sum_{t=1}^{T}\sum_{m=1}^{M}
\Big(
  \|r_{m,s}^{\,t}\|_2^2
  + \lambda_m^{\,t\top}\mathcal D_m^{\,t}
  + \tfrac{\rho}{2}\,\|\mathcal D_m^{\,t}\|_2^2
\Big)\!,
\]
where
\begin{equation*}
    \mathcal D_m^{\,t}
:= A_{mm}\!\big(\hat h_{m,a}^{\,t-1}-\hat h_{m,c}^{\,t-1}\big)
   - \sum_{n\neq m} \hat A_{mn}\,\hat h_{n,c}^{\,t-1}
\end{equation*}
and $\rho>0$ is the penalty. Practically, this stabilizes training under large penalties and promotes cleaner separation of client-specific and cross-client effects (``disentanglement''). The trade-off is that the IMEX/AL step solves a proximal, constraint-regularized surrogate rather than the original unconstrained gradient step at each iteration; hence the fixed point can be slightly biased and, in that sense, suboptimal relative to our primary (non-AL) method, even though it is often better behaved numerically. 

For all our experiments, other than for this part, we have implemented the Penalty Method. In Figure \ref{fig:PM_vs_AL}, we have compared the Penalty Method (PM) and Augmented Lagrangian (AL). The learning curve of AL exhibits larger oscillations but achieves better disentanglement quickly at both server and clients (Figures \ref{fig:comp_disentanglement}, \ref{fig:comp_delta_d_1} and \ref{fig:comp_delta_d_2}). This suggests in practice we can use a hybrid schedule: use AL as a warm start to quickly enforce disentanglement, then switch to PM to achieve a smoother descent of $L_s$ and a marginally better final objective.

\begin{figure*}[t]
\centering

\begin{subfigure}{0.32\textwidth}
  \includegraphics[width=\linewidth]{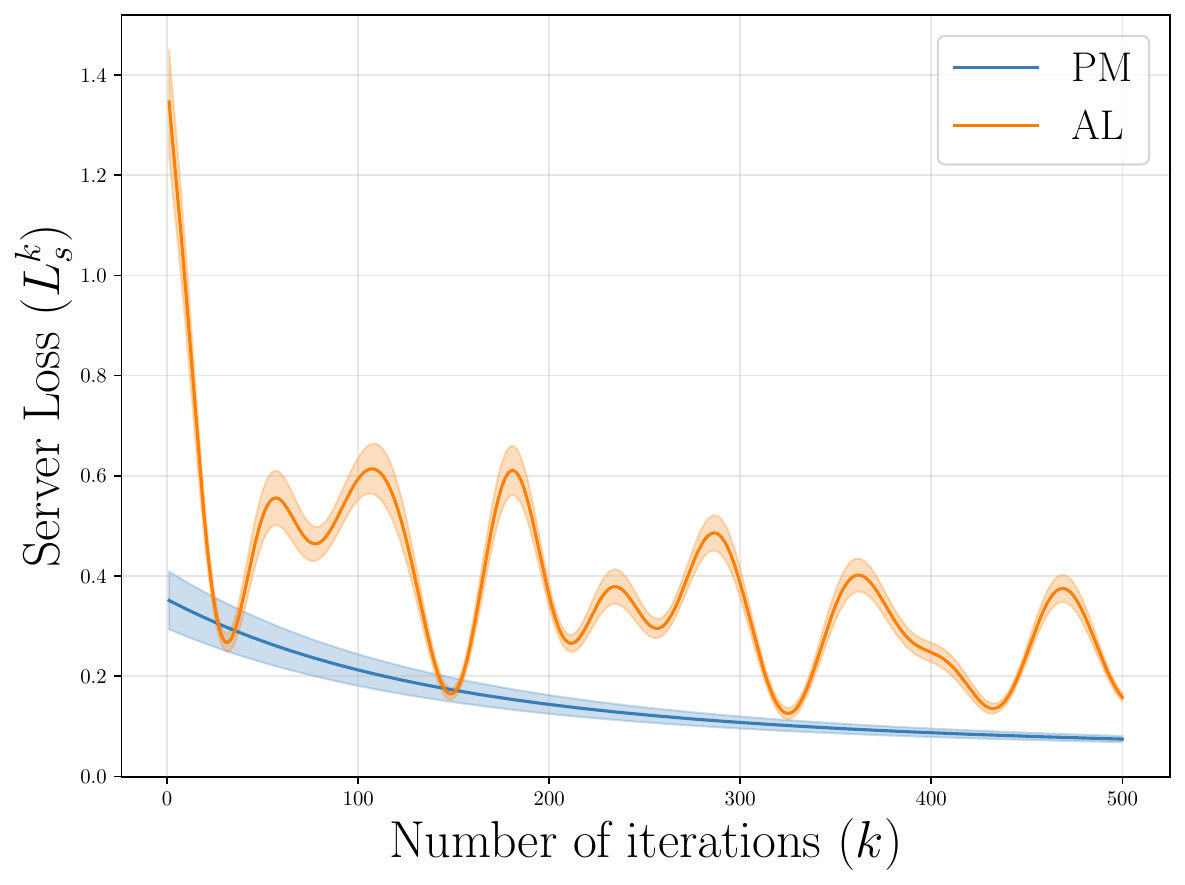}
  \caption{Server Loss}\label{fig:comparison_server_loss}
\end{subfigure}\hfill
\begin{subfigure}{0.32\textwidth}
  \includegraphics[width=\linewidth]{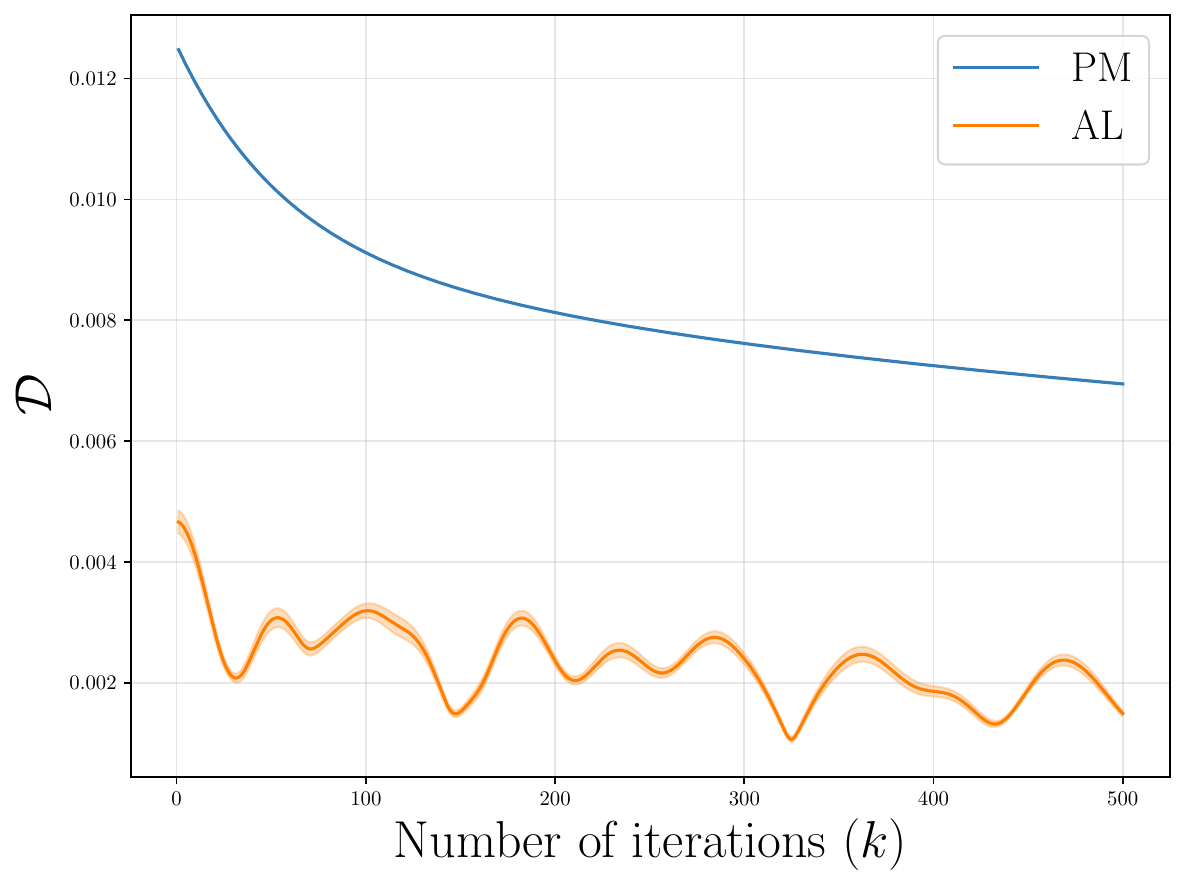}
  \caption{Disentanglement constraint $\mathcal{D}$}\label{fig:comp_disentanglement}
\end{subfigure}\hfill
\begin{subfigure}{0.32\textwidth}
  \includegraphics[width=\linewidth]{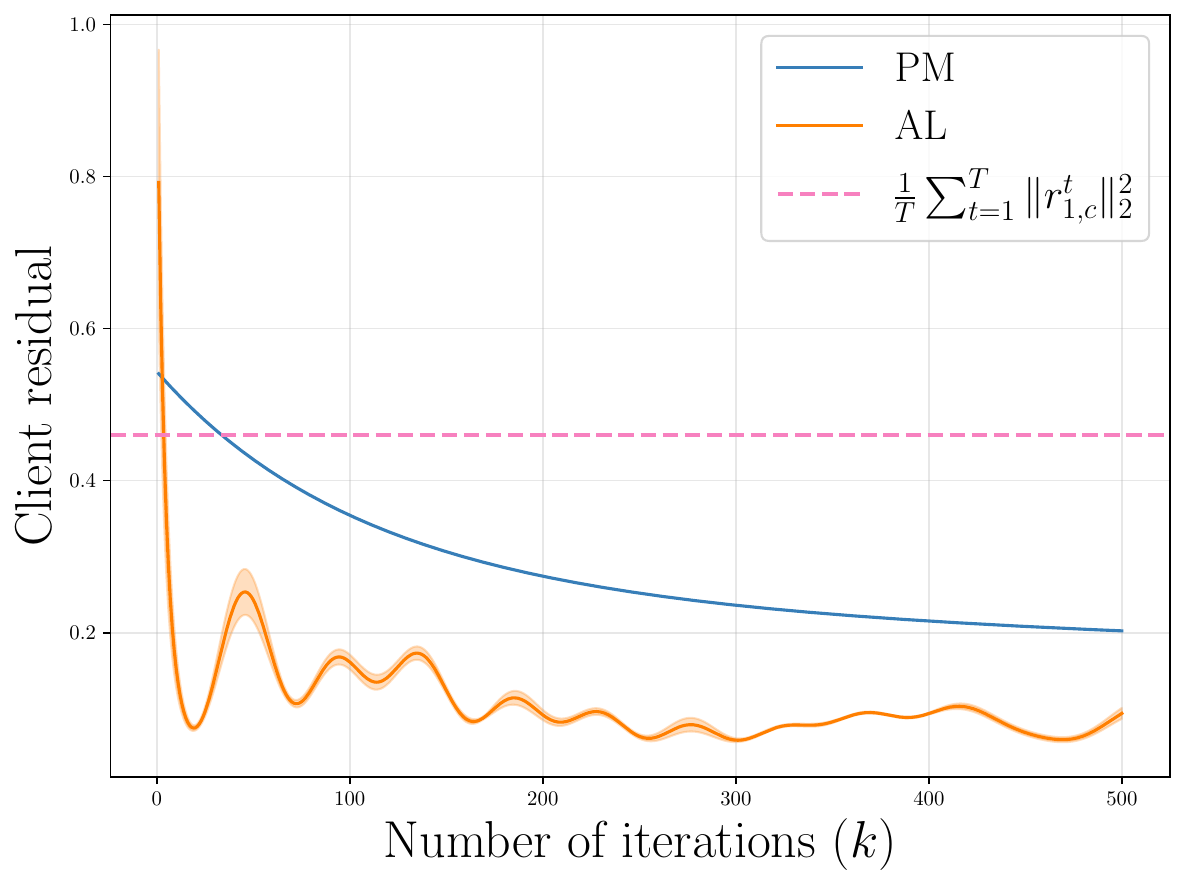}
  \caption{Client 1: $L_{1,a}^k$}\label{fig:comp_local_loss1}
\end{subfigure}

\vspace{0.6em}

\begin{subfigure}{0.32\textwidth}
  \includegraphics[width=\linewidth]{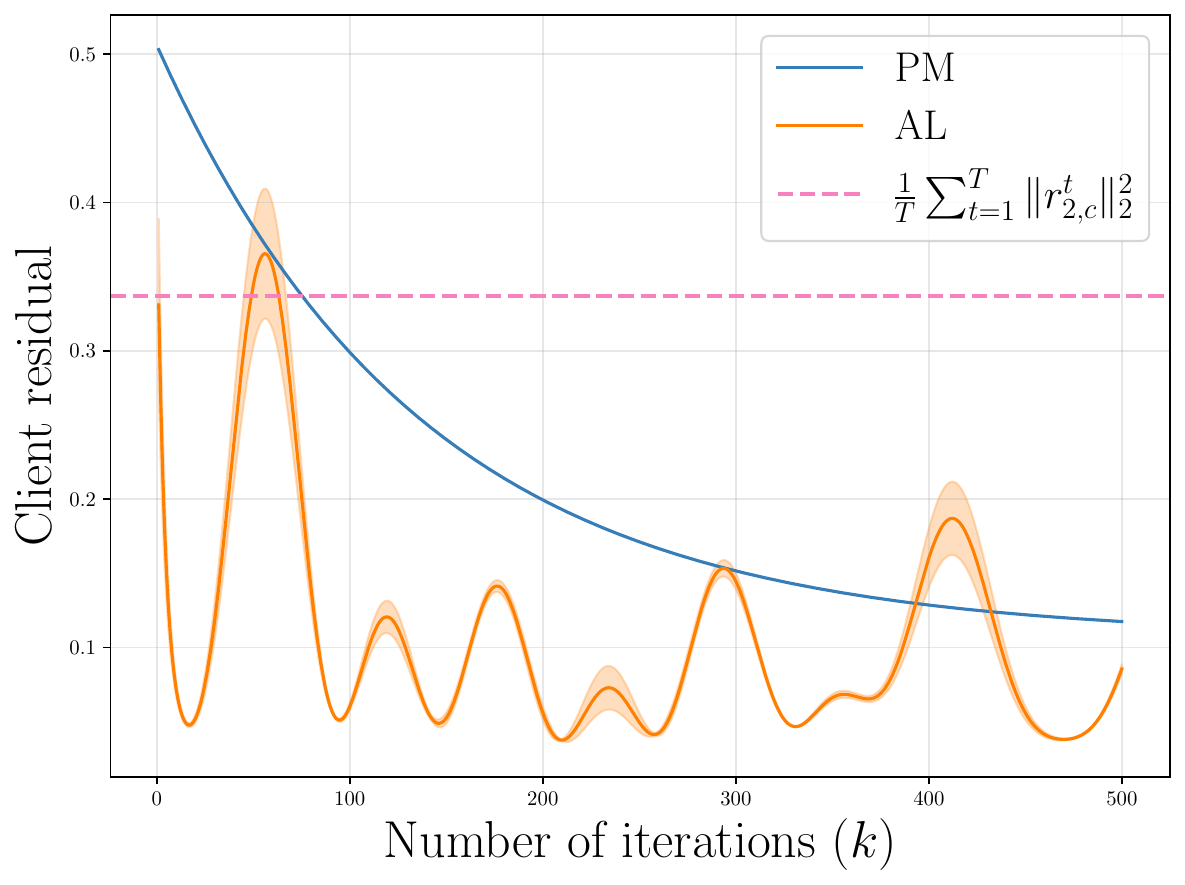}
  \caption{Client 2: $L_{2,a}^k$}\label{fig:comp_local_loss2}
\end{subfigure}\hfill
\begin{subfigure}{0.32\textwidth}
  \includegraphics[width=\linewidth]{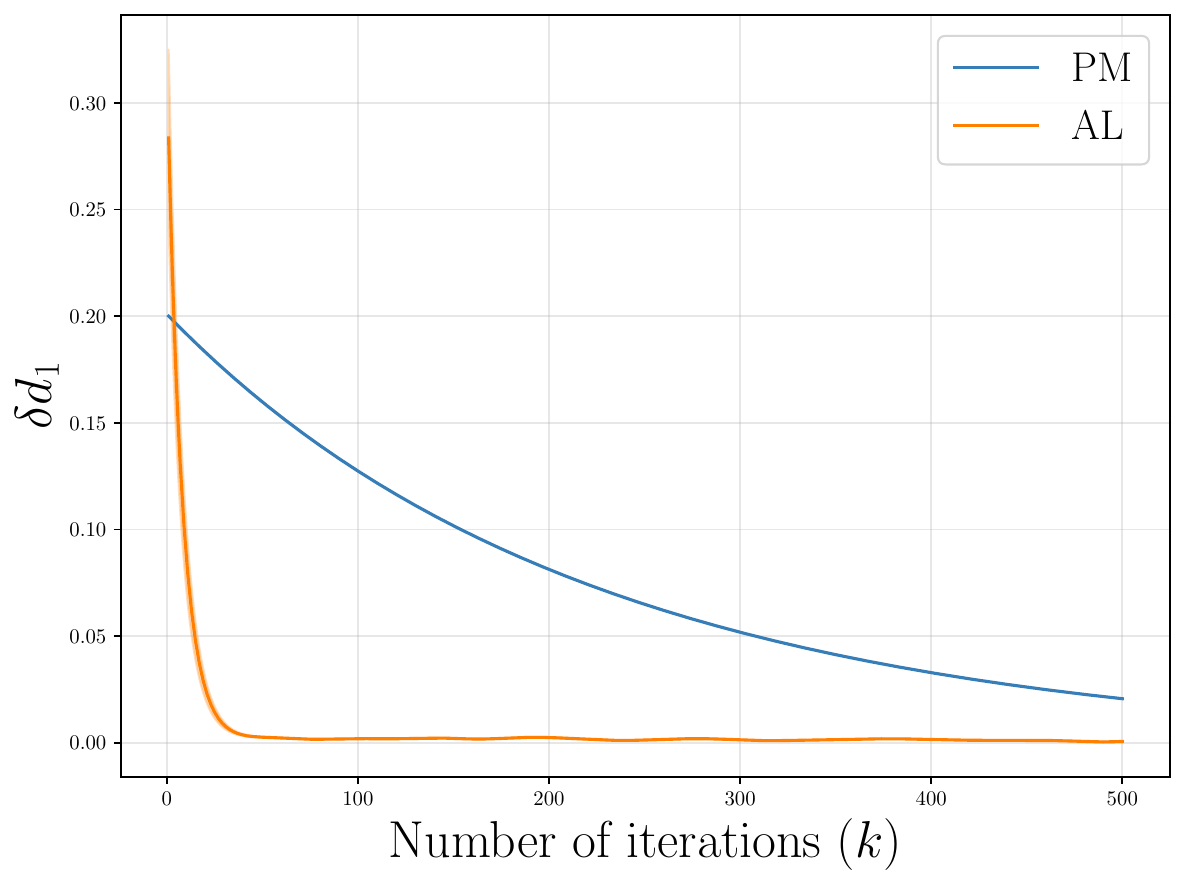}
  \caption{Client 1: $\delta d_1$}\label{fig:comp_delta_d_1}
\end{subfigure}\hfill
\begin{subfigure}{0.32\textwidth}
  \includegraphics[width=\linewidth]{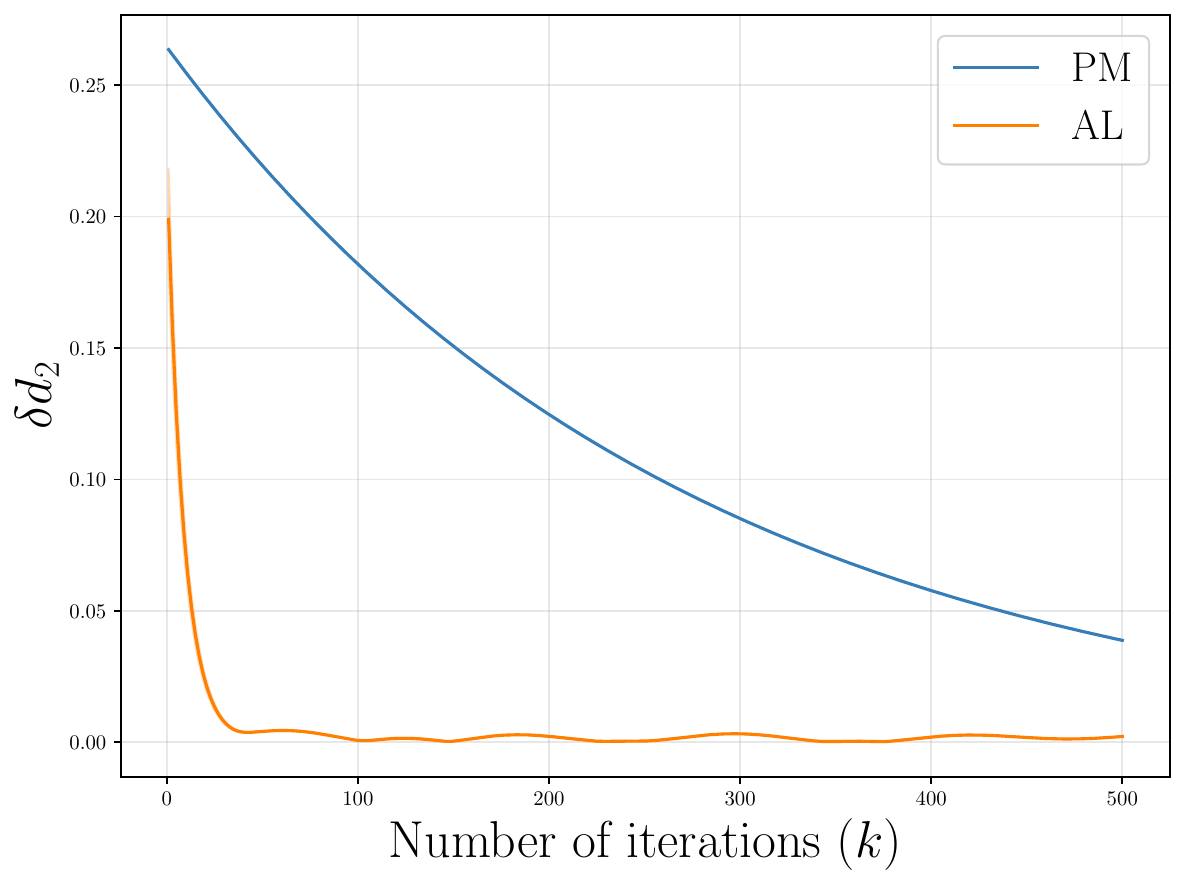}
  \caption{Client 1: $\delta d_1$}\label{fig:comp_delta_d_2}
\end{subfigure}

\caption{Comparison between Penalty Method (PM) and Augmented Lagrangian (AL) for the two client system. $L_s$, $\mathcal{D}$, Local loss (client residual) $L_{m,a}$ and $\delta d_m$ vs Number of iterations ($k$) are plotted. Note that we defined $\delta d_m := \norm{\phi_m - \frac{1}{T}\sum_{t = 1}^T \left(\sum_{n \neq m}\hat{B}_{mn} u_n^{t-1} \right)}_2$}
\label{fig:PM_vs_AL}
\end{figure*}

\paragraph{Scalability:}
We study the scalability of our framework in two ways:
\textbf{(i)} For a two component system, we fix the state dimension $P_m = 2$, input dimension $U_m = 2$ and generate stable LTI systems with measurement dimensions $D_m = \{16, 32, 64, 128\}$. We report the final server loss function $L_s$ and the disentanglement norm $\mathcal{D}$ in Table \ref{tab:scalability_D_dim}. \textbf{(ii)} The state, input and output dimensions are fixed at $P_m = 2$, $U_m = 2$ and $D_m = 8$ respectively. The number of clients are varied as $M = \{2,4,8,16\}$. As before, the final server loss $L_s$ and $\mathcal{D}$ are reported in Table \ref{tab:scalability_D_clients}

\begin{table}[h]
\centering
\caption{$L_s$ and $\mathcal{D}$ by scaling measurement dim. $D_m$}
\begin{tabular}{|cc|cc|cc|cc|}
\hline
\multicolumn{2}{|c|}{$D_m=16$} &
\multicolumn{2}{c|}{$D_m=32$} &
\multicolumn{2}{c|}{$D_m=64$} &
\multicolumn{2}{c|}{$D_m=128$} \\\hline
$L_s$ & $\mathcal{D}$ & $L_s$ & $\mathcal{D}$ & $L_s$ & $\mathcal{D}$ & $L_s$ & $\mathcal{D}$ \\\hline
0.7649 & 0.0034 & 1.0987 & 0.0041 & 1.4243 & 0.0046 & 1.1805 & 0.0047 \\\hline
\end{tabular}
\label{tab:scalability_D_dim}
\end{table}

\begin{table}[h]
\centering
\caption{Server loss $L_s$ by scaling number of clients $M$}
\begin{tabular}{|cc|cc|cc|cc|}
\hline
\multicolumn{2}{|c|}{$M=2$} &
\multicolumn{2}{c|}{$M=4$} &
\multicolumn{2}{c|}{$M=8$} &
\multicolumn{2}{c|}{$M=16$} \\\hline
$L_s$ & $\mathcal{D}$ & $L_s$ & $\mathcal{D}$ & $L_s$ & $\mathcal{D}$ & $L_s$ & $\mathcal{D}$ \\\hline
0.0744 & 0.0013 & 0.0412 & 0.0010 & 0.1714 & 0.0036 & 0.3825 & 0.0069 \\\hline
\end{tabular}
\label{tab:scalability_D_clients}
\end{table}

\noindent
From Tables~\ref{tab:scalability_D_dim} and~\ref{tab:scalability_D_clients}, it is clear that as the observation dimension $D_m$ and the number of clients $M$ increase, the performance degrades. Scaling $D_m$ impacts the objective more: with $P_m$ fixed, the model must explain an ever larger output space with a low-dimensional latent state. This makes the observation map $C$ increasingly tall, amplifies estimation variance and sensor noise aggregation, and exposes unmodeled sensor-specific dynamics; consequently $L_s$ rises even though $\mathcal{D}$ remains small. Increasing $M$ introduces more inter-client couplings (off-diagonal $A_{mn},B_{mn}$ blocks) and consensus constraints to disentangle; the larger, more heterogeneous optimization also raises gradient variance and complicates stepsize selection, leading to a milder but noticeable growth in $L_s$ and $\mathcal{D}$. Overall, the disentanglement metric remains low across all settings ($\mathcal{D}\!\le\!7\times10^{-3}$), indicating that the method preserves separation even as scale grows. 

\paragraph{Transfer Function:}
Poles, zeros, and the spectrum of singular values for the transfer function of estimated and true system are given in Figure~\ref{fig:tf_plots}

\begin{figure}[t]
  \centering

  \begin{subfigure}[t]{0.32\textwidth}
    \centering
    \includegraphics[width=\linewidth]{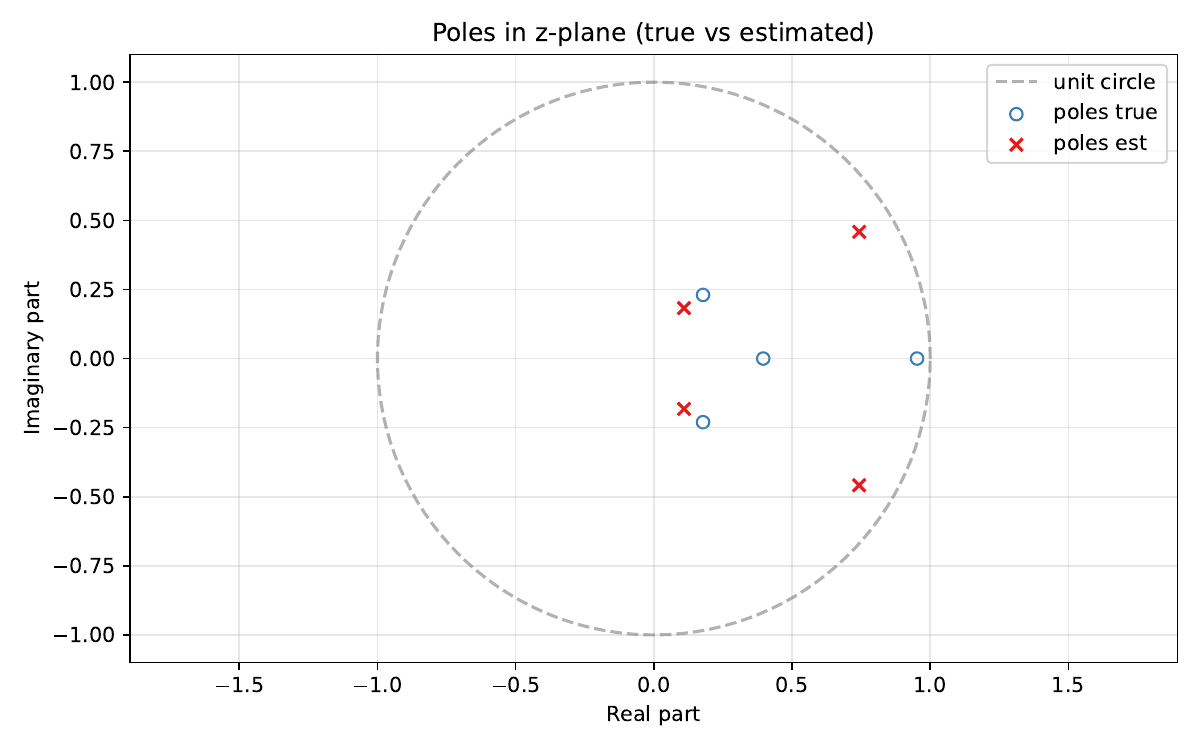}
    \caption{Tranfer Function - Poles}
    \label{fig:tf_poles}
  \end{subfigure}
  \hfill
  \begin{subfigure}[t]{0.32\textwidth}
    \centering
    \includegraphics[width=\linewidth]{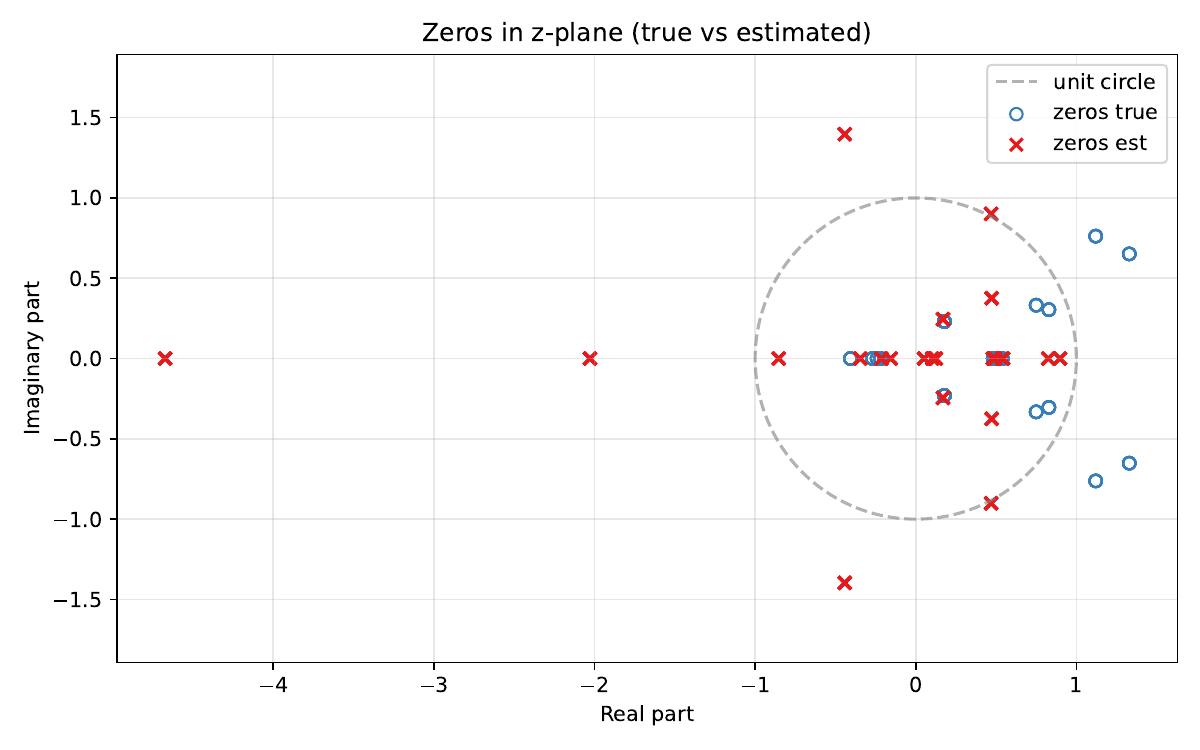}
    \caption{Transfer Function - Zeros}
    \label{fig:ft_zeros}
  \end{subfigure}
  \hfill
  \begin{subfigure}[t]{0.32\textwidth}
    \centering
    \includegraphics[width=\linewidth]{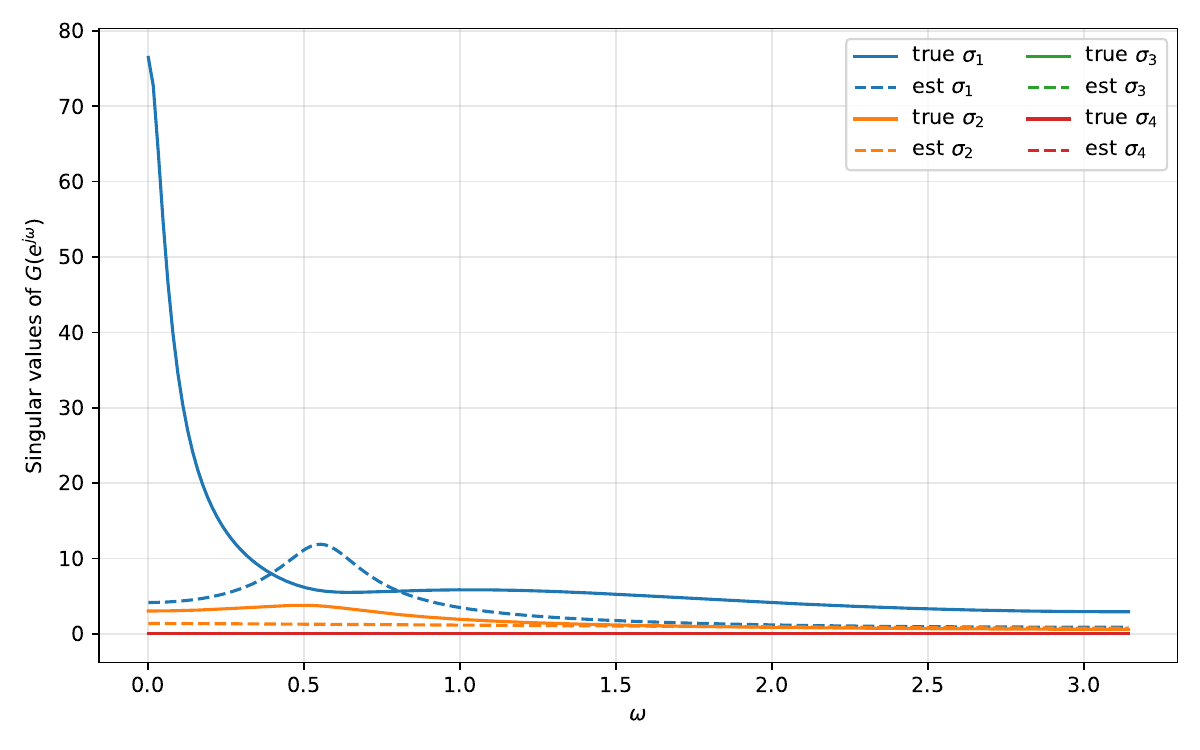}
    \caption{Tranfer Function - Singular Value ($\sigma$) spectrum}
    \label{fig:tf_sigma}
  \end{subfigure}
  \caption{Tranfer Function characteristics}
\label{fig:tf_plots}
\end{figure}

\subsection{Real-world Datasets}\label{appendix:realWorld}
\paragraph{HAI dataset (v21.03) \cite{hai_2103}:}
The HAI industrial control dataset logs a multi–unit process with four coupled subsystems (clients), which we refer to as \textbf{P1–P4}. Tags are prefixed by the process identifier (e.g., \texttt{P1\_}, \texttt{P2\_}, \texttt{P3\_}, \texttt{P4\_}). Within each process, \emph{sensor} channels (e.g., level/flow/pressure/temperature transmitters) are continuous, while \emph{actuator/command} channels (e.g., valve commands and pump toggles) are predominantly discrete. In our federated setting we map \emph{client} $k$ to \emph{process} $P_k$ simply by this prefix rule. Inputs $U$ are taken from actuator/command tags for that process (e.g., FCV/LCV/PCV/PP/\texttt{OnOff}/\texttt{AutoGO}/\texttt{RTR} families), and outputs $Y$ are the continuous transmitter tags (e.g., LIT/FIT/PIT/TIT/SIT/V*T* families), as specified in the HAI tag manual.

\paragraph{Preprocessing}
Starting from the vendor CSV logs, we (i) time-align and parse headers; (ii) remove \emph{constant} channels over the full run and re-check after truncation; (iii) prune discrete-like signals before modeling (drop boolean-like and quasi-discrete channels with low transition counts or extreme dominance); (iv) split variables into $U$ and $Y$ by tag family and process prefix; (v) standardize each column by $z$-scoring and save the means/stds; and (vi) fit a centralized LTI model with unconstrained $A,B$ while enforcing block-diagonal $C,Q,R$ by process. The identified states are then partitioned to processes by maximizing energy of the corresponding $C$-columns on each process’ outputs, with a fix-up step to guarantee at least one state per active process.

Figure~\ref{fig:hai_dataset} shows a consistent convergence story across all panels. The server loss $L_s^k$ (a) decays rapidly in the first few dozen iterations and then tapers off smoothly, indicating a stable approach to a low-loss regime without oscillations. The disentanglement constraint $\mathcal D$ (b) follows the same pattern, decreasing monotonically and confirming that the consensus/coupling conditions are being enforced as optimization proceeds. Overall, the results show fast, stable convergence on HAI with uniformly small constraints and residuals at the end, despite mild per-client differences in the transient


\begin{figure*}[hbt!]
\centering

\begin{subfigure}{0.32\textwidth}
  \includegraphics[width=\linewidth]{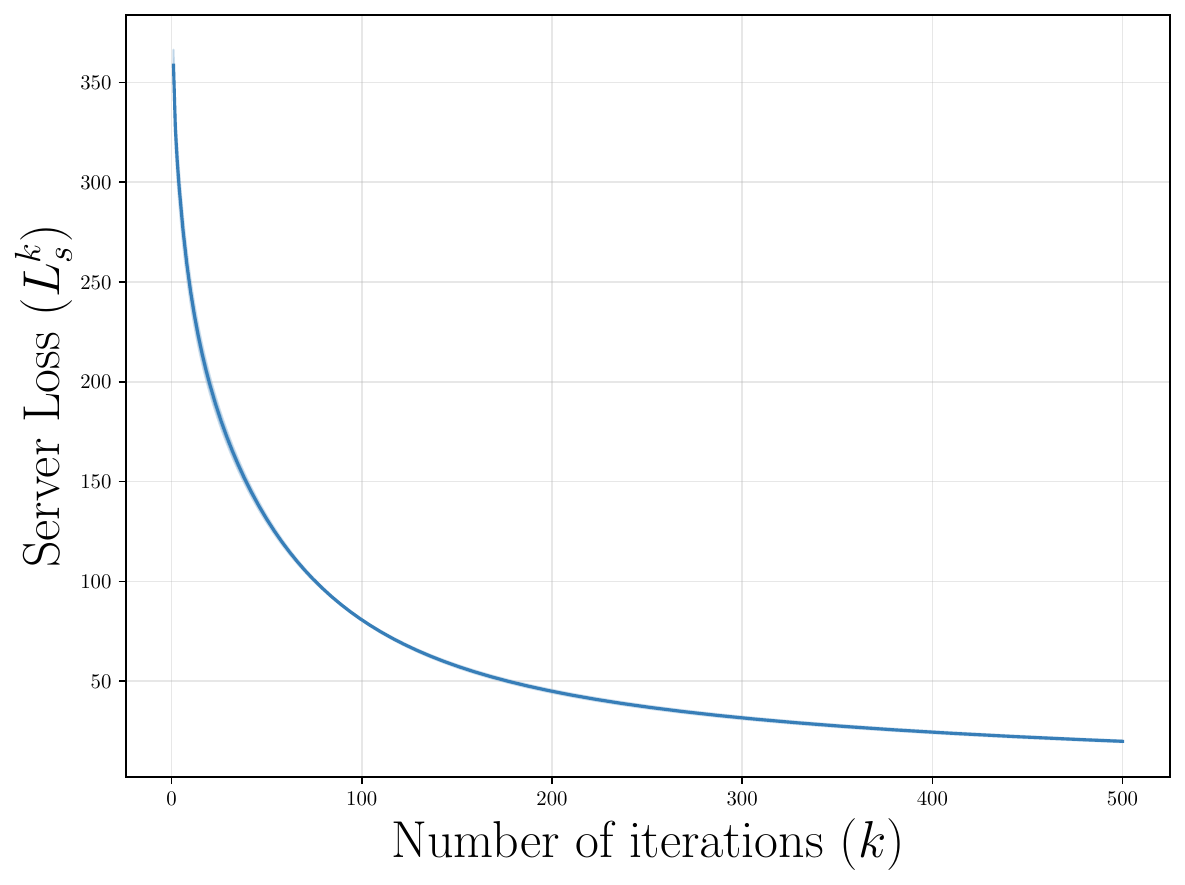}
  \caption{Server Loss}\label{fig:a}
\end{subfigure}\hfill
\begin{subfigure}{0.32\textwidth}
  \includegraphics[width=\linewidth]{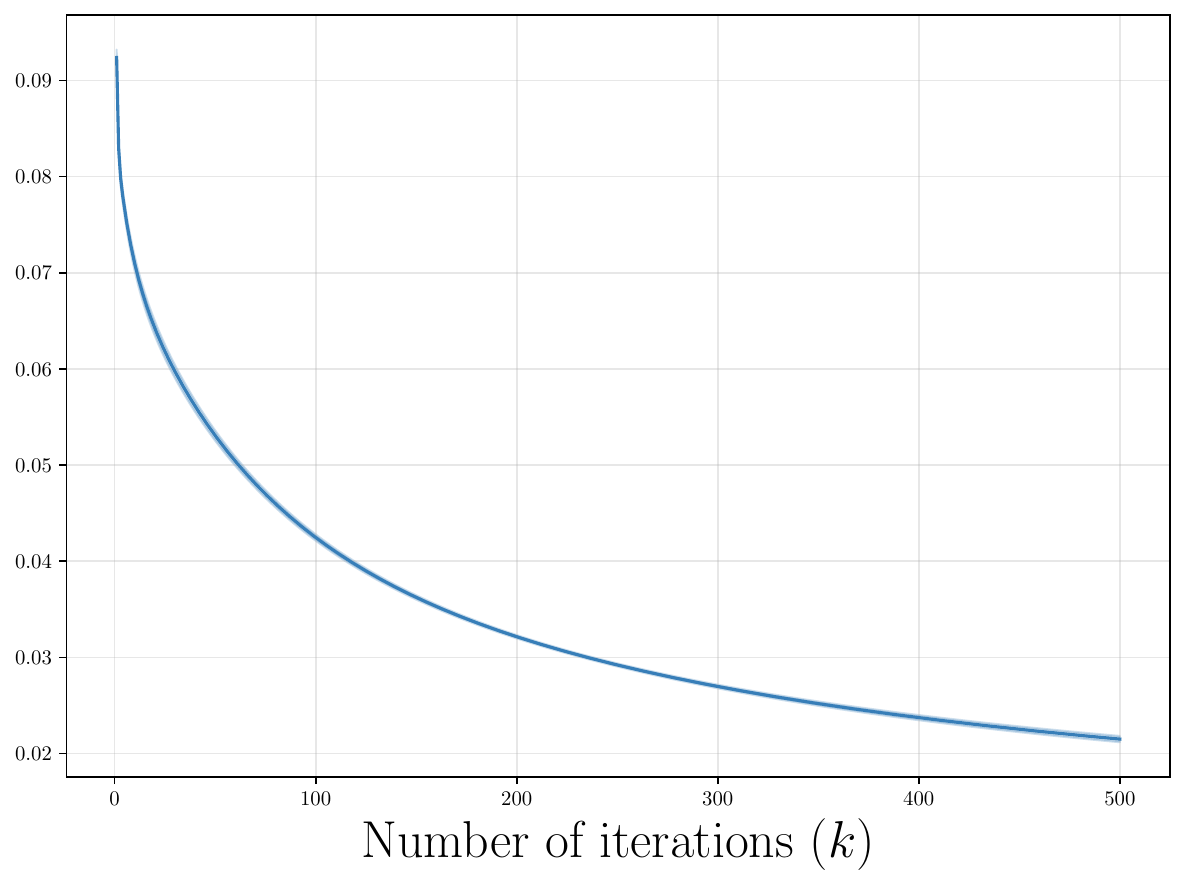}
  \caption{$\mathcal{D}$}\label{fig:b}
\end{subfigure}\hfill
\begin{subfigure}{0.32\textwidth}
  \includegraphics[width=\linewidth]{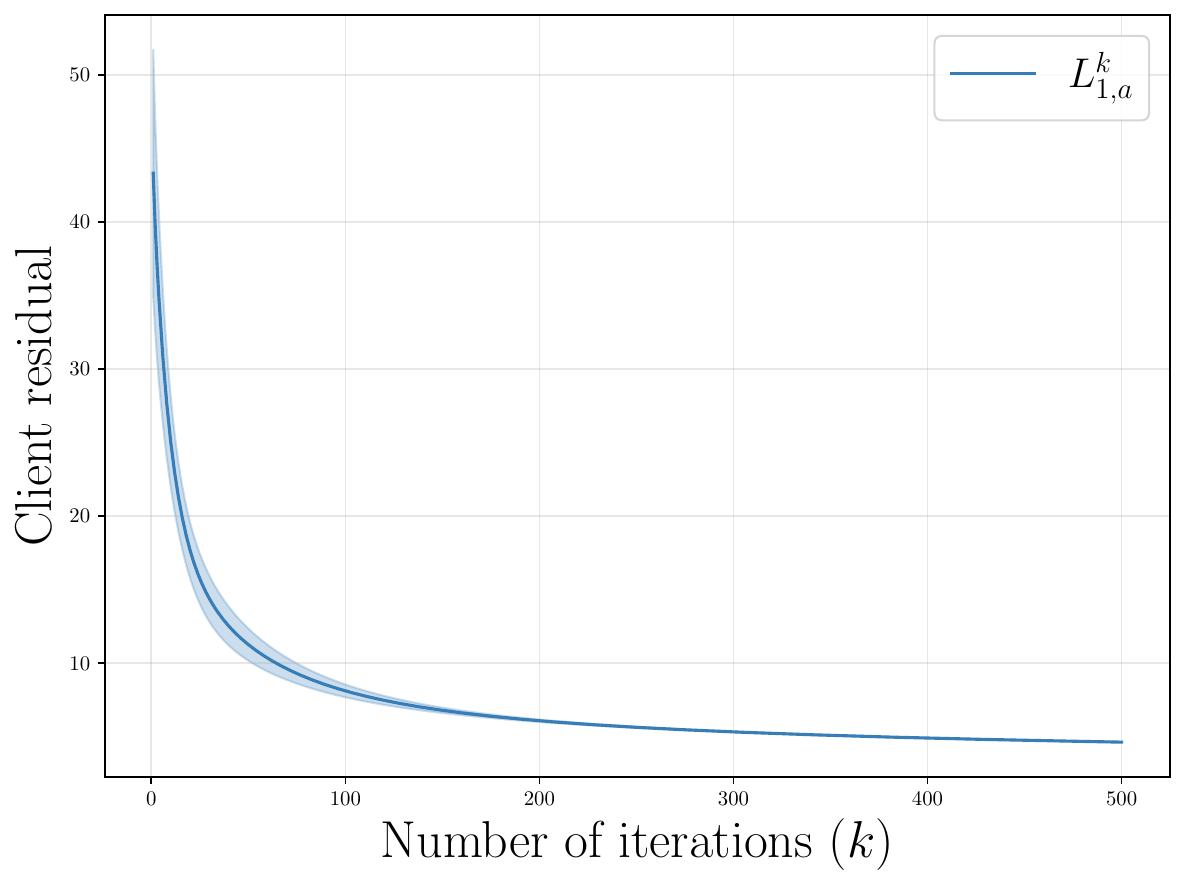}
  \caption{Client 1: $L_{1,a}^k$}\label{fig:c}
\end{subfigure}

\vspace{0.6em}

\begin{subfigure}{0.32\textwidth}
  \includegraphics[width=\linewidth]{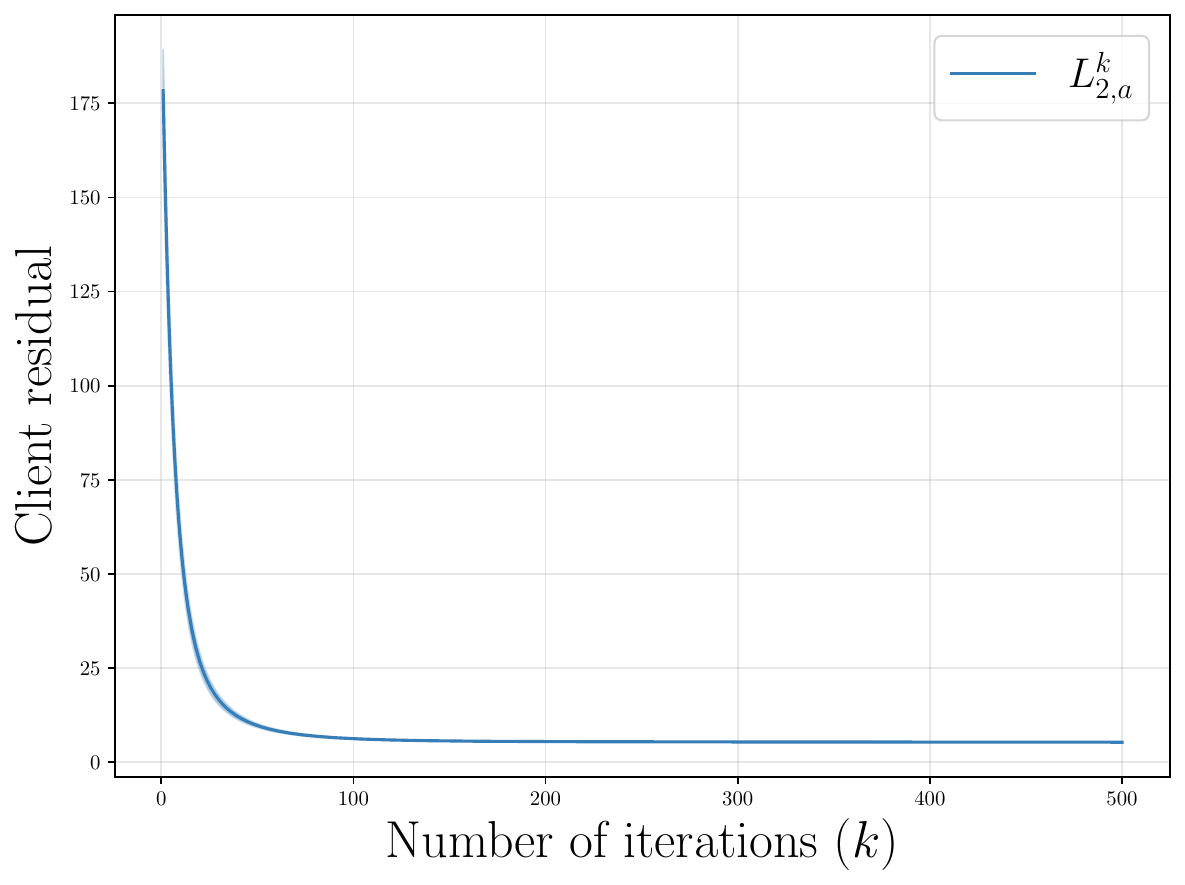}
  \caption{Client 2: $L_{2,a}^k$}\label{fig:d}
\end{subfigure}\hfill
\begin{subfigure}{0.32\textwidth}
  \includegraphics[width=\linewidth]{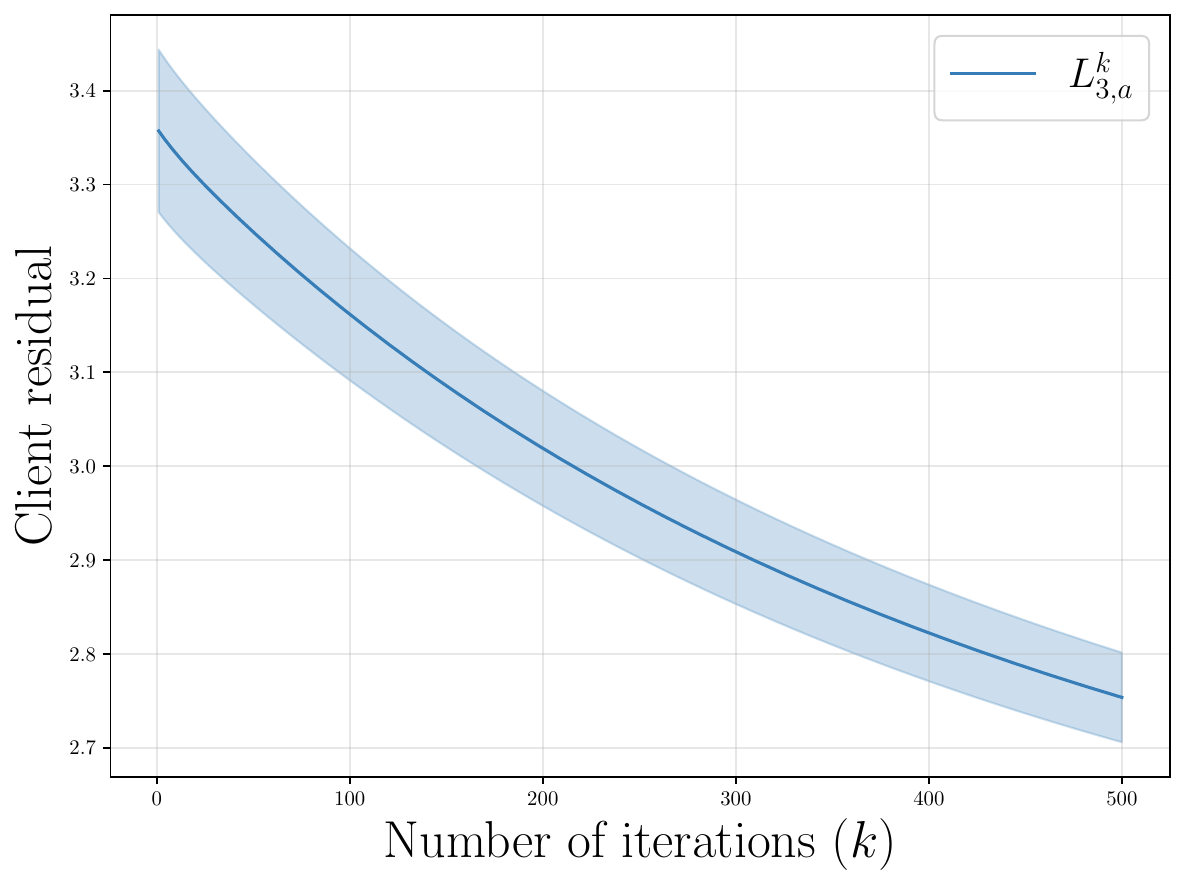}
  \caption{Client 3: $L_{3,a}^k$}\label{fig:e}
\end{subfigure}\hfill
\begin{subfigure}{0.32\textwidth}
  \includegraphics[width=\linewidth]{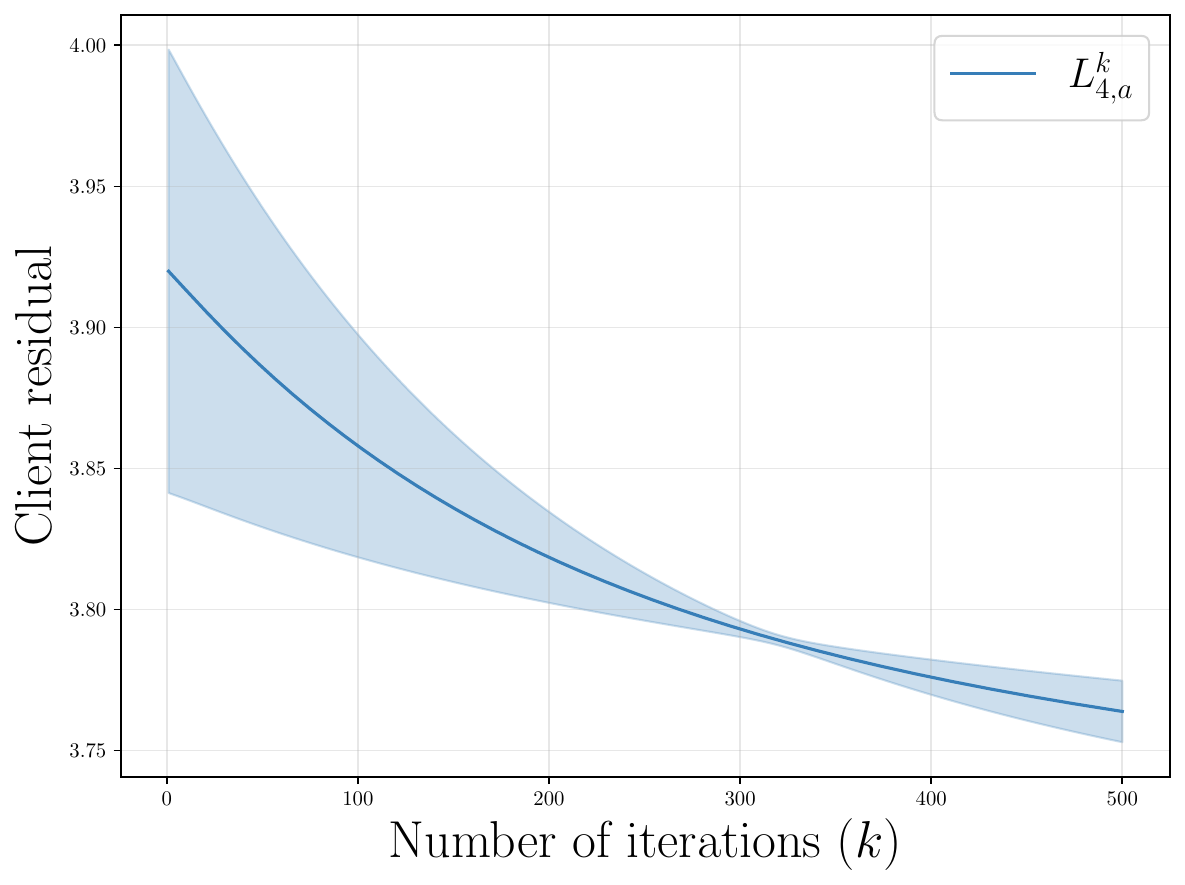}
  \caption{Client 4: $L_{4,a}^k$}\label{fig:f}
\end{subfigure}

\vspace{0.6em}

\begin{subfigure}{0.32\textwidth}
  \includegraphics[width=\linewidth]{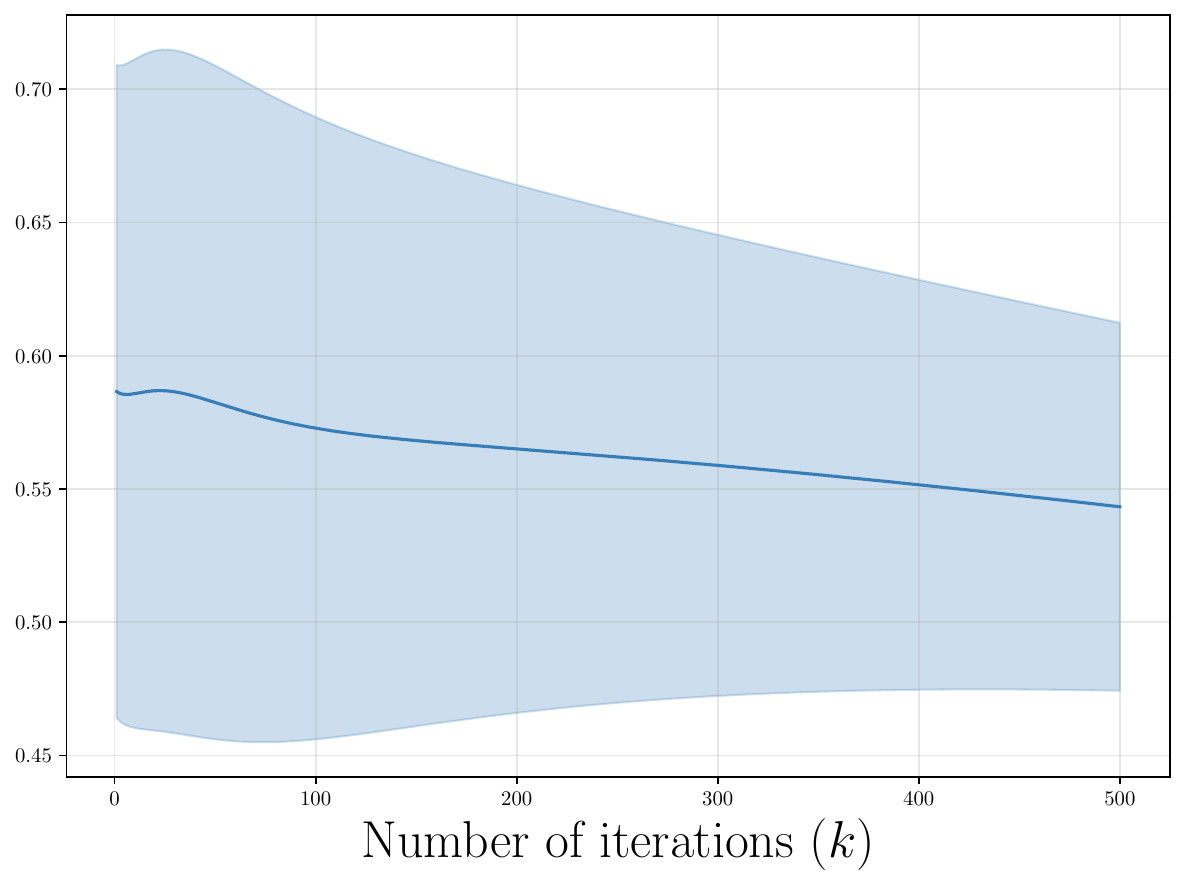}
  \caption{Client 1: $\delta d_1$}\label{fig:g}
\end{subfigure}\hfill
\begin{subfigure}{0.32\textwidth}
  \includegraphics[width=\linewidth]{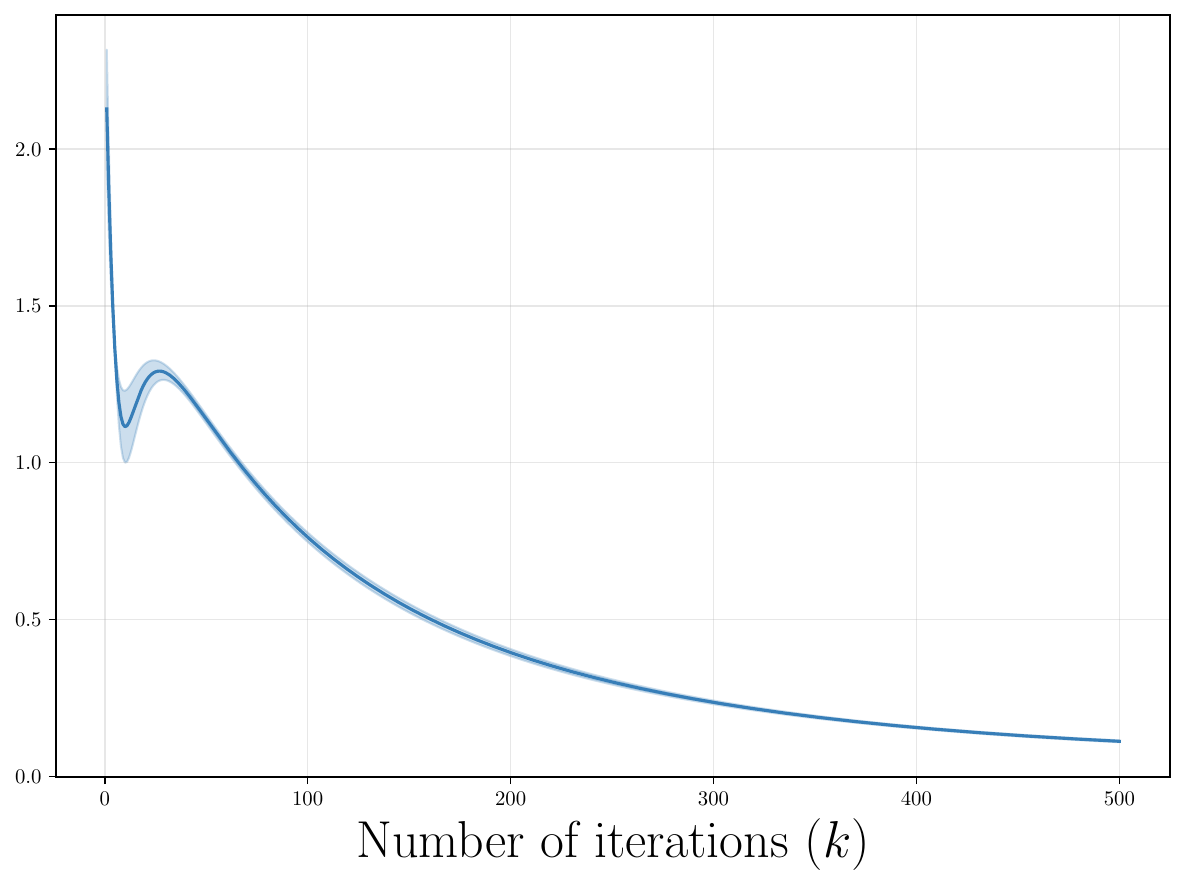}
  \caption{Client 2: $\delta d_2$}\label{fig:h}
\end{subfigure}\hfill
\begin{subfigure}{0.32\textwidth}
  \includegraphics[width=\linewidth]{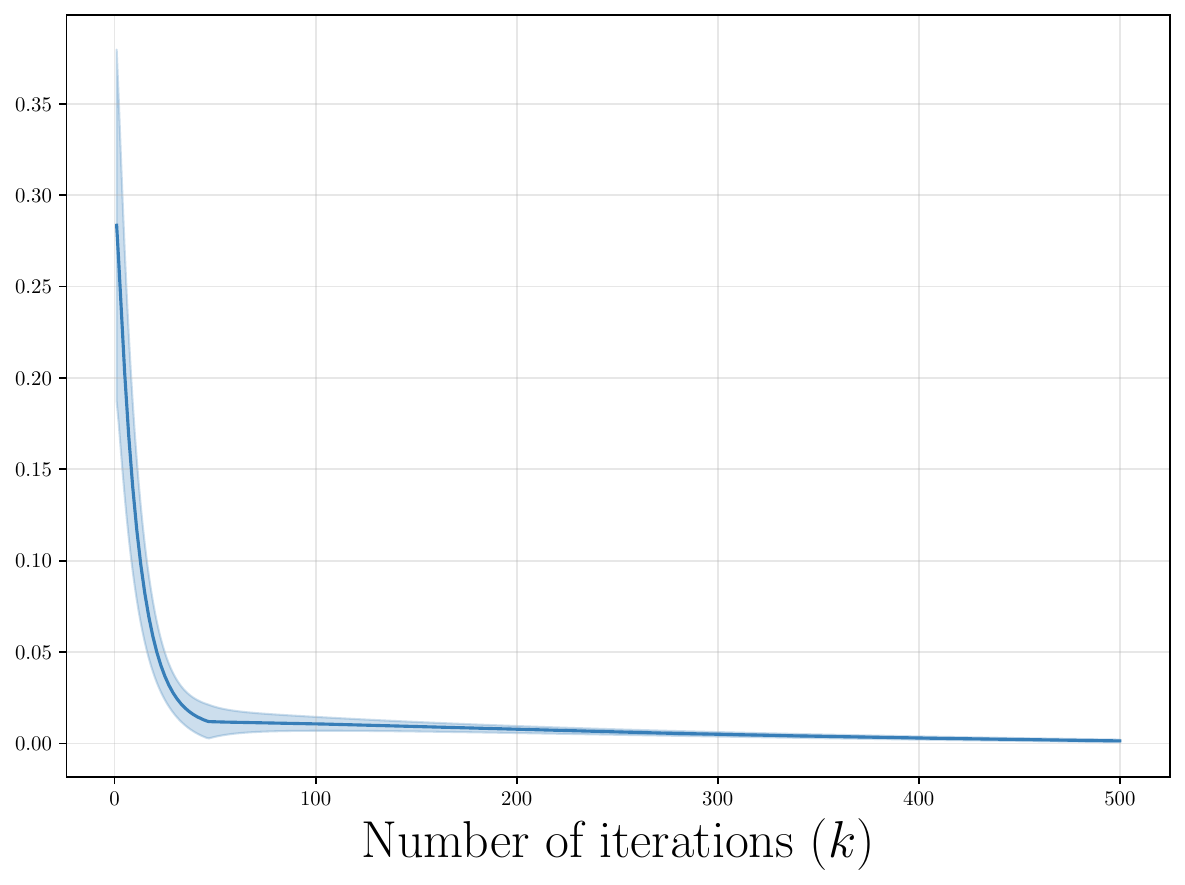}
  \caption{Client 3: $\delta d_3$}\label{fig:i}
\end{subfigure}

\vspace{0.6em}

\begin{subfigure}{0.32\textwidth}
  \includegraphics[width=\linewidth]{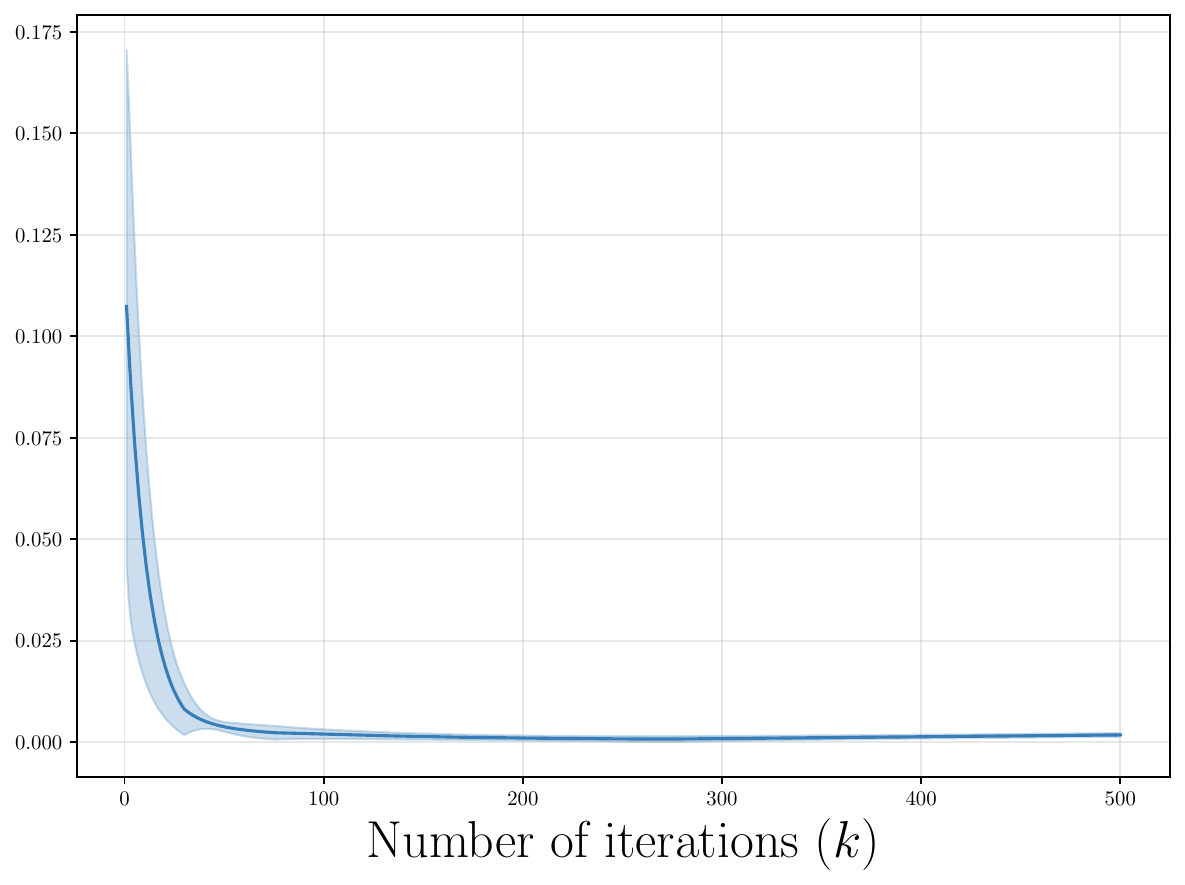}
  \caption{Client 4: $\delta d_4$}\label{fig:j}
\end{subfigure}
\hfill
\begin{minipage}{0.32\textwidth}\hspace{0pt}\end{minipage}
\hfill
\begin{minipage}{0.32\textwidth}\hspace{0pt}\end{minipage}

\caption{HAI dataset}
\label{fig:hai_dataset}
\end{figure*}

\paragraph{SWaT Dataset \cite{swat_dataset}}. The Secure Water Treatment (SWaT) plant is a six-stage potable water facility with tightly controlled pumping, filtration, and disinfection units. Tags follow the instrumentation manual convention: process equipment identifiers (e.g., MV101, P201, UV401) embed the stage in the leading digit. We map federated clients to stages S1–S6 using this digit. Manipulated variables (inputs) come from actuator and command tags—motorised valves (MV***), pumps (P***), UV toggles, and setpoint/command channels. Measurements (outputs) are the continuous transmitters: flow (FIT***), level (LIT***), differential pressure (DPIT***), pressure (PIT***), analogue analysers, etc. Unlike HAI, actuators are almost fully discrete, so we treat them as binary bits (after collapsing occasional tri-state encodings) and reserve all real-valued telemetry for outputs.

\paragraph{Preprocessing}: (i) promote the first row to the header, relabel columns, and align timestamps; (ii) drop constant channels globally and again after truncating to the first 30000 samples; (iii) detect and remove quasi-discrete or boolean-like signals on the measurement side; (iv) classify each tag into inputs or outputs using stage-aware regex heuristics and split per stage; (v) z-score every column, saving the means/standard deviations; and (vi) add small Gaussian noise to the (otherwise discrete) input channels so the centralized LTI identification has non-degenerate excitation. We then identify a centralized LTI model with unconstrained (A, B) but enforce block-diagonal (C, Q, R) by stage. States are assigned to  stages (clients) via the energy of their (C)-columns on each stage’s outputs, with a fix-up to ensure every active stage owns at least one state before exporting per-stage component folders.

The results from SWaT dataset are plotted in Figure \ref{fig:swat_dataset}. Similar to the results from the HAI dataset, Figure~\ref{fig:swat_dataset} illustrates the convergence behavior observed on the SWaT dataset. The server loss $L_s^k$ (a) exhibits a sharp initial decline followed by a smooth saturation phase, indicating a stable approach toward a low-loss regime with no noticeable oscillations. The disentanglement constraint $\mathcal D$ (b) follows a similar monotonic decreasing trend, demonstrating that the consensus and coupling constraints are effectively enforced throughout optimization. Due to the large scale of the SWaT dataset and the constraints of available computational resources, Monte Carlo simulations were not performed; consequently, the reported results correspond to a single representative sample path.

\begin{figure*}[t]
\centering
\captionsetup[subfigure]{font=footnotesize,labelfont=bf,justification=centering}

\begin{adjustbox}{}
\begin{minipage}{\textwidth}

\begin{subfigure}{0.32\textwidth}
  \includegraphics[width=\linewidth]{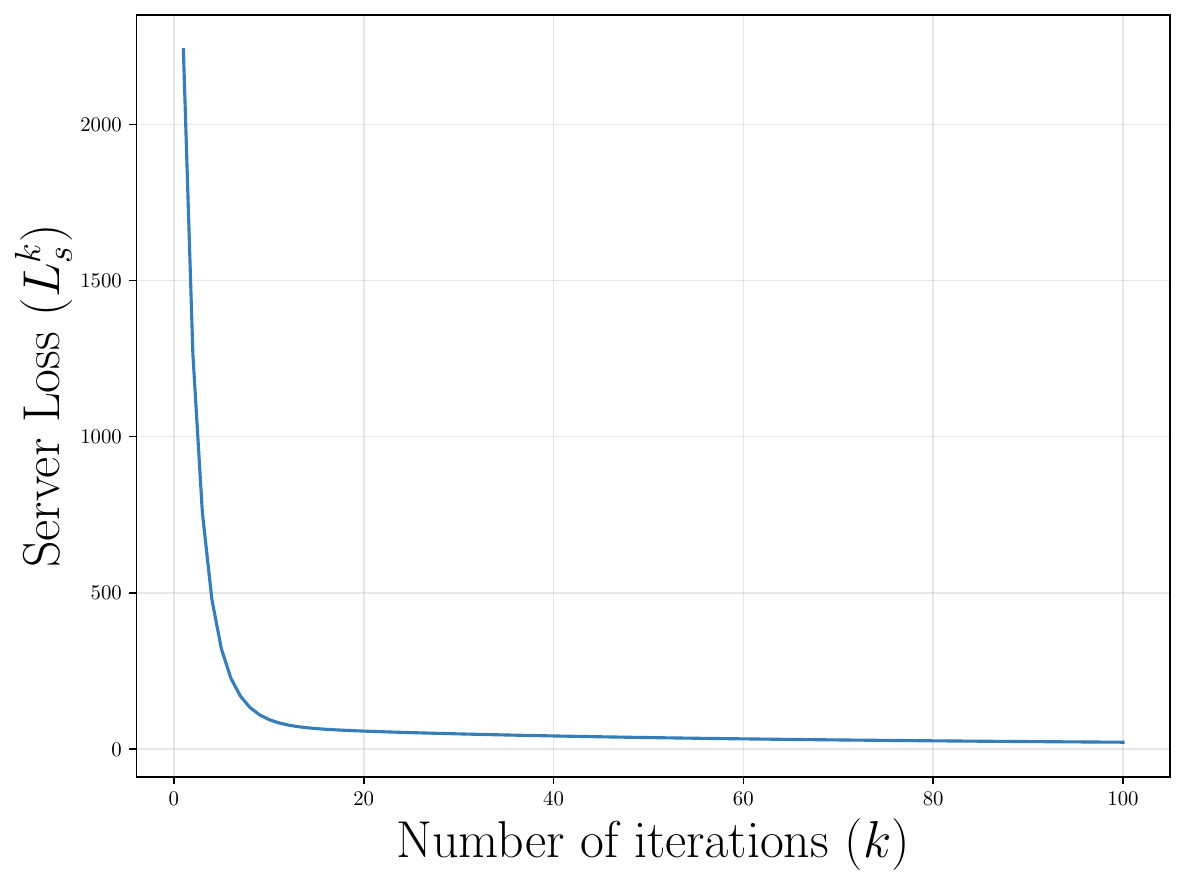}
  \caption{Server Loss}\label{fig:a}
\end{subfigure}\hfill
\begin{subfigure}{0.32\textwidth}
  \includegraphics[width=\linewidth]{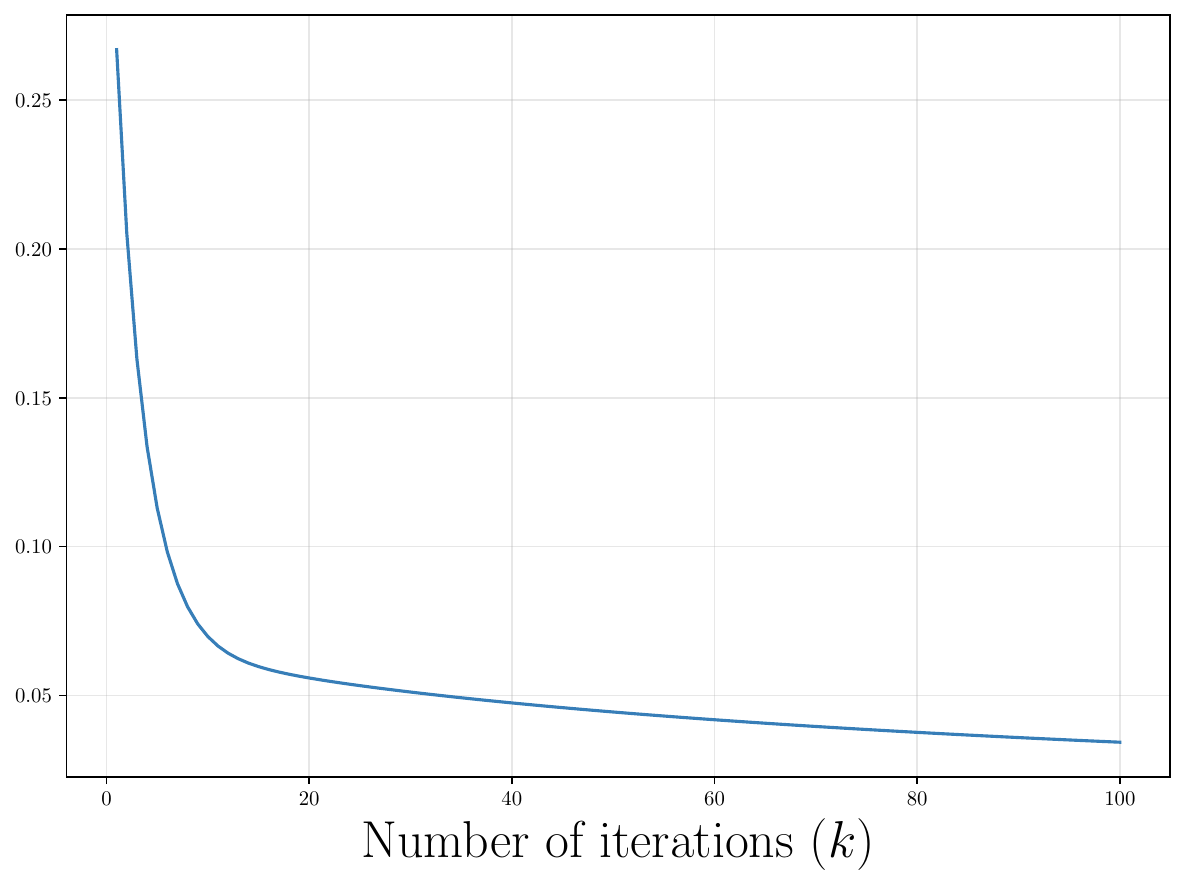}
  \caption{$\mathcal{D}$}\label{fig:b}
\end{subfigure}\hfill
\begin{subfigure}{0.32\textwidth}
  \includegraphics[width=\linewidth]{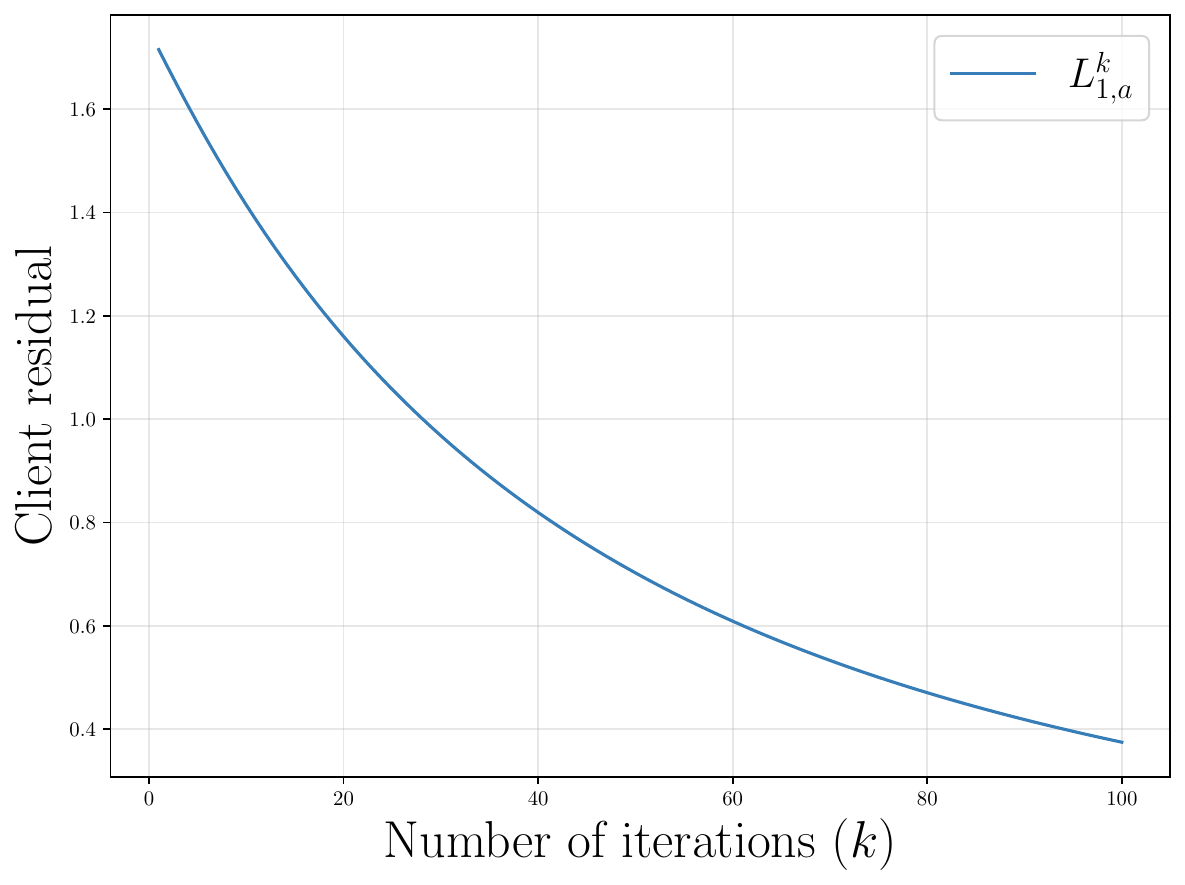}
  \caption{Client 1: $L_{1,a}^k$}\label{fig:c}
\end{subfigure}

\begin{subfigure}{0.32\textwidth}
  \includegraphics[width=\linewidth]{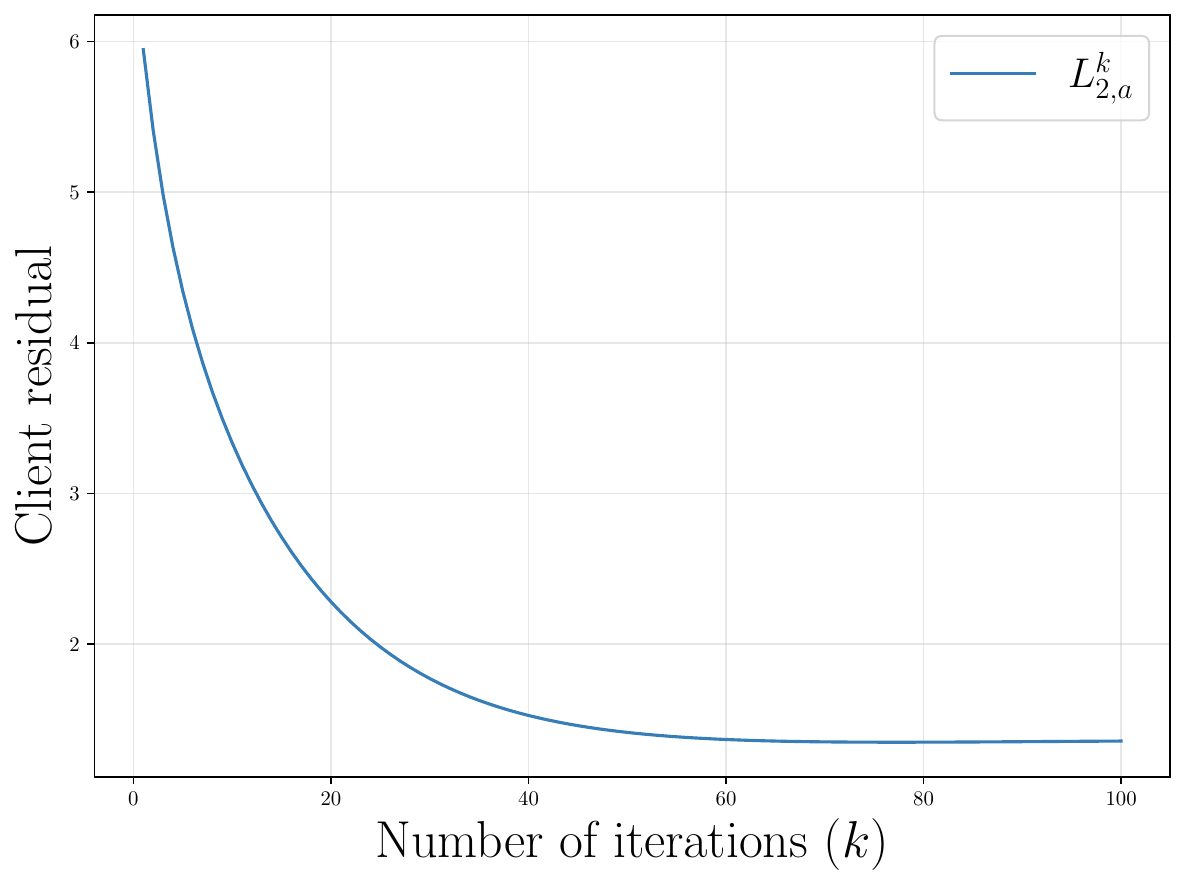}
  \caption{Client 2: $L_{2,a}^k$}\label{fig:d}
\end{subfigure}\hfill
\begin{subfigure}{0.32\textwidth}
  \includegraphics[width=\linewidth]{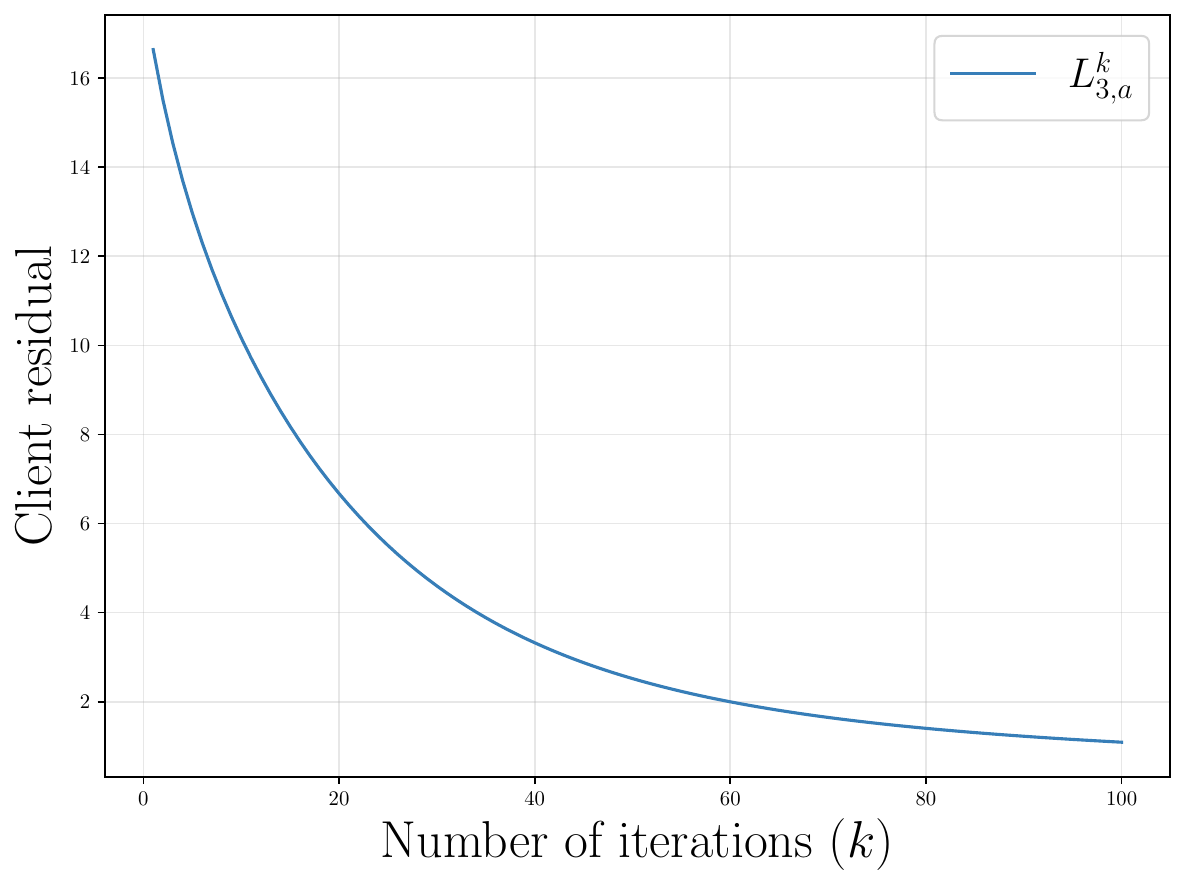}
  \caption{Client 3: $L_{3,a}^k$}\label{fig:e}
\end{subfigure}\hfill
\begin{subfigure}{0.32\textwidth}
  \includegraphics[width=\linewidth]{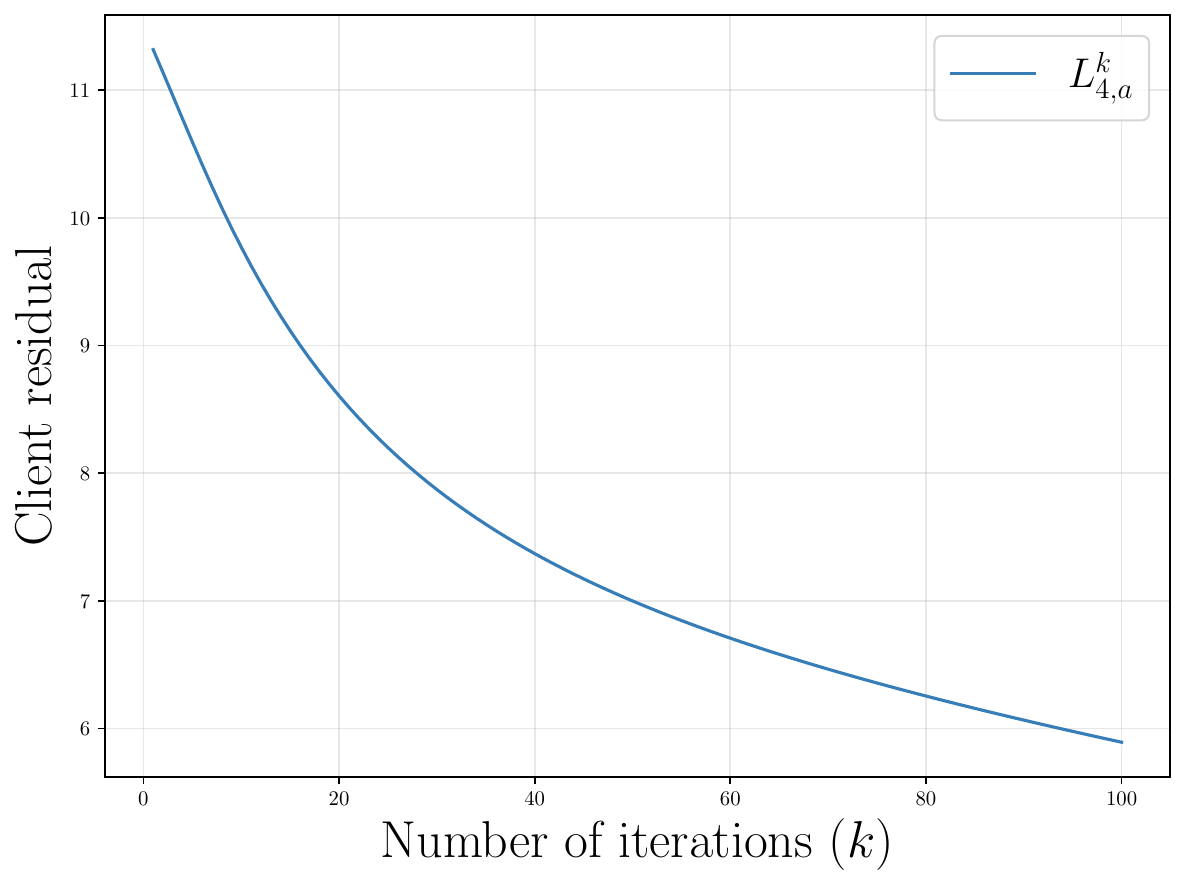}
  \caption{Client 4: $L_{4,a}^k$}\label{fig:f}
\end{subfigure}

\begin{subfigure}{0.32\textwidth}
  \includegraphics[width=\linewidth]{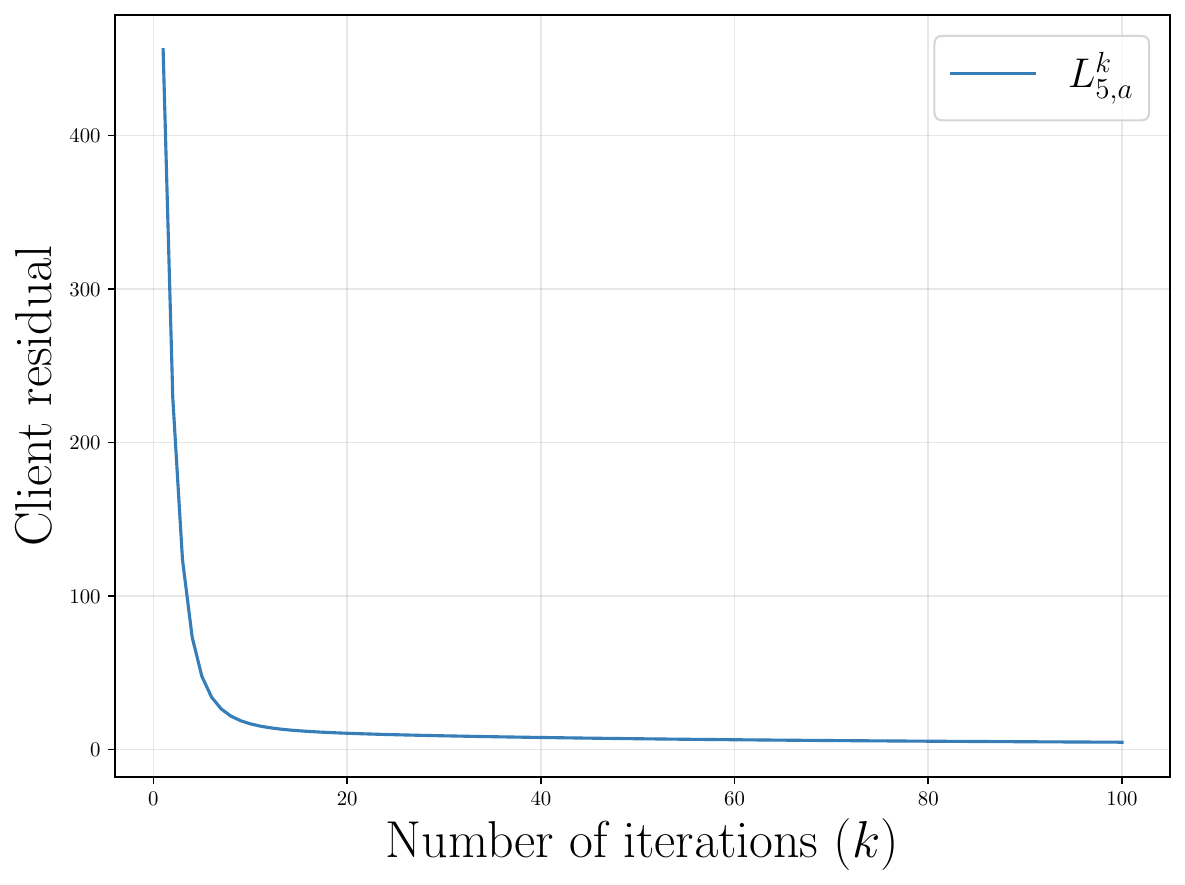}
  \caption{Client 5: $L_{5,a}^k$}\label{fig:k}
\end{subfigure}\hfill
\begin{subfigure}{0.32\textwidth}
  \includegraphics[width=\linewidth]{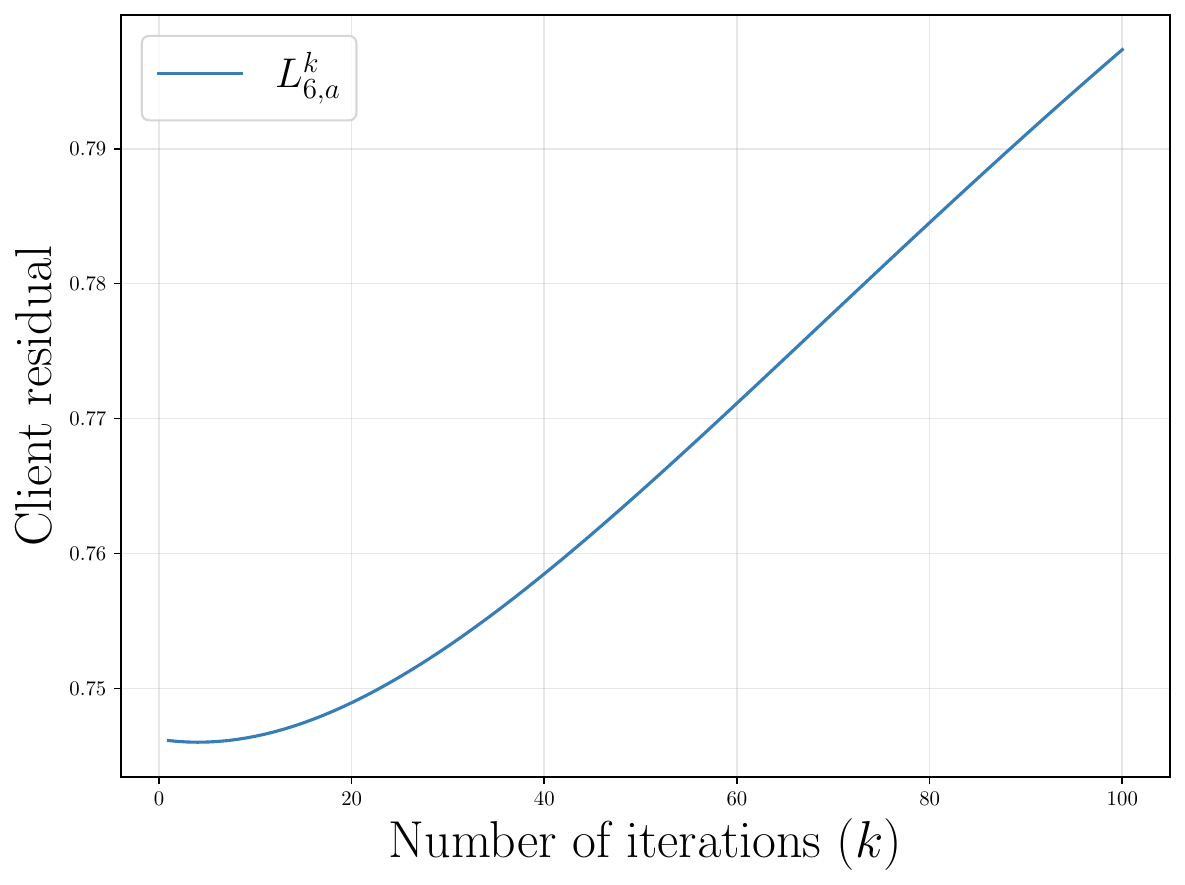}
  \caption{Client 6: $L_{6,a}^k$}\label{fig:l}
\end{subfigure}\hfill
\begin{minipage}{0.32\textwidth}\hspace{0pt}\end{minipage}

\begin{subfigure}{0.32\textwidth}
  \includegraphics[width=\linewidth]{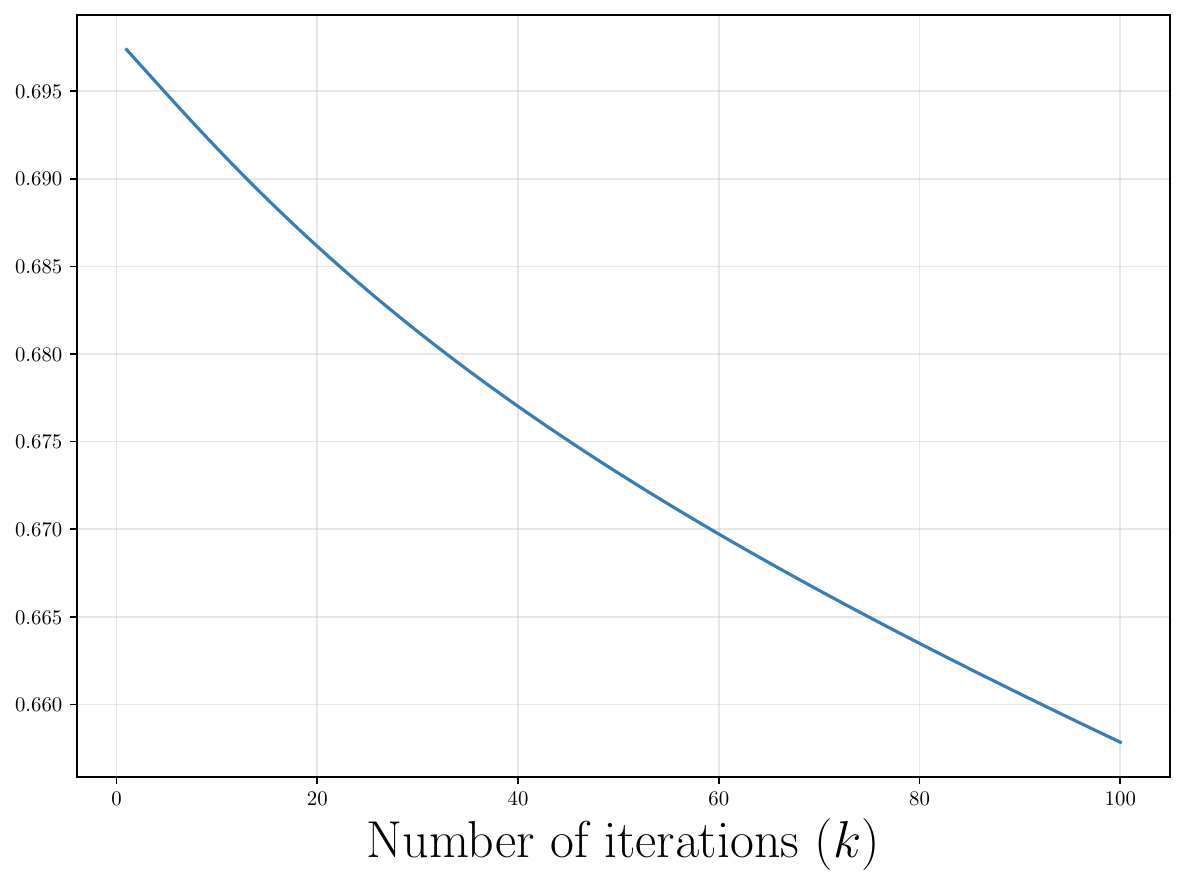}
  \caption{Client 1: $\delta d_1$}\label{fig:g}
\end{subfigure}\hfill
\begin{subfigure}{0.32\textwidth}
  \includegraphics[width=\linewidth]{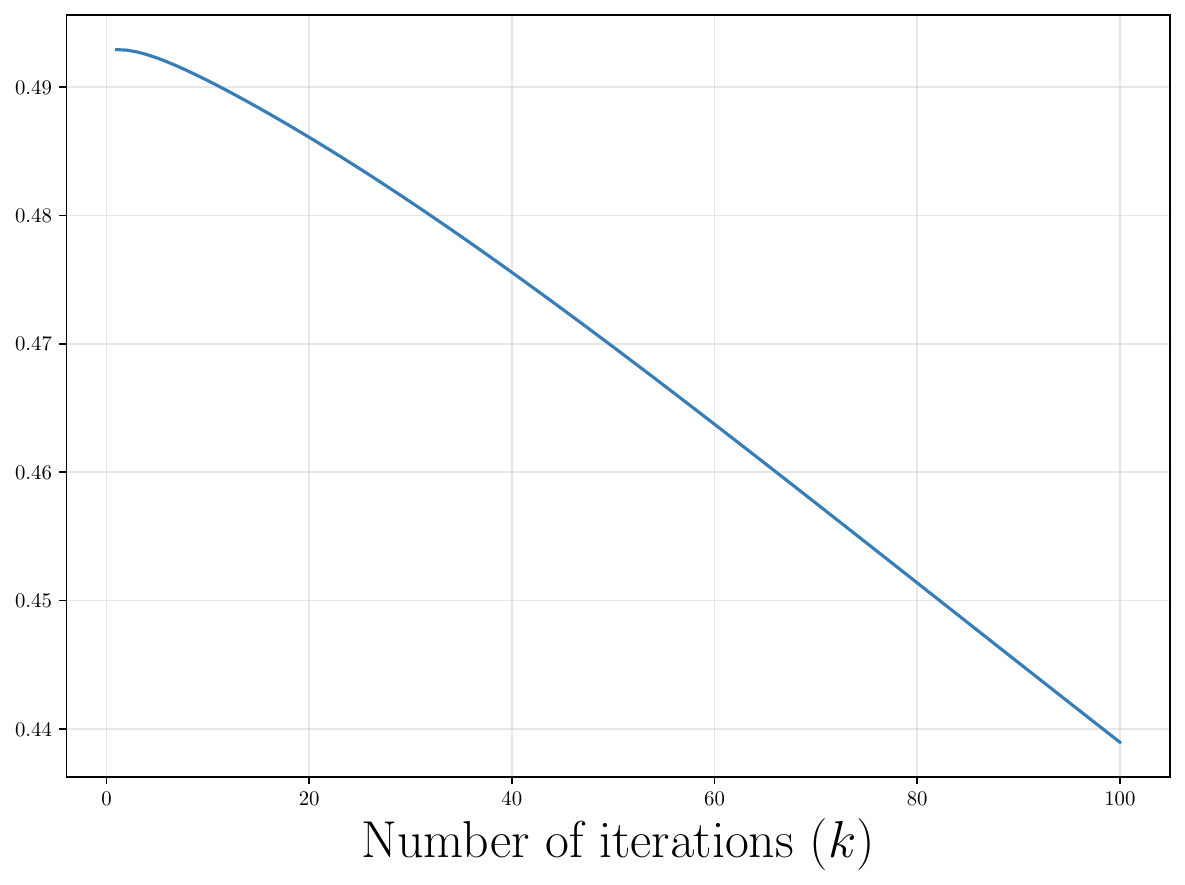}
  \caption{Client 2: $\delta d_2$}\label{fig:h}
\end{subfigure}\hfill
\begin{subfigure}{0.32\textwidth}
  \includegraphics[width=\linewidth]{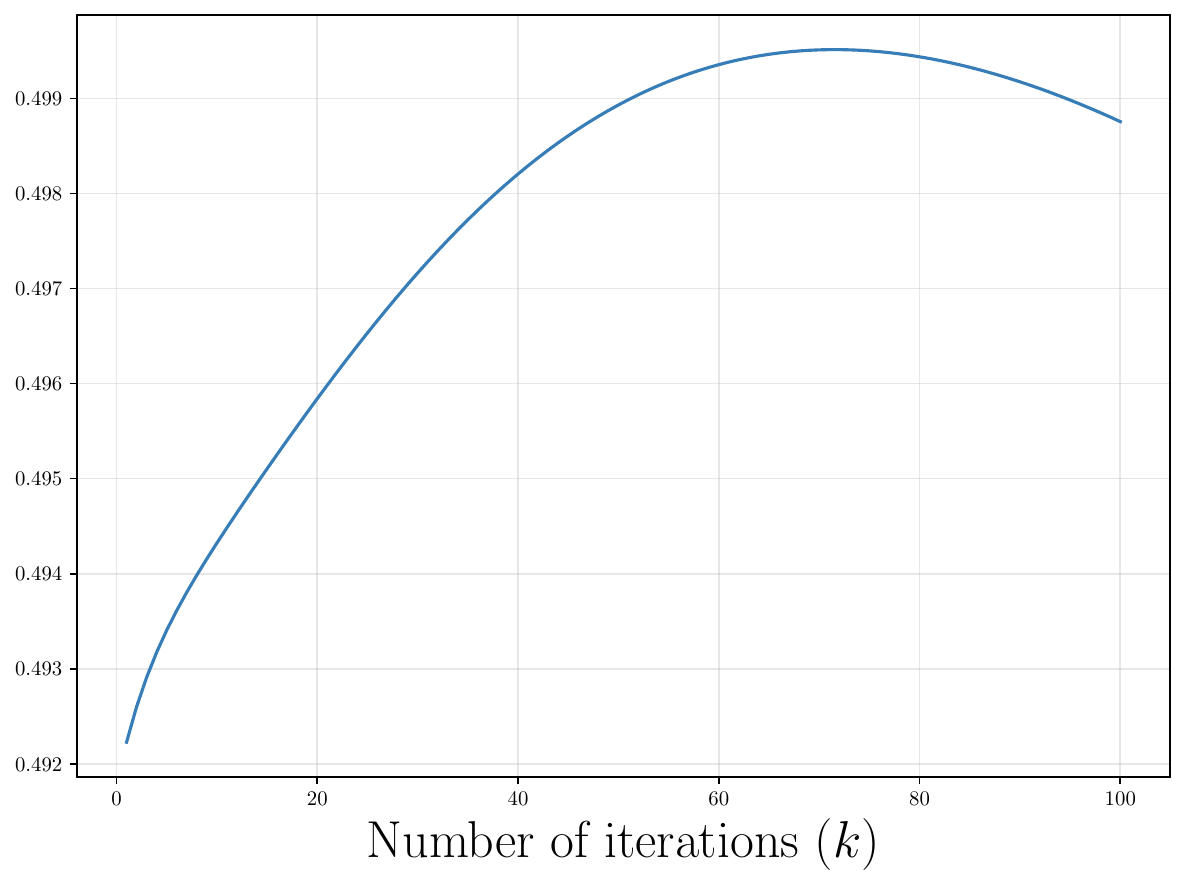}
  \caption{Client 3: $\delta d_3$}\label{fig:i}
\end{subfigure}

\begin{subfigure}{0.32\textwidth}
  \includegraphics[width=\linewidth]{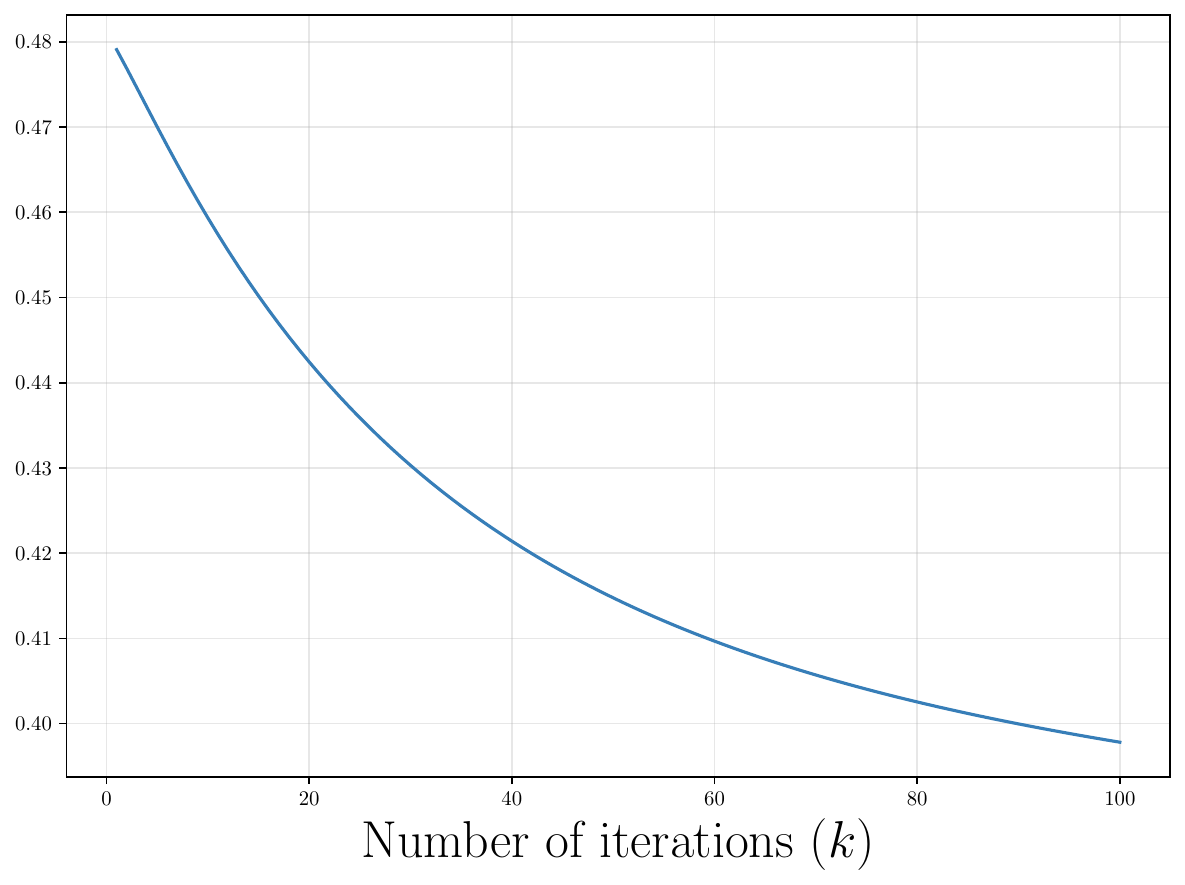}
  \caption{Client 4: $\delta d_4$}\label{fig:j}
\end{subfigure}\hfill
\begin{subfigure}{0.32\textwidth}
  \includegraphics[width=\linewidth]{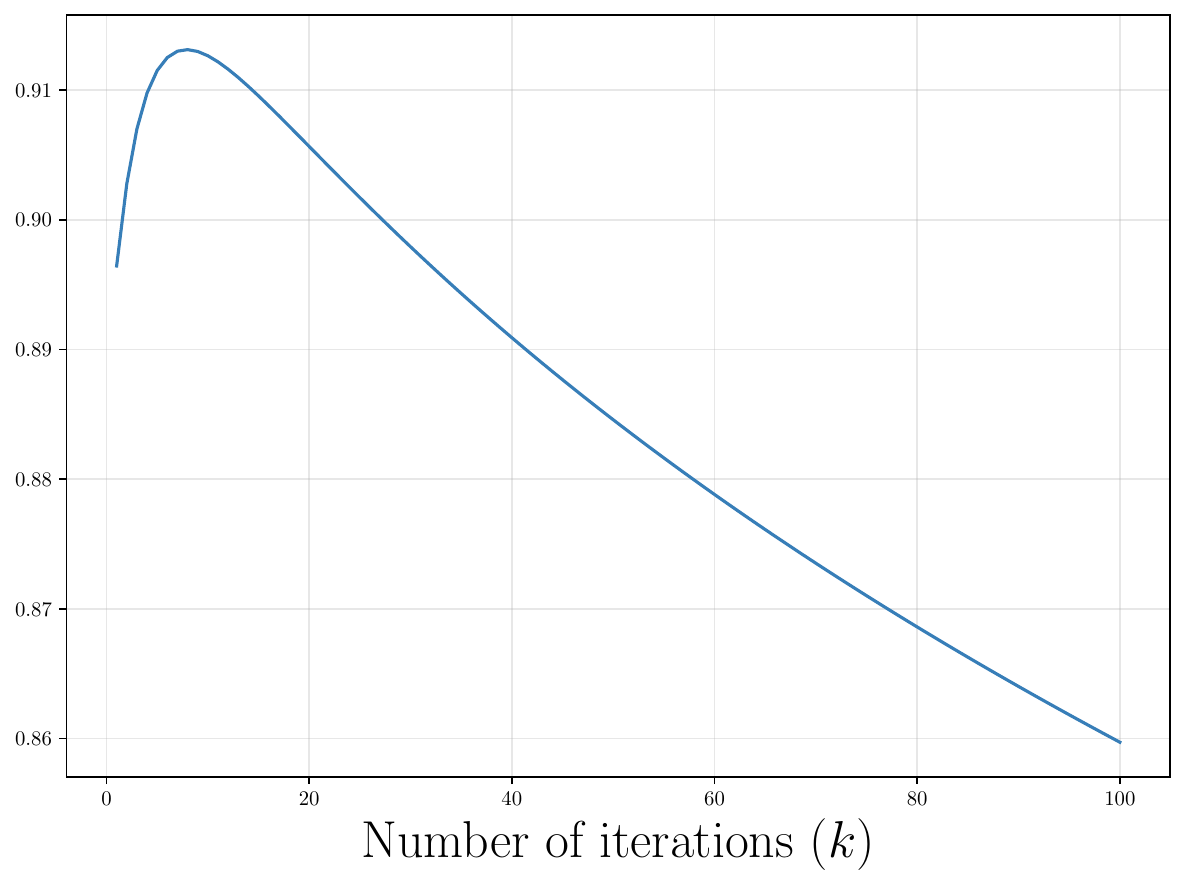}
  \caption{Client 5: $\delta d_5$}\label{fig:m}
\end{subfigure}\hfill
\begin{subfigure}{0.32\textwidth}
  \includegraphics[width=\linewidth]{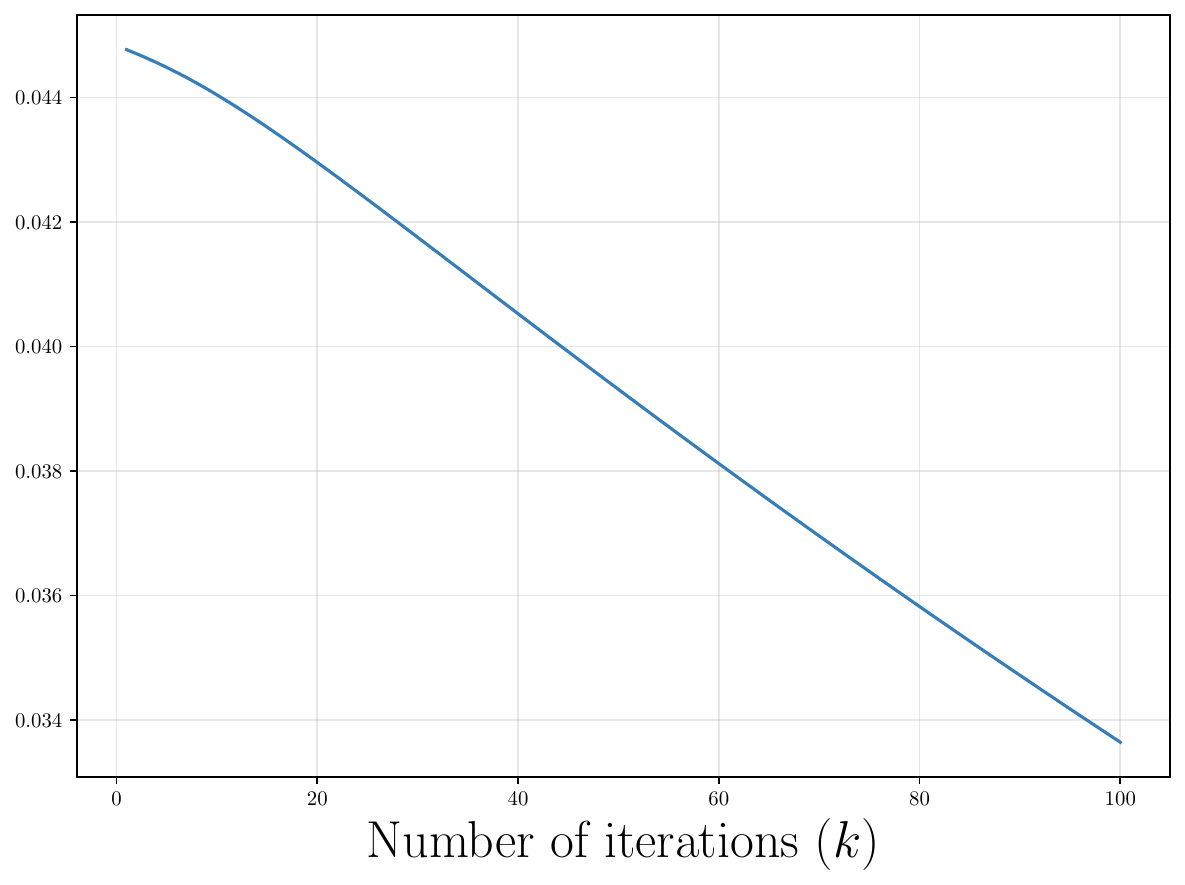}
  \caption{Client 6: $\delta d_6$}\label{fig:n}
\end{subfigure}

\end{minipage}
\end{adjustbox}

\caption{SWaT dataset}
\label{fig:swat_dataset}
\end{figure*}

\clearpage
\section{Pseudocode}
Refer to Algorithm \ref{algo:algorithm_1} for the pseudocode of the proposed federated learning framework.

\begin{algorithm}[H]
  \caption{Federated Learning of $\hat{A}_{mn}$ and $\hat{B}_{mn}$} \label{algo:algorithm_1}
  \begin{algorithmic}[1]
    \State \textbf{Inputs:} $T$, $A_{mm}$, $B_{mm}$, $C_{mm}$, $\hat{h}^t_{m,c}, u^t_m \;\forall m \in \{1,...,M\}, t \in \{1,...,T\}$
    \State \textbf{Choose:} iterations, tolerance $tol$, set $k = 0$
    \State \textbf{Initialize at Server:} $\{\hat{A}_{mn}^0\}_{m \ne n},\; \{\hat{B}_{mn}^0\}_{m \ne n}$
    \State \textbf{Choose at Server:} learning rates $\alpha_A$, $\alpha_B$, $\{\xi_m\} \; \forall m$
    \State \textbf{Initialize at Client $m$:} $\theta_m^0$, $\phi_m^0$, $\hat{h}^{0, t}_{m,a} \; \forall t$
    \State \textbf{Choose at Client:} $\eta_1$, $\eta_2$, $\gamma_1$, $\gamma_2$

    \While{$k < \textit{iterations}$ or $L_s > tol$}

        \For{Each client $m$}
            \For{$t = 1$ to $T$}
                \State $h^{k,t}_{m,a} \gets A_{mm} \cdot \hat{h}^{k,t-1}_{m,a} + B_{mm} \cdot u^{t-1}_m + \phi_m^k$
                \State $r_{m,a}^{k,t} \gets y_m^t - C_{mm} h^{k,t}_{m,a}$
                \State $\hat{h}^{k,t}_{m,a} \gets \hat{h}^t_{m,c} + \theta_m^k\, y_m^t$
                \State $L_{m,a}^k \gets L_{m,a}^k + \frac{1}{T} \norm{r_{m,a}^{k,t}}_2^2$
            \EndFor
            \If{$k = 0$}
                \State Send $\left\{h^{k,t}_{m,a}, \hat{h}^{k,t-1}_{m,a}, \hat{h}^{t-1}_{m,c}, u^{t-1}_m\right\}_{t = 1}^T$ to the server
            \Else
                \State Send $\left\{h^{k,t}_{m,a}\right\}_{t = 1}^T$ to the server
            \EndIf
        \EndFor
        \For{\textbf{[At the Server]}}  
            \For{$t = 1$ to $T$}
                \For{each $m$}
                    \State $h^{k,t}_{m,s} \gets A_{mm} \cdot \hat{h}^{t-1}_{m,c} + B_{mm} u_m^{t-1} + \sum\limits_{n \ne m} ( \hat{A}^k_{mn} \cdot \hat{h}^{t-1}_{n,c} + \hat{B}_{mn}^k \cdot u^{t-1}_n )$
                    
                \EndFor
                \State $g_s^{k,t} \coloneqq \sum_{m=1}^M \left( \| h^{k,t}_{m,a} - h^{k,t}_{m,s} \|_2^2 + \xi \| A_{mm}(\hat h^{k,t-1}_{m,a} - \hat h^{t-1}_{m,c}) - \sum_{n\ne m} \hat A_{mn} \hat h^{t-1}_{n,c} \|_2^2 \right)$
                \State $\nabla_{h^{k,t}_{m,a}} L_s \gets \frac{1}{T}\nabla_{h^{k,t}_{m,a}} g_s^{k,t}$
                \State $L_s^k \gets L_s^k + \frac{1}{T}g_s^{k,t}$
            \EndFor
            \State $\hat{A}^{k+1}_{mn} \gets \hat{A}^{k}_{mn} - \alpha_A \cdot \nabla_{\hat{A}^{k}_{mn}} L_s^k$
            \State $\hat{B}^{k+1}_{mn} \gets \hat{B}^{k}_{mn} - \alpha_B \cdot \nabla_{\hat{B}^{k}_{mn}} L_s^k$
            \State Send $\{ \nabla_{h^{k,t}_{m,a}} L_s^k \}_{t=1}^T$ to client $m$
        \EndFor
        \For{Each client $m$}
            \For{$t = 1$ to $T$}
                
                \State $\nabla_{\theta_m^k} L_s^k \gets \nabla_{\theta_m^k} L_s^k + A_{mm}^\top \left(\nabla_{h^t_{m,a}} L_s^k \right) {y_m^{t-1}}^\top$
            \EndFor
            \State $\theta_m^{k+1} \gets \theta_m^k - \eta_1 \cdot \nabla_{\theta_m^k} L_{m,a}^k - \eta_2 \cdot \nabla_{\theta_m^k} L_s^k$
            \State $\phi_m^{k+1} \gets \phi_m^k - \gamma_1 \cdot \nabla_{\phi_m^k} L_{m,a}^k - \gamma_2 \nabla_{\phi_m^k} L_s^k$
        \EndFor
    \EndWhile
  \end{algorithmic}
\end{algorithm}

\section{Proofs}
\subsection{Derivatives of $L_s$ with respect to client parameters $\theta_m$ and $\phi_m$}
\begin{proposition}
    Using the chain rule, the derivatives of the server loss $L_s$ with respect to the augmented client parameters of client $m$ (i.e., $\phi_m$ and $\theta_m$) are
    \begin{align}
        \nabla_{\theta_m} L_s &= \sum_{t=1}^T \big(A_{mm}^\top \nabla_{h^t_{m,a}} L_s + \nabla_{\hat h^{t-1}_{m,a}} L_s\big)\,{y_m^{t-1}}^{\!\top},\\
        \nabla_{\phi_m} L_s &= \sum_{t=1}^T \nabla_{h^t_{m,a}} L_s .
    \end{align}
\end{proposition}

\begin{proof}
From the model,
\[
h^t_{m,a}=A_{mm}\hat h^{t-1}_{m,a}+b_m^{t-1}+\phi_m,\qquad
\hat h^{t-1}_{m,a}=c_m^{t-1}+\theta_m y_m^{t-1}.
\]
To make the chain rule bookkeeping explicit, consider the proxy loss
\[
\ell(h^t,\hat h^{t-1}) \coloneqq \frac1T \sum_{t=1}^T
\Big(\,\|h^t-X^t\|_2^2+\xi\|A\hat h^{t-1}-Z^{t-1}\|_2^2\Big),
\]
with the same computational graph
\[
h^t=A\hat h^{t-1}+b^{t-1}+\phi,\qquad
\hat h^{t-1}=c^{t-1}+\theta y^{t-1}.
\]
We use Einstein summation over repeated Latin indices. Indices $p,k,r\in\{1,\dots,P\}$ denote state components and $j,s\in\{1,\dots,D\}$ denote measurement components. The Kronecker delta is $\delta_{ij}$, and subscripts on a vector/matrix denote components.

\paragraph{Derivative with respect to $\theta$.}
By the chain rule,
\begin{align*}
\frac{\partial \ell}{\partial \theta_{rs}}
&=\frac1T\sum_{t=1}^T\sum_{p=1}^P
\left(
\frac{\partial \ell}{\partial h^t_p}\frac{\partial h^t_p}{\partial \theta_{rs}}
+
\frac{\partial \ell}{\partial \hat h^{t-1}_p}\frac{\partial \hat h^{t-1}_p}{\partial \theta_{rs}}
\right).
\end{align*}
From $h^t_p=A_{pk}\hat h^{t-1}_k+b^{t-1}_p+\phi_p$ and $\hat h^{t-1}_k=c^{t-1}_k+\theta_{kj}y^{t-1}_j$,
\[
\frac{\partial \hat h^{t-1}_k}{\partial \theta_{rs}}=\delta_{kr}y^{t-1}_s,
\qquad
\frac{\partial h^t_p}{\partial \theta_{rs}}=A_{pk}\frac{\partial \hat h^{t-1}_k}{\partial \theta_{rs}}
=A_{pk}\delta_{kr}y^{t-1}_s=A_{pr}y^{t-1}_s.
\]
Hence,
\begin{align*}
\frac{\partial \ell}{\partial \theta_{rs}}
&=\frac1T\sum_{t=1}^T
\left[
\Big(\tfrac{\partial \ell}{\partial h^t_p}\Big)A_{pr}
+
\Big(\tfrac{\partial \ell}{\partial \hat h^{t-1}_p}\Big)\delta_{pr}
\right] y^{t-1}_s
\\
&=\frac1T\sum_{t=1}^T
\left[
\big(A^\top\nabla_{h^t}\ell\big)_r
+
\big(\nabla_{\hat h^{t-1}}\ell\big)_r
\right] y^{t-1}_s.
\end{align*}
Collecting the $(r,s)$-entries, the matrix form is
\[
\nabla_\theta \ell=\frac1T\sum_{t=1}^T \big(A^\top\nabla_{h^t}\ell+\nabla_{\hat h^{t-1}}\ell\big)\,{(y^{t-1})}^\top.
\]

\paragraph{Derivative with respect to $\phi$.}
Again by the chain rule,
\begin{align*}
\frac{\partial \ell}{\partial \phi_a}
&=\frac1T\sum_{t=1}^T\sum_{p=1}^P
\frac{\partial \ell}{\partial h^t_p}\frac{\partial h^t_p}{\partial \phi_a}
=\frac1T\sum_{t=1}^T\sum_{p=1}^P
\frac{\partial \ell}{\partial h^t_p}\,\delta_{pa}
=\frac1T\sum_{t=1}^T \big[\nabla_{h^t}\ell\big]_a,
\end{align*}
or, in matrix form,
\[
\nabla_\phi \ell=\frac1T\sum_{t=1}^T \nabla_{h^t}\ell.
\]

Finally, reinstating the client-$m$ notation
\[
A\mapsto A_{mm},\quad
h^t\mapsto h^t_{m,a},\quad
\hat h^{t-1}\mapsto \hat h^{t-1}_{m,a},\quad
y^{t-1}\mapsto y^{t-1}_m,
\]
and replacing $\ell$ by $L_s$ (the factor $1/T$ is an inconsequential constant), yields
\[
\nabla_{\theta_m} L_s = \sum_{t=1}^T \big(A_{mm}^\top \nabla_{h^t_{m,a}} L_s + \nabla_{\hat h^{t-1}_{m,a}} L_s\big)\,{y_m^{t-1}}^{\!\top},
\qquad
\nabla_{\phi_m} L_s = \sum_{t=1}^T \nabla_{h^t_{m,a}} L_s,
\]
as claimed.
\end{proof}
\subsection{Proposition 4.1}
\begin{proof}
Consider the LTI system without a direct term:
\[
h^{t}=A h^{t-1}+B u^{t-1}+w^{t-1},\qquad y^{t}=C h^{t}+v^{t},
\]
with $w^{t-1}\!\sim\!\mathcal N(0,Q)$, $v^{t}\!\sim\!\mathcal N(0,R)$, independent.

\textbf{(1) Abduction.}
Let $\mathcal Y_{t-1}$ denote data up to $t-1$. The Kalman filter yields
\[
h^{t-1}\mid \mathcal Y_{t-1} \sim \mathcal N(\hat{h}^{t-1},\,P^{t-1|t-1}).
\]
Carry this posterior (and noise laws) unchanged into the counterfactual.

\textbf{(2) Action.}
Apply the intervention $\mathrm{do}(u^{t-1}=u)$: replace $u^{t-1}$ by the chosen constant $u$ in the state update; all other mechanisms unchanged.

\textbf{(3) Prediction.}
Under the intervention (action) represented by $\mathrm{do}(u^{t-1}=u)$ we have, 
\[
h^{t}\mid \mathcal Y_{t-1} \sim \mathcal N\big(A\hat{h}^{t-1}+B u,\; P^{t|t-1}\big),\quad
P^{t|t-1}=A P^{t-1|t-1}A^\top+Q,
\]
and hence
\[
y^{t}\mid \mathrm{do}(u^{t-1}=u),\,\mathcal Y_{t-1}
\sim \mathcal N\big(\mu(u),\,S_t\big),\quad
\mu(u):=C(A\hat{h}^{t-1}+B u),\quad
S_t:=C P^{t|t-1} C^\top + R.
\]

Therefore, for $u_0,u_1$,
\[
\mu(u_1)-\mu(u_0)=CB\,(u_1-u_0),\qquad S_t\ \text{is identical.}
\]
Hence we directly obtain $ATE$ as follows, 
\[
\mathrm{ATE}
= \mathbb E[y^{t}\!\mid\!\mathrm{do}(u_1)]-\mathbb E[y^{t}\!\mid\!\mathrm{do}(u_0)]
= CB\,(u_1-u_0).
\]
\end{proof}

\subsection{Theorem 6.1}
\begin{proof}
For any $m \in \{1,\dots,M\}$, define
\[
x_m^t := \sum_{n \neq m} \hat A_{mn}\,\hat h_{n,c}^{\,t},\quad
r_m^t := h_{a,m}^{t} - A_{mm}\hat h_{m,c}^{\,t} - B_{mm}u_m^{t} - \sum_{n \neq m}\hat B_{mn}u_n^{t},\quad
z_m^t := A_{mm}\!\left(\hat h_{a,m}^{t} - \hat h_{m,c}^{\,t}\right).
\]
The per–time server loss in the $A$-block is
\[
L_s^t \;=\; \|r_m^t - x_m^t\|_2^2 + \xi\,\|z_m^t - x_m^t\|_2^2.
\]
Since $L_s^t$ depends on $\{\hat A_{mn}\}$ only through $x_m^t$, by the chain rule
\[
\nabla_{\hat A_{mn}}L_s^t \;=\; \big(\nabla_{x_m^t}L_s^t\big)\,(\hat h_{n,c}^{\,t})^\top,\qquad
\nabla_{x_m^t}L_s^t \;=\; 2\big[(1+\xi)x_m^t - (r_m^t+\xi z_m^t)\big].
\]
Training minimizes the time–average $\frac1T\sum_{t=1}^TL_s^t$, hence at a stationary point
\[
\frac1T\sum_{t=1}^T \nabla_{x_m^t}L_s^t \;=\; 0
\ \Longrightarrow\
\frac1T\sum_{t=1}^T\big[(1+\xi)x_m^t - (r_m^t+\xi z_m^t)\big] \;=\; 0.
\]
By th ergodicity assumption, time averages converge to expectations, so
\[
(1+\xi)\,\mathrm{E}[x_m^t] \;=\; \mathrm{E}[r_m^t] + \xi\,\mathrm{E}[z_m^t].
\]
Letting $\xi\to\infty$ yields
\[
\mathrm{E}[x_m^t] \;=\; \mathrm{E}[z_m^t].
\]
Finally, by augmentation $\hat h_{a,m}^{t}=\hat h_{m,c}^{\,t}+\theta_m^{*}y_m^{t}$, so
\[
z_m^t \;=\; A_{mm}\theta_m^{*}y_m^{t},
\qquad
x_m^t \;=\; \sum_{n\neq m}\hat A_{mn}^*\,\hat h_{n,c}^{\,t},
\]
and therefore
\[
\mathrm{E}\!\Big[A_{mm}\theta_m^* y_m^{t}\Big] \;=\; \mathrm{E}\!\Big[\sum_{n \neq m} \hat A_{mn}^*\,\hat h^{\,t}_{n,c}\Big].
\]
\end{proof}

\subsection{Corollary 6.2}
\begin{proof}
For any client $m$, we define, 
\[
p_m^t := \sum_{n\neq m}\hat B_{mn}\,u_n^{t},
\qquad
c_m^t := h_{a,m}^{t} - A_{mm}\hat h_{m,c}^{\,t} - B_{mm}u_m^{t}.
\]
With $x_m^t$ fixed, the per–time loss in the $B$-block is
\[
L_s^t(p_m^t) \;=\; \|c_m^t - x_m^t - p_m^t\|_2^2 \;+\; \xi\,\|z_m^t - x_m^t\|_2^2,
\]
so
\[
\nabla_{p_m^t}L_s^t \;=\; -2\,(c_m^t - x_m^t - p_m^t).
\]
Stationarity of the time–average gives
\[
\frac1T\sum_{t=1}^T \nabla_{p_m^t}L_s^t \;=\; 0
\ \Longrightarrow\
\frac1T\sum_{t=1}^T \big(c_m^t - x_m^t - p_m^t\big) \;=\; 0.
\]
By ergodicity,
\[
\mathrm{E}[p_m^t] \;=\; \mathrm{E}[c_m^t - x_m^t].
\]
From Theorem~6.1, as $\xi\to\infty$, $\mathrm{E}[x_m^t]=\mathrm{E}[z_m^t]=\mathrm{E}[A_{mm}(\hat h_{a,m}^{t}-\hat h_{m,c}^{\,t})]$, hence
\[
\mathrm{E}[c_m^t - x_m^t]
= \mathrm{E}\big[h_{a,m}^{t} - A_{mm}\hat h_{m,c}^{\,t} - B_{mm}u_m^{t} - A_{mm}(\hat h_{a,m}^{t}-\hat h_{m,c}^{\,t})\big]
= \mathrm{E}\big[h_{a,m}^{t} - A_{mm}\hat h_{a,m}^{t} - B_{mm}u_m^{t}\big].
\]
Using the augmented client model $h_{a,m}^{t}=A_{mm}\hat h_{a,m}^{t}+B_{mm}u_m^{t}+\phi_m^{*}$ at stationarity yields
\[
\mathrm{E}[c_m^t - x_m^t] \;=\; \mathrm{E}[\phi_m^{*}] \;=\; \phi_m^{*},
\]
so
\[
\phi_m^{*} \;=\; \mathrm{E}[p_m^t] \;=\; \mathrm{E}\!\Big[\sum_{n\neq m}\hat B_{mn}^*\,u_n^{t}\Big].
\]
\end{proof}

\subsection{Lemma 7.2}
\begin{proof}
Since $L_{m,a}=\tfrac{1}{T}\sum_{t=1}^T\|r_{m,a}^t\|^2$, its gradients are given as:
\begin{align*}
\nabla_{\theta_m}L_{m,a}
&= -\frac{2}{T}\sum_{t=1}^T (C_{mm}A_{mm})^\top r_{m,a}^t (y_m^{t-1})^\top,\\
\nabla_{\phi_m}L_{m,a}
&= -\frac{2}{T}\sum_{t=1}^T C_{mm}^\top r_{m,a}^t.
\end{align*}
At a stationary point, $\nabla_{\theta_m}L_{m,a}=0$ and $\nabla_{\phi_m}L_{m,a}=0$, hence
\begin{align*}
\frac1T\sum_{t=1}^T (C_{mm}A_{mm})^\top r_{m,a}^{*,t} (y_m^{t-1})^\top&=0,
&
\frac1T\sum_{t=1}^T C_{mm}^\top r_{m,a}^{*,t}&=0.
\end{align*}
By the ergodicity assumption, these time-averages converge to expectations, giving
\begin{align*}
\mathrm{E}\!\big[(C_{mm}A_{mm})^\top r_{m,a}^{*,t}(y_m^{t-1})^\top\big]&=0, \qquad \text{and} \qquad
\mathrm{E}\!\big[C_{mm}^\top r_{m,a}^{*,t}\big]=0.
\end{align*}
\end{proof}

\subsection{Theorem 7.3}
\begin{proof}
From the augmented client model, we write $\tilde y_{m,a}^t$ as
\begin{align*}
h_{m,a}^t
&= A_{mm}\big(\hat h_{m,c}^{t-1}+\theta_m^* y_m^{t-1}\big)+B_{mm}u_m^{t-1}+\phi_m^*,\\
\tilde y_{m,a}^t(\theta_m^*,\phi_m^*)
&= C_{mm}h_{m,a}^t\\
&= C_{mm}A_{mm}\theta_m^*\,y_m^{t-1}
  \;+\; C_{mm}A_{mm}\,\hat h_{m,c}^{t-1}
  \;+\; C_{mm}B_{mm}\,u_m^{t-1}
  \;+\; C_{mm}\phi_m^*.
\end{align*}
Adding and subtracting the cross-input block so that all inputs $u^{t-1}=[u_1^{t-1};\ldots;u_M^{t-1}]$ appear linearly:
\begin{align*}
\tilde y_{m,a}^t(\theta_m^*,\phi_m^*)
&= C_{mm}A_{mm}\theta_m^*\,y_m^{t-1}
  \;+\; C_{mm}\big[\,\hat B_{m1}^*,\ldots,B_{mm},\ldots,\hat B_{mM}^*\,\big]\,u^{t-1}
  \;+\; C_{mm}A_{mm}\,\hat h_{m,c}^{t-1}\\
&\qquad\qquad
  +\; C_{mm}\!\Big(\phi_m^*-\sum_{n\neq m}\hat B_{mn}^*u_n^{t-1}\Big).
\end{align*}
Define the stacked regressor and the coefficient matrix
\begin{align*}
z_m^{t-1}&:=\begin{bmatrix}y_m^{t-1}\\ u^{t-1}\\ \hat h_{m,c}^{t-1}\end{bmatrix},\qquad
M_{\mathrm{fed}}:=\big[\,C_{mm}A_{mm}\theta_m^*\ \ \ C_{mm}[\,\hat B_{m1}^*,\ldots,B_{mm},\ldots,\hat B_{mM}^*\,]\ \ \ C_{mm}A_{mm}\,\big],
\end{align*}
and the zero-mean error $e_t$ as, 
\begin{align*}
e_t:=C_{mm}\!\Big(\phi_m^*-\sum_{n\neq m}\hat B_{mn}^*u_n^{t-1}\Big),\qquad \mathrm{E}[e_t]=0
\ \ \text{(by Corollary 6.2)}.
\end{align*}
Then we have this affine relationship, 
\begin{align*}
\tilde y_{m,a}^{t}(\theta_m^*,\phi_m^*) \;=\; M_{\mathrm{fed}}\,z_m^{t-1}+e_t.
\end{align*}

From the theorem, we know that $\Sigma_z:=\mathrm{E}[z_m^{t-1}z_m^{t-1\!\top}]$ and $\Gamma_{yz}:=\mathrm{E}[y_m^t z_m^{t-1\!\top}]$. Using Lemma 7.2 we have,
\[
\mathrm{E}\!\big[(C_{mm}A_{mm})^\top r_{m,a}^{*,t}(y_m^{t-1})^\top\big]=0,
\qquad
\mathrm{E}\!\big[C_{mm}^\top r_{m,a}^{*,t}\big]=0.
\]
Since $C_{mm}A_{mm}$ is full column rank, there exists a left inverse
$L_{mm}$ with $L_{mm}(C_{mm}A_{mm})=I$, so
\[
\mathrm{E}\!\big[r_{m,a}^{*,t}(y_m^{t-1})^\top\big]=0.
\]
Stacking with the input and state–estimate blocks gives
\[
\mathrm{E}\!\big[(y_m^t-\tilde y_{m,a}^{t})\,z_m^{t-1\!\top}\big]
=\big[\,0,\ \ \mathrm{E}\!\big[r_{m,a}^{*,t}(u^{t-1})^\top\big],\ \ \mathrm{E}\!\big[r_{m,a}^{*,t}(\hat h_{m,c}^{t-1})^\top\big]\,\big].
\]
Substitute $\tilde y_{m,a}^{t}=M_{\mathrm{fed}}z_m^{t-1}+e_t$ to obtain
\[
\Gamma_{yz}-M_{\mathrm{fed}}\Sigma_z-\mathrm{E}[e_t z_m^{t-1\!\top}]
=\big[\,0,\ \ \mathrm{E}\!\big[r_{m,a}^{*,t}(u^{t-1})^\top\big],\ \ \mathrm{E}\!\big[r_{m,a}^{*,t}(\hat h_{m,c}^{t-1})^\top\big]\,\big].
\]
Now use $r_{m,a}^{*,t}:=y_m^t-y_{m, a}^t=y_m^t-(M_{\mathrm{fed}}z_m^{t-1}+e_t)$ to rewrite the two nonzero blocks:
\[
\mathrm{E}[r_{m,a}^{*,t}(u^{t-1})^\top]
= \mathrm{E}[(y_m^t-M_{\mathrm{fed}}z_m^{t-1})(u^{t-1})^\top]-\mathrm{E}[e_t(u^{t-1})^\top],
\]
and similarly
\[
\mathrm{E}[r_{m,a}^{*,t}(\hat h_{m,c}^{t-1})^\top]
= \mathrm{E}[(y_m^t-M_{\mathrm{fed}}z_m^{t-1})(\hat h_{m,c}^{t-1})^\top]-\mathrm{E}[e_t(\hat h_{m,c}^{t-1})^\top].
\]
But the first terms on the right in these two lines are exactly the $u$- and $\hat h$-blocks of $(\Gamma_{yz}-M_{\mathrm{fed}}\Sigma_z)$. Therefore the whole right-hand side equals $(\Gamma_{yz}-M_{\mathrm{fed}}\Sigma_z)-\mathrm{E}[e_t z_m^{t-1\!\top}]$, and we conclude
\[
\Gamma_{yz}-M_{\mathrm{fed}}\Sigma_z-\mathrm{E}[e_t z_m^{t-1\!\top}]=0.
\]

hence we obtain, 
\begin{align*}
M_{\mathrm{fed}}=\Gamma_{yz}\Sigma_z^{-1}-\mathrm{E}[e_t z_m^{t-1\!\top}]\Sigma_z^{-1}.
\end{align*}

Define the oracle (centralized model) linear predictor on the same regressors:
\begin{align*}
\tilde y_{m,o}^t:=M_o z_m^{t-1},\qquad M_o:=\Gamma_{yz}\Sigma_z^{-1}.
\end{align*}
Therefore
\begin{align*}
\tilde y_{m,o}^{t}-\tilde y_{m,a}^{t}(\theta_m^*,\phi_m^*)
&=(M_o-M_{\mathrm{fed}})z_m^{t-1}-e_t
= \underbrace{\mathrm{E}[e_t z_m^{t-1\!\top}]\Sigma_z^{-1}}_{=:J_m}\,z_m^{t-1}-e_t.
\end{align*}
Taking expectations and using $\mathrm{E}[e_t]=0$ gives the claimed relation
\begin{align*}
\mathrm{E}\big[\tilde y_{m,o}^{t}-\tilde y_{m,a}^{t}(\theta_m^{*},\phi_m^{*})\big]
= J_m\,\mathrm{E}[z_m^{t-1}].
\end{align*}
\end{proof}

\subsection{Lemma A.6}
\begin{proof}
By definition, $z_m^t=[\hat h_{m,c}^t;\hat h_{m,a}^t;h_{m,a}^t;u_m^t]$ with the same $u_m^t$ in $\mathcal{D}_m$ and $\mathcal{D}_m'$. Thus, 
\[
\|z_m^t-z_m^{t\,\prime}\|_2
\le
\|\hat h_{m,c}^t-\hat h_{m,c}^{t\,\prime}\|_2
+ \|\hat h_{m,a}^t-\hat h_{m,a}^{t\,\prime}\|_2
+ \|h_{m,a}^t-h_{m,a}^{t\,\prime}\|_2 .
\]
Contractivity implies that a one-time perturbation at $y_m^{t^\star}$ of magnitude at most $R_y\le R_{\max}$ yields
\[
\|h_{m,a}^{t}-h_{m,a}^{t\,\prime}\|_2 \le L_m\beta_m^{t-t^\star}\,R_{\max} \quad (t\ge t^\star),
\]
Collecting constants from the observer/augmenter maps and using subadditivity of $\|\cdot\|_2$, there exists a finite $\kappa_m$ (depending on $L_m,\beta_m$ and the estimator gains) such that
\[
\|z_m^t-z_m^{t\,\prime}\|_2 \le \kappa_m R_{\max}.
\]
One we have the above inequality, we define $\Delta_{m,\mathrm{msg}}:=\kappa_m R_{\max}$.
\end{proof}

\subsection{Lemma A.8}
\begin{proof}
For notational brevity we write, \[H_m^t := A_{mm} \hat{h}_{m,c}^{t-1}
+ \sum_{n \ne m} \hat{A}_{mn} \hat{h}_{n,c}^{t-1}
+ B_{mm} u_m^{t-1}
+ \sum_{n \ne m} \hat{B}_{mn} u_n^{t-1}\] 

Thus, the server gradient expression is now given as, 
\[
g_m^t = \tfrac{2}{T}\big(h_{m,a}^t - H_m^t\big).
\]

Hence we obtain, 
\[
\|g_m^t-g_m^{t\,\prime}\|_2 \le \tfrac{2}{T}\big(\|h_{m,a}^t-h_{m,a}^{t\,\prime}\|_2 + \|H_m^t-H_m^{t\,\prime}\|_2\big).
\]
\textbf{(i)} \emph{First term:} a one-time change in $y_{m'}^{t^\star}$ propagates to $h_{m',a}$ using
$\|h_{m',a}^{t}-h_{m',a}^{t\,\prime}\|_2 \le L_{m'}\beta_{m'}^{t-t^\star}R_{\max}$. This yields the contribution $L_{m'}\|C_{mm}\|\beta_{m'}^{t-t^\star}R_{\max}$.

\textbf{(ii)} \emph{Second term:} For $n\neq m$, $\|\hat h_{n,c}^{t-1}-\hat h_{n,c}^{t-1\,\prime}\|_2\le L_n\beta_n^{t-1-t^\star}R_{\max}$ if $n=m'$ and zero otherwise (each client’s proprietary estimator depends only on its own measurements). Thus we obtain, 
\[
\|H_m^t-H_m^{t\,\prime}\|_2
\le \sum_{n\neq m}\|\hat A_{mn}\|\,\|\hat h_{n,c}^{t-1}-\hat h_{n,c}^{t-1\,\prime}\|_2
\le \sum_{n\neq m}\|\hat A_{mn}\|\,L_{n}\beta_{n}^{t-1-t^\star}R_{\max}.
\]
Combining \textbf{(i)} and \textbf{(ii)}, absorbing $\beta_{n}^{t-1-t^\star}\le \beta_{n}^{t-t^\star}$ and multiplying by the prefactor $\frac{2}{T}$ plus the linear factor $(1+\|A_{mm}\|)$ from the residual structure yields
\[
\|g_m^t-g_m^{t\,\prime}\|_2
\le \frac{2}{T}(1+\|A_{mm}\|)\,
L_{m'}\!\left(\|C_{mm}\| + \sum_{n\ne m}\|A_{mn}\|\beta_{n}^{t-t^\star}\right)\!R_{\max}.
\]
Adding $\sum_{n\neq m}\|B_{mn}\|$ inside the bracket gives a looser but still valid bound; this produces exactly the stated $\kappa_{m,\mathrm{mix}}$.
\end{proof}

\subsection{Proposition A.11}
\begin{proof}
Consider a mechanism with $\ell_2$-sensitivity $\Delta$ and clipping $C$ s.t., 
$\tilde{x}= \mathcal{M}(x)+\mathcal{N}(0,\sigma^2 C^2 I)$. For neighboring inputs $x,x'$, the privacy loss random variable is Gaussian-subgaussian; the classical analysis of the (basic) Gaussian mechanism yields
\[
\sigma \ \ge\ \frac{\Delta}{C}\cdot \frac{\sqrt{2\ln(1.25/\delta)}}{\varepsilon}
\]
as a sufficient condition for $(\varepsilon,\delta)$-DP. 

Applying this with $(\Delta,C)=(\Delta_{m,\mathrm{msg}}, C_{\mathrm{msg}})$ for messages and $(\Delta_{m,\mathrm{grad}}, C_{\mathrm{grad}})$ for gradients gives the two stated inequalities.
\end{proof}

\subsection{Proposition A.12}
\begin{proof}
\textbf{(Sequential composition.)} If $\mathcal{M}_1$ is $(\varepsilon_1,\delta_1)$-DP and $\mathcal{M}_2$ is $(\varepsilon_2,\delta_2)$-DP, then the pair $(\mathcal{M}_1,\mathcal{M}_2)$ is $(\varepsilon_1+\varepsilon_2,\delta_1+\delta_2)$-DP by the standard composition theorem (union bound on failure events and multiplicativity of the likelihood ratio bounds). Inducting over $R$ rounds yields $(\sum_r \varepsilon_r, \sum_r \delta_r)$.

\textbf{(Joint per-round composition.)} In each round, releasing both message and gradient corresponds to composing two mechanisms with guarantees $(\varepsilon_{\mathrm{msg}},\delta_{\mathrm{msg}})$ and $(\varepsilon_{\mathrm{grad}},\delta_{\mathrm{grad}})$, so the per-round guarantee is $(\varepsilon_{\mathrm{msg}}+\varepsilon_{\mathrm{grad}}, \delta_{\mathrm{msg}}+\delta_{\mathrm{grad}})$. Composing these $R$ times proves the stated bound.
\end{proof}

\subsection{Corollary A.13}
\begin{proof}
Let $\mathcal{M}$ denote the overall privatized training mechanism that produces the transcript
$\mathcal{T}_{\mathrm{priv}} = \{\tilde z_m^t, \tilde g_m^t\}_{m,t}$.
We know that, $\mathcal{M}$ satisfies $(\varepsilon,\delta)$-differential privacy with respect to the raw measurement
$\mathcal{D} = \{y_m^t\}$.

After the privatized transcript is generated, all subsequent computations,
including optimization of system matrices
$\{\hat A_{mn}, \hat B_{mn}\}$,
client updates of
$\{\theta_m,\phi_m\}$,
and downstream counterfactual predictions are deterministic or randomized functions of $\mathcal{T}_{\mathrm{priv}}$ only.
Let us denote this transformation by
\[
\mathcal{T}_{\mathrm{post}}:\mathcal{T}_{\mathrm{priv}}\mapsto
\big(\hat A,\hat B,\theta,\phi,\widehat{\mathrm{ATE}},\widehat{\mathrm{CF}}\big).
\]

For any pair of neighboring datasets $\mathcal{D},\mathcal{D}'$
that differ in one measurement $y_{m^\star}^{t^\star}$,
the privacy guarantee of $\mathcal{M}$ implies that for every measurable set $S$,
\[
\Pr[\mathcal{M}(\mathcal{D})\in S]
\le e^{\varepsilon}\Pr[\mathcal{M}(\mathcal{D}')\in S]+\delta.
\]
Because $\mathcal{T}_{\mathrm{post}}$ is applied only to the outputs of $\mathcal{M}$,
we can substitute $S'=\mathcal{T}_{\mathrm{post}}^{-1}(S)$ to obtain
\[
\Pr[\mathcal{T}_{\mathrm{post}}\!\circ\!\mathcal{M}(\mathcal{D})\in S]
= \Pr[\mathcal{M}(\mathcal{D})\in S']
\le e^{\varepsilon}\Pr[\mathcal{M}(\mathcal{D}')\in S']
+\delta
= e^{\varepsilon}\Pr[\mathcal{T}_{\mathrm{post}}\!\circ\!\mathcal{M}(\mathcal{D}')\in S]+\delta.
\]
Hence $\mathcal{T}_{\mathrm{post}}\!\circ\!\mathcal{M}$ is also $(\varepsilon,\delta)$-DP.

\end{proof}

\end{document}